\documentclass[11pt]{article} 

\usepackage{times} 

\usepackage{xspace}

\usepackage{amsmath}
\usepackage{graphicx}
\usepackage{framed}
\usepackage{bbm}
\usepackage{algorithm}
\usepackage{amsthm}
\usepackage[algo2e, vlined, ruled, linesnumbered]{algorithm2e}

\usepackage{natbib}

\usepackage{mathrsfs}
\usepackage{amssymb}
\usepackage{bm}
\usepackage[colorlinks=true, linkcolor=blue, citecolor=blue,urlcolor=black,linktocpage=true]{hyperref}
\usepackage{enumitem}
\setlist[itemize]{leftmargin=1cm}
\usepackage{makecell}
\usepackage{tabulary}
\usepackage{xcolor}
\usepackage{prettyref}
\usepackage{mathtools}
\usepackage{amsmath}
\usepackage{multirow}
\usepackage{nicefrac}
\usepackage{amssymb}
\usepackage{pifont}

\newcommand{\hatQ}{\widehat{Q}}

\newcommand{\hq}{\widehat{\rm q}}

\newcommand{\CW}[1]{{\color{red}{\small{[CW: #1]}}}}

\DeclareMathOperator{\Cov}{Cov}
\DeclareMathOperator{\Tr}{Tr}

\newcommand{\bcov}{\overline{\Cov}}
\newcommand{\hGamma}{\widehat{\Gamma}}

\definecolor{Green}{rgb}{0.13, 0.65, 0.3}
\usepackage{colortbl}
\usepackage[]{color-edits}
\addauthor{HL}{red}
\DeclareMathOperator*{\argmin}{argmin} 
\DeclareMathOperator*{\polylog}{polylog}
\DeclareMathOperator*{\poly}{poly}
\DeclareMathOperator*{\clip}{clip}

\newcommand{\regterm}{\textbf{\textup{reg-term}}\xspace}
\newcommand{\biasone}{\textbf{\textup{bias-1}}\xspace}
\newcommand{\biastwo}{\textbf{\textup{bias-2}}\xspace}
\newcommand{\biasthree}{\textbf{\textup{bias-3}}\xspace}
\newcommand{\ftrl}{\textbf{\textup{ftrl}}\xspace}
\newcommand{\penal}{\textbf{\textup{penalty}}\xspace}
\newcommand{\stabone}{\textbf{\textup{stability-1}}\xspace}
\newcommand{\stabtwo}{\textbf{\textup{stability-2}}\xspace}
\newcommand{\error}{\textbf{\textup{error}}\xspace}
\newcommand{\con}{C_\iota}

\newcommand{\e}{\mathrm{e}\xspace}

\newcommand{\ind}{\mathbb{I}}

\newcommand{\hatSigma}{\hat{\Sigma}}

\newcommand{\Cbonus}{C_{\textup{bonus}}}

\newcommand{\calN}{\mathcal{N}}

\newcommand{\calV}{\mathcal{V}}
\newcommand{\calA}{\mathcal{A}}
\newcommand{\calB}{\mathcal{B}}
\newcommand{\calE}{\mathcal{E}}
\newcommand{\calZ}{\mathcal{Z}}
\newcommand{\calH}{\mathcal{H}}
\newcommand{\calS}{\mathcal{S}}

\newcommand{\calL}{\mathcal{L}}
\newcommand{\calC}{\mathcal{C}}

\newcommand{\calD}{\mathcal{D}}

\newcommand{\calR}{\mathcal{R}}

\newcommand{\otil}{\widetilde{\order}}

\newcommand{\E}{\mathbb{E}}
\newcommand{\order}{\mathcal{O}}

\newcommand{\norm}[1]{\left\|#1\right\|}

\newcommand{\calX}{\mathcal{X}}

\newcommand{\calF}{\mathcal{F}}

\newtheorem{theorem}{Theorem}

\newcommand{\highlight}[1]{{#1}}
\newtheorem{lemma}[theorem]{Lemma}

\newtheorem{definition}[theorem]{Definition}

\usepackage{tikz}

\newcommand{\nonl}{\renewcommand{\nl}{\let\nl}}

\usepackage{prettyref}
\newcommand{\pref}[1]{\prettyref{#1}}

\newcommand{\savehyperref}[2]{\texorpdfstring{\hyperref[#1]{#2}}{#2}}
\newrefformat{eq}{\savehyperref{#1}{Eq.~\textup{(\ref*{#1})}}}
\newrefformat{eqn}{\savehyperref{#1}{Eq.~\textup{(\ref*{#1})}}}
\newrefformat{lem}{\savehyperref{#1}{Lemma~\ref*{#1}}}
\newrefformat{lemma}{\savehyperref{#1}{Lemma~\ref*{#1}}}
\newrefformat{def}{\savehyperref{#1}{Definition~\ref*{#1}}}
\newrefformat{line}{\savehyperref{#1}{Line~\ref*{#1}}}
\newrefformat{thm}{\savehyperref{#1}{Theorem~\ref*{#1}}}
\newrefformat{corr}{\savehyperref{#1}{Corollary~\ref*{#1}}}
\newrefformat{cor}{\savehyperref{#1}{Corollary~\ref*{#1}}}
\newrefformat{sec}{\savehyperref{#1}{Section~\ref*{#1}}}
\newrefformat{subsec}{\savehyperref{#1}{Section~\ref*{#1}}}
\newrefformat{app}{\savehyperref{#1}{Appendix~\ref*{#1}}}
\newrefformat{assum}{\savehyperref{#1}{Assumption~\ref*{#1}}}
\newrefformat{ex}{\savehyperref{#1}{Example~\ref*{#1}}}
\newrefformat{fig}{\savehyperref{#1}{Figure~\ref*{#1}}}
\newrefformat{alg}{\savehyperref{#1}{Algorithm~\ref*{#1}}}
\newrefformat{rem}{\savehyperref{#1}{Remark~\ref*{#1}}}
\newrefformat{conj}{\savehyperref{#1}{Conjecture~\ref*{#1}}}
\newrefformat{prop}{\savehyperref{#1}{Lemma~\ref*{#1}}}
\newrefformat{proto}{\savehyperref{#1}{Protocol~\ref*{#1}}}
\newrefformat{prob}{\savehyperref{#1}{Problem~\ref*{#1}}}
\newrefformat{claim}{\savehyperref{#1}{Claim~\ref*{#1}}}
\newrefformat{que}{\savehyperref{#1}{Question~\ref*{#1}}}
\newrefformat{op}{\savehyperref{#1}{Open Problem~\ref*{#1}}}
\newrefformat{fn}{\savehyperref{#1}{Footnote~\ref*{#1}}}
\newrefformat{tab}{\savehyperref{#1}{Table~\ref*{#1}}}
\newrefformat{fig}{\savehyperref{#1}{Figure~\ref*{#1}}}
\newrefformat{proc}{\savehyperref{#1}{Procedure~\ref*{#1}}}

\usepackage{caption}

\newcommand{\bias}{\textbf{\textup{bias}}}
\newcommand{\FTRL}{\textbf{\textup{ftrl}}}
\newcommand{\bonus}{\textbf{\textup{bonus}}}
\newcommand{\estimate}{\text{EstOM}}

\renewcommand{\hatSigma}{\widehat{\Sigma}}
\usepackage{hyperref}
\usepackage{url}
\usepackage{algorithmic}
\usepackage{amsmath}
\usepackage{bbding}
\usepackage{utfsym}

\usepackage{titletoc}
\usepackage[toc,page,header]{appendix}

\title{Towards Optimal Regret in Adversarial Linear MDPs \\with Bandit Feedback}

\date{}
\author{%
    Haolin Liu\thanks{The authors are listed in alphabetical order. }  \\
    \scalebox{0.9}{University of Virginia}\\
    \scalebox{0.9}{\texttt{srs8rh@virginia.edu}} 
    \and 
    Chen-Yu Wei$^*$ \\
    \scalebox{0.9}{University of Virginia} \\ 
    \scalebox{0.9}{\texttt{chenyu.wei@virginia.edu}} 
    \and Julian Zimmert$^*$ \\
    \scalebox{0.9}{Google Research} \\ 
    \scalebox{0.9}{\texttt{zimmert@google.com}}
}

\begin{document}

\maketitle

\begin{abstract}
We study online reinforcement learning in linear Markov decision processes with adversarial losses and bandit feedback, without prior knowledge on transitions or access to simulators. We introduce two algorithms that achieve improved regret performance compared to existing approaches. The first algorithm, although computationally inefficient, ensures a regret of \( \otil(\sqrt{K}) \), where $K$ is the number of episodes. This is the first result with the optimal $K$ dependence in the considered setting. The second algorithm, which is based on the policy optimization framework, guarantees a regret of \( \otil(K^{\nicefrac{3}{4}}) \) and is computationally efficient. Both our results significantly improve over the state-of-the-art: a computationally inefficient algorithm by \cite{kong2023improved} with \( \otil(K^{\nicefrac{4}{5}}+\poly(\nicefrac{1}{\lambda_{\min}})) \) regret, for some problem-dependent constant $\lambda_{\min}$ that can be arbitrarily close to zero, and a computationally efficient algorithm by \cite{sherman2023improved} with \( \otil(K^{\nicefrac{6}{7}}) \) regret.

\end{abstract}
\section{Introduction}
We study finite-horizon online reinforcement learning in a large state space with adversarial losses amd bandit feedback. 
We assume the linear Markov decision process (MDP) structure: every state-action pair is equiped with a known feature representation, and both the transitions and the losses can be represented as a linear function of the feature. 
This problem has received significant attention recently, with fairly complete results when the agent has access to a simulator to query transitions of the MDP \citep{dai2023refined}.
In the much harder simulator-free setting, the pioneering work of \citet{luo2021policy} showed that no-regret ($K^{\nicefrac{14}{15}}$ regret) is possible, where $K$ is the number of episodes. Several followup works have successively improved the $K$ dependence \citep{dai2023refined, sherman2023improved, kong2023improved}, with the state-of-the-art being \cite{kong2023improved}'s $K^{\nicefrac{4}{5}}+\poly(1/\lambda_{\min})$ regret through a computationally inefficient algorithm, and \cite{sherman2023improved}'s $K^{\nicefrac{6}{7}}$ regret through a computationally efficient algorithm. Still, there remain significant gaps between the current upper bounds and the $\sqrt{K}$ lower bound. 
In this work, we push the frontiers both on the information theoretical limits and the achievable bounds under computational constraints: 
1) we present the first (computationally inefficient) algorithm that provably obtains $\otil(\sqrt{K})$ regret, showing that this is the minimax $K$ dependence (\pref{sec: ineff alg}); 
2) we obtain $\otil(K^{\nicefrac{3}{4}})$ regret with a polynomial-time algorithm (\pref{sec:efficient}). Below, we briefly describe the elements in our approaches. 

\paragraph{Inefficient $\sqrt{K}$ algorithm. }
We convert the linear MDP problem to a linear bandit problem by mapping each policy to a single $dH$-dimensional feature vector, where $d$ is the ambient dimension of the linear MDP and $H$ is the horizon length. The challenge is that this conversion depends on the transition of the MDP, which is not available to the learner. Therefore, the learner has to estimate the feature of every policy during the learning process.  
Previous work in this direction \citep{kong2023improved} faced obstacles in controlling the estimation error and was only able to show a $K^{\nicefrac{4}{5}} + \poly(\nicefrac{1}{\lambda_{\min}})$ regret bound assuming there exists an exploratory policy inducing a covariance matrix $\succeq \lambda_{\min}I$. We addressed the obstacles through 1) state space discretization (\pref{sec: discretization}), and 2) model-free estimation for the occupancy measure of policies over the discretized state space (\pref{sec: estimating mu and phi}). These allow us to emulate the success in the tabular case \citep{jin2020learning} and obtain the tight $\sqrt{K}$ regret. 

\paragraph{Efficient $K^{\nicefrac{3}{4}}$ algorithm. } The efficient algorithm is based on the policy optimization framework \citep{luo2021policy}. Different from previous works that all use exponential weights, we use Follow-the-Regularized-Leader (FTRL) with log-determinant (logdet) barrier regularizer to perform policy updates, which has the benefit of keeping the algorithm more stable \citep{zimmert2022return, liu2023bypassing}. We carefully combine logdet-FTRL with existing algorithmic/analysis techniques to further improve the regret bound. These include 1) an initial exploration phase to control the transition estimation error \citep{sherman2023rate}, 2) optimistic least-square policy evaluation in bonus construction \citep{sherman2023improved}, 3) dilated bonus construction \citep{luo2021policy}, and 4) a tighter concentration bound for covariance matrix estimation \citep{liu2023bypassing}. 

\subsection{Related Work}
In this subsection, we review prior works on adversarial MDPs and policy optimization. 

\paragraph{Learning in Adversarial MDPs. } Adversarial MDPs refer to a class of MDP problems where the transition is fixed while the loss function changes over time. Learning adversarial \emph{tabular} MDPs under bandit feedback and unknown transition has been extensively studied \citep{rosenberg2019online, jin2020learning, lee2020bias, jin2021best, shani2020optimistic, chen2021finding, luo2021policy, dai2022follow, dann2023best}. In this line of work, not only $\sqrt{K}$ regret bounds have been shown, several data-dependent bounds are also established.  
For adversarial MDPs with a large state space which necessitates the use of function approximation, $\sqrt{K}$ bounds have only been shown under simpler cases such as 1) full-information loss feedback \citep{cai2020provably, he2022near, sherman2023rate}, and 2) known transition or access to generative models / simulators \citep{neu2021online, dai2023refined, foster2022complexity}. Therefore, to our knowledge, we provide the first $\sqrt{K}$ regret for adversarial MDPs with large state spaces under bandit feedback and unknown transitions.\footnote{Although \cite{zhao2022learning} provided a $\sqrt{K}$ regret bound for linear mixture MDPs with bandit feedback and unknown transition, the polynomial dependence on the number of states prohibits its application to MDPs with large state spaces. } 
For linear MDPs, a series of recent work has made significant progress in improving the regret bound: \cite{luo2021policy}, \cite{dai2023refined}, \cite{sherman2023improved} proposed efficient (polynomial-time) algorithms with $K^{\nicefrac{14}{15}}$, $K^{\nicefrac{8}{9}}$, and $K^{\nicefrac{6}{7}}$ regret, respectively, and \cite{kong2023improved} proposed an inefficient algorithm with $K^{\nicefrac{4}{5}} + \poly(\nicefrac{1}{\lambda_{\min}})$ regret. Our $\sqrt{K}$ regret through an inefficient algorithm and $K^{\nicefrac{3}{4}}$ regret through an efficient algorithm further push the frontiers.

\textbf{Policy Optimization with Exploration. } Policy optimization has been regarded as sample inefficient due to its local search nature. Recently, efforts to alleviate this issue have incorporated exploration bonus in policy updates \citep{agarwal2020pc, shani2020optimistic, zanette2021cautiously, luo2021policy, dai2023refined, sherman2023improved, zhong2023theoretical, liu2023optimistic, sherman2023rate}. In the case of linear MDPs with a \emph{fixed} loss function, the state-of-the-art result is by \cite{sherman2023rate}, who provide a computationally efficient policy optimization algorithm with a tight $\sqrt{K}$ regret. In the case of linear MDPs with \emph{adversarial} losses, the best existing regret bound is $K^{\nicefrac{6}{7}}$ by \cite{sherman2023improved}, while we improve it to $K^{\nicefrac{3}{4}}$ in this paper. Beyond theoretical advancement, exploration in policy optimization has also showcased its potential in addressing real-world challenges, as evidenced by empirical studies \citep{burda2018exploration, pan2019policy}.

\section{Preliminaries}
\textbf{No-Regret Learning in MDPs.} An (episodic) MDP is specified by a tuple $\mathcal M=(\calS, \calA, P)$ where
$\calS$ is  the state space (possibly infinite),
$\calA$ is  the action space (assumed to be finite with size $A = \lvert \calA\rvert$),
$P \colon \calS\times \calA\to \Delta(\calS)$ is the transition kernal. 
The state space is assumed to be \textit{layered}, i.e., $\calS=\calS_1 \cup \calS_2 \cup \cdots \cup \calS_H$ where $\calS_{h}\cap \calS_{h'}=\varnothing$ for any $1\le h<h'\le H$, and transition is only possible from one layer to the next, that is, $P(s' \mid s,a) \neq 0$ only when $s \in \calS_h$ and $s' \in \calS_{h+1}$. 
Without loss of generality, we assume $\calS_1=\{s_1\}$.

We consider a process where the learner interact with the MDP for $K$ episodes, each time with a different loss function. Before the game starts, an adversary arbitrarily chooses the loss functions for all episodes $(\ell_{k}: \calS \times \calA \to [0,1])_{k=1}^K$, and does not reveal them to the learner. For each episode $k\in[K]$, the learner starts at state $s_{k,1} = s_1$; for each step $h\in [H]$ within episode $k$, after observing the state 
$s_{k,h}\in\calS_h$, the learner chooses an action $a\in \calA$, suffers and observes the loss $\ell_{k}(s_{k,h},a_{k,h})$, and transits to a new state $s_{k,h+1}$ sampled from the transition $P(\cdot \mid s_{k,h},a_{k,h})$.

A policy $\pi$ is a mapping from $\calS$ to $\Delta(\calA)$.
The \textit{state-value function} (or V-function in short) $V^{\pi}(s; \ell)$ is the cumulative loss starting from state $s$, following policy $\pi$ and under loss function $\ell$. This is formally defined as the following for $s\in\calS_h$: 
\begin{align*}
V^{\pi}(s; \ell)\triangleq \E \left [\sum_{h'=h}^H \ell(s_{h'},a_{h'})~\middle \vert~ s_h=s,\ \  a_{h'}\sim \pi(\cdot \mid s_{h'}),\ \  s_{h'+1}\sim P(\cdot \mid s_{h'},a_{h'}), \ \ \forall h'\geq h \right ].
\end{align*}
The \textit{action-value function} (\textit{a.k.a.} Q-function), on the other hand, is the expected loss suffered by a policy $\pi$ starting from a given state-action pair $(s,a)$. Formally, we define for all $(s,a)\in \cal{S} \times \calA$:
\begin{align}
{Q^\pi(s,a; \ell) = \ell(s,a) + \ind[s \notin \calS_H]\cdot \E_{\substack{s'\sim P(\cdot\mid s,a)}}\left[V^\pi(s'; \ell)\right].   \label{eq: Q-function}}
\end{align}
Let $\pi_k$ be the policy used by the learner in episode $k$. The learner aims to minimize the \emph{regret} with respect to the best fixed policy, defined as 
\begin{definition}[Regret]
$
\mathcal{R}_K  \triangleq \E\left [\sum_{k=1}^K V^{\pi_k}(s_1; \ell_k) \right ]- \min_{\pi}\sum_{k=1}^K V^{\pi}(s_1; \ell_k). 
$
\end{definition}

\paragraph{Occupancy measures. }
For a policy $\pi$ and a state $s$, we define $\mu^\pi(s)$ to be the probability of visiting state $s$ within an episode when following $\pi$, which can be written as $\mu^{\pi}(s) = V^{\pi}(s_1; \delta_s)$ with $\delta_s(s',a')=\ind\{s'=s\}$.  
Further define $\mu^{\pi}(s,a)=\mu^{\pi}(s)\pi(a|s)$.
By definition, we have $V^{\pi}(s_1; \ell) = \sum_{s\in\calS}\sum_{a\in\calA} \mu^{\pi}(s,a)\ell(s,a)$.\footnote{For readability, throughout the paper, we use summation over states instead of integration. Technically, all our results hold for case of continuous and infinite state space. } 

\subsection{Linear MDP}
%
Linear MDP is formally defined as follows. 
\begin{definition}[Linear MDP]\label{def: linear MDP}
In a \textit{linear MDP}, each state-action pair $(s,a)$ is associated with a known feature $\phi(s,a)\in \mathbb R^d$ with $\lVert \phi(s,a)\rVert_2\le 1$.
There exists a mapping $\psi \colon \calS\to \mathbb{R}^d$ such that the transition can be expressed as
\begin{align}
P(s'\mid s,a) &= \langle \phi(s,a),\psi(s')\rangle,\quad \forall (s,a,s')\in \bigcup_{h=1}^{H-1} \calS_{h}\times \calA\times \calS_{h+1}.    \label{eq: P assumption}
\end{align}
Here, $\psi$ is unrevealed to the learner. 
Moreover, for any episode $k\in[K]$ and any layer $h\in[H]$, there exists a (hidden) vector $\theta_{k,h}\in\mathbb{R}^d$ such that 
\begin{align}
    \ell_k(s,a) = \langle \phi(s,a), \theta_{k,h}\rangle, \quad \forall (s,a)\in \calS_h\times \calA.  \label{eq: Q assumption} 
\end{align}

Following previous work, we assume $\|\sum_{s\in\calS_h} |\psi(s)|\|_2\leq \sqrt{d}$ (the absolute value $|\cdot|$ over a vector is element-wise) and $\lVert \theta_{k,h}\rVert_2\le \sqrt{d}$ for all $k,h, \pi$.  
\end{definition}

We also define misspecifeid linear MDPs, which is used in \pref{sec: ineff alg}. 
\begin{definition}[Misspecified Linear MDP]\label{def: mis-linear MDP}
A $\zeta$-misspecified linear MDP follows all the assumptions in \pref{def: linear MDP} except that \pref{eq: P assumption} and \pref{eq: Q assumption} are respectively modified to 
\begin{align}
\left\|P(\cdot\mid s,a) - \langle \phi(s,a),\psi(\cdot)\rangle\right\|_1\leq \zeta
\qquad \text{and} \qquad 
\left|\ell_k(s,a) - \langle \phi(s,a),  \theta_{k,h}\rangle\right|\leq \zeta.     \label{eq: misspdcified MDP}
\end{align}
\end{definition}


\section{Rate-Optimal Algorithm}\label{sec: ineff alg}
The aim of this section is to show that there is no statistical barrier to obtaining $\sqrt{K}$ regret for linear MDPs with bandit feedback and adversarial losses. 
The proposed algorithm is computationally inefficient and it remains an open question if the same can be achieved with an efficient algorithm.

\subsection{Solution Ideas}

Observe that the expected loss of policy $\pi$ in episode $k$ can be written as $\sum_{s\in\calS}\sum_{a\in\calA} \mu^\pi(s,a) \ell_k(s,a)=$ $\sum_{h=1}^H \sum_{s\in\calS_h}\sum_{a\in\calA} \mu^\pi(s,a) \phi(s,a)^\top \theta_{k,h}$. 
This can be further written as 
$\langle \phi^\pi, \theta_k\rangle$, where
\begin{align*}
     \phi^\pi &= (\phi^\pi_1 ,\ldots, \phi^\pi_H), \quad \theta_{k} = (\theta_{k,1}, \ldots, \theta_{k,H}), \qquad \text{\ with\ \ } \phi^\pi_h = \sum_{s\in\calS_h}\sum_{a\in\calA}  \mu^\pi(s,a) \phi(s,a). 
\end{align*}
In other words, the adversarial linear MDP problem can be viewed as an adversarial linear bandit problem with $(\phi^\pi)_{\pi\in\Pi}$ as the underlying action set. Therefore, if computation is not an issue (i.e., if we are allowed to run linear bandits over an exponentially large action set), the only additional challenge in linear MDPs is that $(\phi^\pi)_{\pi\in\Pi}$ is not known in advance and the learner must learn the transition to estimate them.  This viewpoint has been taken by \cite{kong2023improved} to design computationally inefficient algorithms with improved regret bounds. To estimate $(\phi^\pi)_{\pi\in\Pi}$, \cite{kong2023improved} use an initial pure exploration phase to estimate $\phi^\pi$ up to an accuracy of $\epsilon$ for all $\pi$, and then run a $\epsilon$-misspecified linear bandit algorithm over policies in the second phase. Their approach gives $K^{\nicefrac{4}{5}}+\poly(\nicefrac{1}{\lambda_{\min}})$ regret. 

A natural idea to improve the regret bound is to estimate $(\phi^\pi)_{\pi\in\Pi}$ \emph{on the fly} instead of in a separate initial phase. That is, we directly start a linear bandit algorithm. Then during the learning process, for policies that are more often used by the learner, their $\phi^\pi$ estimation will become more and more accurate, and for others, larger error is allowed. Intuitively, this better balances exploitation and exploration because the learner will not spend too much efforts in estimating $\phi^\pi$ for bad policies. However, there are technical difficulties in doing so. Recall that $\phi^\pi_h=\sum_{s\in\calS_h} \sum_{a\in\calA} \mu^\pi(s)\pi(a|s)\phi(s,a)$. To estimate this, the learner needs to first estimate $\mu^\pi$. A natural estimator $\hat{\mu}^\pi$ would be defined recursively as $\hat{\mu}^\pi(s') = \sum_{s\in\calS_{h}}\sum_{a\in\calA}  \hat{\mu}^\pi(s)\pi(a|s) \hat{P}(s'|s,a)$ for $s'\in\calS_{h+1}$, with the transition estimator $\hat{P}$ obtained from linear regression:
$\hat{P}(s'|s,a) = \phi(s,a)^\top \left(\Lambda_h^{-1}\sum_{(\tilde{s}, \tilde{a}, \tilde{s}')\in\calD_h} \phi(\tilde{s}, \tilde{a})\ind\{\tilde{s}'=s'\}  \right)$ where $\calD_h$ consists of historical data of the form $(s, a, s')\in\calS_h\times \calA\times \calS_{h+1}$ and $\Lambda_h=I+\sum_{(s, a, s')\in\calD_h} \phi(s, a) \phi(s, a)^\top$. 
This is the exact idea of \cite{kong2023improved}. Notice that the $\hat{\mu}^\pi$ obtained in this way may not be \emph{valid}, i.e., they may not satisfy $\hat{\mu}^\pi(\cdot)\in\Delta(\calS)$. Their approach suffers from the issue that it is difficult to control the magnitude of $\hat{\mu}^\pi(s)$ when the amount of data in $\calD_h$ is still small. This is why they use an initial phase to explore all directions in the feature space and control the error $\|\hat{\phi}_h^\pi - \phi^\pi_h\|$ uniformly for all policies.

However, ``on-the-fly estimation'' without the initial phase has been proven to work in the tabular case \citep{jin2020learning} to get a $\sqrt{K}$ regret.  The key difference between the tabular case and the linear case is that the transition estimator $\hat{P}$ in the tabular case is always a valid transition (i.e., $\hat{P}(\cdot|s,a)\in\Delta(\calS)$), and thus the induced occupancy measure estimator $\hat{\mu}^\pi$ is also always valid. This avoids the aforementioned technical difficulty. 

With this observation, we propose to incorporate the constraint that $\hat{\mu}^\pi$ be a valid occupancy measure when dealing with linear MDPs. To find such a $\hat{\mu}^\pi$, we search over the space of valid occupancy measures and pick one that is consistent with the past data. This is different from the approach of \cite{kong2023improved}, where $\hat{P}$ is obtained via linear regression over the past data first, and then $\hat{\mu}^\pi$ is derived from it, which can fail to be valid.   

Since the state space and policy space can both be infinite, in order to get a runnable algorithm for finding $\hat{\mu}^\pi(s)$, we discretize both the state space and the policy space. These are described in the next subsection.  

\subsection{The Discretization Procedures}\label{sec: discretization}

\paragraph{Discretization of the state space. }
For linear MDPs, we can assume that a state $s$ is uniquely defined by its action feature set $\calA_s=\{\phi(s,a)\,|\,a\in\calA\}$.
If there are distinct states with identical feature sets, we can collapse them into a single state by combining their $\psi(s)$.

In order to approximate an infinite-state linear MDP as a finite-state MDP, we perform discretization for the entire feature space $\mathbb{B}^d(1)$. To decide the discretization resolution, assume that $\phi(s,a)$ is the true feature and $\phi'(s,a)$ is its approximation, and $\|\phi(s,a) - \phi'(s,a)\|_2 \leq \epsilon$ for all $s,a$. Then we have $\|P(\cdot|s,a) - \langle \phi'(s,a), \psi(\cdot)\rangle\|_1 = \|\langle \phi(s,a) - \phi'(s,a), \psi(\cdot)\rangle\|_1\leq \sum_{s'}\|\phi(s,a)-\phi'(s,a)\|_2 \|\psi(s')\|_2\leq \epsilon\sum_{s'}\|\psi(s')\|_2 \leq \epsilon \sum_{i=1}^d\sum_{s'}|\psi_i(s')|\leq \epsilon \sqrt{d} \|\sum_{s'}|\psi(s')|\|_2\leq \epsilon d$ and $|\ell_k(s,a) - \langle \phi'(s,a), \theta_{k,h}\rangle| = |\langle \phi'(s,a)-\phi(s,a), \theta_{k,h}\rangle|\leq \|\phi'(s,a)-\phi(s,a)\|_2 \|\theta_{k,h}\|_2 \leq \epsilon \sqrt{d}$ by \pref{def: linear MDP}. Thus, the MDP with $\phi'(s,a)$ as the underlying feature is a misspecified linear MDP with misspecification error $\zeta=\epsilon d$ by \pref{def: mis-linear MDP}. It turns out that it suffices to set $\epsilon=\frac{1}{K}$ and make the misspecification error $\zeta=\frac{d}{K}$. The number of states after the discretization is upper bounded by (size of $\epsilon$-net of the feature space)$^{A}=(1/\epsilon)^{\order(dHA)}=K^{\order(dHA)}$. 

There is a caveat when working with this discretized state space. Since the true feature space $\Phi=\{\phi(s,a):~ s\in\calS, a\in\calA\}$ may not cover the entire $\mathbb{B}^d(1)$, the state space construction above (i.e., by discretizing the whole $\mathbb{B}^d(1)$) may produce states that do not really exist. In fact, there is no problem viewing these non-existing states as part of the state space because their $\psi(s)$ can be set to zero, making them unreachable under the linear MDP assumption. The only thing we have to be careful about is that the assumptions \pref{eq: P assumption}, \pref{eq: Q assumption}, \pref{eq: misspdcified MDP}, and their implications, such as $-\zeta\leq \langle \phi, \psi(s') \rangle\leq 1+\zeta$ and $|\langle \phi, \theta_{k,h} \rangle|\leq 1+\zeta$, are only guaranteed for $\phi$ in the \emph{true feature space} $\Phi$, but not for the whole feature space $\mathbb{B}^d(1)$. To avoid ambiguity, we use notation $\calS$ to denote the set of discretized states from the \emph{true} MDP, and use $\calX$ to denote the set of discretized states constructed from the entire $\mathbb{B}^d(1)$. Apparently, $\calS\subseteq \calX$. We clarify that, 1) the learner knows $\calX$, but does not know $\calS$ before interacting with the environment, 2) the misspecified linear MDP assumption \pref{eq: misspdcified MDP} is only guaranteed for $\phi(s,a)$ with~$s\in\calS$, 3) $\calX\setminus \calS$ are unreachable states and their $\psi(s)$ are set to zero. We use $(\calX_h)_{h\in[H]}$ to denote partitions of $\calX$ on different layers.

\paragraph{Discretization of the policy space. } 
We consider a discretization of the policy space for \pref{alg: exponetial weights}.
The policy class is the set of linear policies defined as 
\begin{align}
     \Pi = \left\{\pi_\theta:~ \theta\in\Theta^H, ~~\pi_{\theta}(s) = \argmin_{a\in\calA} \phi(s,a)^\top \theta_h \text{\ for \ } s\in \calX_h\right\}   \label{eq: definition of Pi}
\end{align}
where $\Theta$ is an $1$-net of $\mathbb{B}^{d}(K)$.
The next lemma shows that this policy set contains a near optimal one. See \pref{app: policy disc} for the proof. 

\begin{lemma}\label{lem: policy discretization}
For any policy $\pi:\calX\rightarrow \Delta(\calA)$ and any sequence of losses $(\theta_{k,h})_{h\in[H], k\in[K]}$, there exists a policy $\pi'\in\Pi$ such that
$
\sum_{k=1}^K\sum_{h=1}^H \sum_{s\in\calS_h} \sum_{a\in\calA}(\mu^{\pi'}(s,a)-\mu^{\pi}(s,a))\phi(s,a)^\top \theta_{k,h} \leq \sqrt{d}H^2\,.$
\end{lemma}

\subsection{Estimating $\mu^\pi(s)$}\label{sec: estimating mu and phi}
With the state space discretized, we are now faced with a finite state problem. To estimate $\mu^\pi$, a potential way is to find a transition estimation $(\hat{P}(s'|s,a))_{s,a,s'}$ which is consistent with the historical data and satisfies the constraint that the $\hat{\mu}^\pi$ induced by $\hat{P}$ is a valid occupancy measure. The issue of this is that since $P(s'|s,a)\approx\phi(s,a)^\top \psi(s')$, this method requires us to estimate $\psi(s')$ for all $s'$, whose complexity will scale with $|\calS|$ because $\psi(s')$ for different $s'$ are unrelated. Indeed, as noted by previous works \citep{foster2022note}, the linear MDP model does not allow efficient model-based estimation. 

Inspired by previous model-free approaches for linear MDPs \citep{jin2020provably}, instead of estimating $\psi(s')$, we will directly estimate $\sum_{s'} \psi(s')f(s')$ for a class of functions $f$ that is rich enough for our purpose (i.e., to estimate $(\phi^\pi)_{\pi\in\Pi}$ well). This class of functions turns out can be chosen as $\bigcup_{\pi\in\Pi} \calF^\pi$ where $ \calF^\pi=\calF^\pi_1\cup \calF^\pi_2$ and 
\begin{align}
    \calF_1^\pi &= \Bigg\{f: \calX\rightarrow [-1,1]\ \ ~\bigg|~ 
    \ \ f(s) = 
    \sum_{a\in\calA} \pi(a|s) \clip\left[\phi(s,a)^\top \theta\right] \text{\ for some} \ \  \theta\in \mathbb{B}^d(\sqrt{d})
    \Bigg\},    \nonumber \\
    \calF_2^\pi &= \Bigg\{f: \calX\rightarrow [-1,1]\ \ ~\bigg|~ 
      \ \   f(s)=\sum_{a\in\calA}\pi(a|s) \|\phi(s,a)\|_\Gamma  \text{\ for some\ } \Gamma \text{\ with\ } \mathbf{0} \preceq \Gamma \preceq I
    \Bigg\},   \label{eq: def of Fpi} 
\end{align}
where we define $\clip[a]=\max(\min(a,1), -1)$. Given historical data $(\calD_h)_{h=1}^H$ which consists of $(s,a,s')$ tuples, our way of obtaining $\hat{\mu}^\pi$ is summarized in \pref{alg: estimate mu}. 
\begin{algorithm}[t]
    \caption{$\estimate(\pi$, $(\calD_h)_{h=1}^H)$\ \  (\textbf{Est}imate \textbf{O}ccupancy \textbf{M}easure)} \label{alg: estimate mu}
    \textbf{Input}: target policy $\pi$, historical data $(\calD_h)_{h=1}^H$ where $\calD_h$ consists of tuples $(s,a,s')\in\calS_h\times \calA\times \calS_{h+1}$ with $s'\sim P(\cdot|s,a)$. \\[3pt]
   Find $(\hat{\mu}^\pi(s))_{s\in\calX}\subset[0,1]$ and  $(\hat{\xi}_{h,f})_{h\in[H], f\in\calF^\pi}\subset\mathbb{B}^d(\sqrt{d})$ that satisfy the following for all $h\in[H]$ and all $f\in\calF^\pi$ (recall the definition of $\calF^\pi$ in \pref{eq: def of Fpi}, and $\zeta$ in \pref{sec: discretization}). 
\begin{align}
    & \scalebox{0.95}{$\displaystyle\sum_{s\in\calX_h}\hat{\mu}^\pi(s)=1$},   \label{eq: const 1}\\
    &\scalebox{0.95}{$\displaystyle\Bigg|\sum_{s'\in\calX_{h+1}}\hat{\mu}^{\pi}(s')f(s') - \sum_{s\in\calX_h}\sum_{a\in\calA} \hat{\mu}^{\pi}(s)\pi(a|s)\clip\left[\phi(s,a)^\top \hat{\xi}_{h,f}\right]\Bigg| \leq \zeta$} \label{eq: const 2} \\
    &\scalebox{0.95}{$\displaystyle\sum_{(s,a,s')\in\calD_h} \big(f(s') -\phi(s,a)^\top \hat{\xi}_{h,f}\big)^2  -  \min_{\xi\in\mathbb{B}^d(\sqrt{d})}\sum_{(s,a,s')\in\calD_h}\left(f(s') - \phi(s,a)^\top \xi\right)^2 \leq 16d^{\frac{5}{2}}\log\frac{18d^\frac{3}{2}K}{\delta}$} \label{eq: const 4} 
\end{align}
\textbf{Output}: $(\hat{\mu}^\pi(s))_{s\in\calX}$ (if \pref{eq: const 1}-\pref{eq: const 4} is not feasible, output any solution that satisfies \pref{eq: const 1}).  
\end{algorithm}
In \pref{alg: estimate mu}, \pref{eq: const 1} sets the constraint that $\hat{\mu}^\pi$ is a valid occupancy measure, \pref{eq: const 4} requires that $\hat{\xi}_{h,f}$ approximates $\xi_{h,f}^\star=\sum_{s'\in\calS_{h+1}}\psi(s')f(s')$ well on the historical data $(\calD_h)_{h=1}^H$, and \pref{eq: const 2} relates $\hat{\mu}^\pi$ with $\hat{\xi}_{h,f}$ according to their definitions. In the following \pref{lem: feasibility}, we show that \pref{eq: const 1}-\pref{eq: const 4} is feasible with high probability. Then in \pref{lem: value diff linear case}, we show the key property that $\hat{\mu}^\pi$ is close to $\mu^\pi$ when evaluated on any $f\in\calF^\pi$. The proofs of \pref{lem: feasibility} and \pref{lem: value diff linear case} can be found in \pref{app: feature estimation}. 
Below, we define $\hat{\mu}^\pi(s,a):=\hat{\mu}^\pi(s)\pi(a|s)$.

\begin{lemma}\label{lem: feasibility}
    With probability at least $1-\frac{\delta}{K}$, \pref{eq: const 1}-\pref{eq: const 4} is feasible for all $\pi\in\Pi$.  
\end{lemma}

\begin{lemma}\label{lem: value diff linear case}
   Let $(\hat{\mu}^\pi(s))_{s\in\calX}$ be the output of \pref{alg: estimate mu}. Then with probability at least $1-\frac{\delta}{K}$, for any $\pi\in\Pi$ and all $f\in\calF^\pi$, $\left|\sum_{s\in\calX_h} (\hat{\mu}^\pi(s) - \mu^\pi(s))f(s)\right|$ is upper bounded by 
   \begin{align*}
      10d^{\frac{5}{4}}\sqrt{\log\frac{18d^\frac{3}{2}K}{\delta}} \times \sum_{h'<h}\min\left\{\sum_{s\in\calX_{h'}}\sum_{a\in\calA} \mu^{\pi}(s,a)\|\phi(s,a)\|_{\Lambda_{h'}^{-1}}, \sum_{s\in\calX_{h'}}\sum_{a\in\calA} \hat{\mu}^{\pi}(s,a)\|\phi(s,a)\|_{\Lambda_{h'}^{-1}}\right\} + 2\zeta H
   \end{align*}
   where $\Lambda_h:=I+\sum_{(s,a,s')\in\calD_h} \phi(s,a)\phi(s,a)^\top$. 
\end{lemma}




   

\subsection{Algorithm: Exponential Weights}
From \pref{sec: estimating mu and phi}, we know how to obtain the estimation for $(\mu^\pi)_{\pi\in\Pi}$. Now we can use them to construct estimators of $(\phi^\pi)_{\pi\in\Pi}$ via $\hat{\phi}^\pi_h=\sum_{s\in\calX_h}\sum_{a\in\calA}\hat{\mu}^\pi(s)\pi(a|s)\phi(s,a)$, and run a linear bandit algorithm viewing $(\hat{\phi}^\pi)_{\pi\in\Pi}$ as actions. The algorithm is presented in \pref{alg: exponetial weights}. At the beginning of each episode $k$, we call $\estimate$ (\pref{alg: estimate mu}) for all policies with the data up to episode $k-1$ (\pref{line: call est all pi}). This returns the occupancy measure estimator $\hat{\mu}^\pi_k$ for all $\pi$, which we can use to construct the feature estimator $\hat{\phi}^\pi_k$. Then we use the standard exponential weight together with John's exploration to update the distribution over policies. 
To deal with the bias induced by the estimation error of $\hat{\phi}^\pi_k$, we incorporate a bonus term $b_k^\pi$ in the update. Similar ideas have also been used in, e.g., \cite{luo2021policy, sherman2023improved, dai2023refined, kong2023improved, liu2023bypassing}. We defer the regret analysis of this algorithm to \pref{app: regret analysis for ineff}, and only state the final guarantee in the next theorem. 

\begin{theorem}\label{thm: inefficient bound}
    The regret of \pref{alg: exponetial weights} is bounded by $\calR_K\leq \otil(\sqrt{d^7H^7K})$. 
\end{theorem}

\begin{algorithm}[H]
    \caption{Exponential Weights}
    \label{alg: exponetial weights}
    \begin{algorithmic}[1]
    
    \STATE Let $\Pi$ be the policy set defined in \pref{eq: definition of Pi}.
    Let $\gamma = \min\big\{d^2H^{\frac{1}{2}}K^{-\frac{1}{2}}, \frac{1}{2}\big\}$, $\eta=\frac{\gamma}{2dH}$. \\
    \STATE  For all $h\in[H]$, $\calD_{1,h}\leftarrow \emptyset$, $\Lambda_{1,h}\leftarrow I$. 
    \FOR{$k=1, 2, \ldots$}
       \STATE     For all $\pi\in\Pi$, let $\hat{\mu}_k^\pi=\estimate(\pi, (\calD_{k,h})_{h=1}^H)$ (call \pref{alg: estimate mu}).  
       \STATE Define $\hat{\phi}^\pi_{k,h}=\sum_{s\in\calX_h}\sum_{a\in\calA} \hat{\mu}^\pi(s)\pi(a|s)\phi(s,a)$ and $\hat{\phi}_k^\pi=(\hat{\phi}^\pi_{k,1}, \ldots,  \hat{\phi}^\pi_{k,H})$.  \label{line: call est all pi}
       \STATE     Compute $q_k\in\Delta(\Pi)$ as $
              q_k(\pi) \propto \exp\left(-\eta \sum_{i=1}^{k-1} \left( 
\hat{\phi}^{\pi^\top}_i \hat{\theta}_i - b_i^\pi \right) \right)$.   \label{line: exponential weight}
       \STATE     Let $q_k' = (1-\gamma)q_k + \gamma J_k$ where $J_k\in\Delta(\Pi)$ is John's exploration over $\{\hat{\phi}^\pi_k\}_{\pi\in\Pi}$.
        \STATE    Sample $\pi_k\sim q_k'$, execute $\pi_k$, and obtain trajectory $(s_{k,1}, a_{k,1}, \ell_{k,1}, \ldots, s_{k,H}, a_{k,H}, \ell_{k,H})$. 
       \STATE    Define for $\Cbonus = 10d^{\frac{5}{4}}H\sqrt{\log\frac{18d^\frac{3}{2}K}{\delta}}$,
           \begin{align*}  
               M_k &=  
                \sum_{\pi\in\Pi}q_k'(\pi) \hat{\phi}_k^\pi  (\hat{\phi}_k^\pi)^\top, \qquad 
               \hat{\theta}_k = M_k^{-1} \hat{\phi}_k^{\pi_k}L_k, \qquad \text{where}\ \ L_k=\sum_{h=1}^H \ell_{k,h}, \\
               b_k^\pi &= \Cbonus\sum_{h=1}^H\sum_{s\in\calX_h}\sum_{a\in\calA}\hat{\mu}_k^\pi(s,a)\|\phi(s,a)\|_{\Lambda_{k,h}^{-1}} + \eta\|\hat{\phi}_k^\pi\|_{M_k^{-1}}^2.   
           \end{align*}
       \STATE    For all $h\in[H]$,  
           \begin{align*}
               \calD_{k+1,h}\leftarrow \calD_{k,h}\cup\{(s_{k,h}, a_{k,h}, s_{k,h+1})\}, \ \ \Lambda_{k+1,h}\leftarrow \Lambda_{k,h} + \phi(s_{k,h}, a_{k,h})\phi(s_{k,h}, a_{k,h})^\top
           \end{align*}
        \ENDFOR
    \end{algorithmic}
\end{algorithm}

\section{Computationally Efficient Policy Optimization Algorithm}
\label{sec:efficient}
In \pref{alg: exponetial weights}, we convert the linear MDP problem to a linear bandit problem. It is generally hard to ensure computational efficiency in this paradigm due to the non-linear mapping of policy to occupancy measure and the exponential size of the policy space. 
A promising alternative is to use the policy optimization framework \citep{luo2021policy, dai2023refined, sherman2023improved}, which allows to run a Follow-the-Regularized-Leader (FTRL) algorithm over the locally available state-action feature set. 
An algorithm of this type needs to overcome several hurdles: 1) The algorithm needs to construct loss estimates with carefully controlled bias, which is difficult because the learner does not know the feature covariance matrix under the current policy (required in the constructing a standard unbiased loss estimator), and has to estimate it. 2) The algorithm needs to inject bonus to ensure sufficient exploration. These bonus terms not only need to compensate the uncertainty in transitions, but also the bias induced in loss estimates mentioned in the previous item. The bonus \emph{itself} needs to be estimated and induces more bias due to the estimation error. 3) Since policy optimization behaves like a layered bandit over bandit algorithm, the algorithm needs to construct bonus terms accumulated over layers. Specifically, the bonus in earlier layers need to additionally compensate the bias of the bonus terms in later layers, as mentioned in the previous item. 4) The algorithm needs to ensure that the magnitudes of loss estimates and bonuses are small enough for the FTRL-based algorithm.

These challenges are fully exposed in the  \emph{adversarial loss}, \emph{bandit feedback}, \emph{unknown transition} setting, because in this case the loss estimators usually have larger magnitudes and necessitate larger bonuses. This make achieving near-optimal bounds 
difficult, and the current best regret is $\otil(K^{6/7})$ by \cite{sherman2023improved}. We successfully improve it to $\otil(K^{3/4})$ by several improved design choices, which we describe in the following.

\begin{algorithm}[!ht]
    \caption{Logdet FTRL with initial exploration}
    \label{alg: logdet FTRL} 
	\begin{algorithmic}[1]
        \STATE \textbf{Parameters:} $\eta = \frac{1}{3328\sqrt{d}H^2}K^{-\frac{1}{4}}, \ \ \gamma = 5d\log\left(6dHK^4\right)K^{-\frac{1}{2}}, \ \ \beta = \sqrt{d}K^{-\frac{1}{4}}, \ \ \alpha = HK^{\frac{3}{4}}$,   $\ \ \tau = K^{\frac{1}{2}}$, \ \ $\delta=K^{-3}$, \ \  $\rho = H^{-\frac{1}{2}}d^{-\frac{1}{4}}K^{-\frac{1}{4}}, \ \  \epsilon_{\rm cov} = K^{-\frac{1}{4}}$. \\[5pt]
        \STATE \textbf{Define: } $\widehat{\Cov}(s,p) = \E_{a \sim p}\begin{bmatrix} \phi(s,a)\phi(s,a)^\top & \phi(s,a)\\ \phi(s,a)^\top & 1    
\end{bmatrix}$
	    \STATE Run \pref{alg:reward_free} with parameters $\delta, \rho, \epsilon_{\rm cov}$, which ends within $K_0 = \otil( d^{\frac{3}{2}}H^2K^{\frac{3}{4}} + d^4H^4K^{\frac{1}{4}})$ episodes with high probability. Receive outputs $(\calD_{0,h})_{h=1}^H$ and $(\calZ_h)_{h=1}^H$.   \label{line: initial phase in alg}
	    \FOR{$j = 1, \ldots, \lceil (K - K_0)/(2\tau) \rceil$}
      
             \STATE For $s \in \calS_h$, define  
            \begin{equation*}
                \pmb{\widetilde{H}}_{{j}}(s) = \argmin_{\pmb{H} \in \calH_{s}} \left\{ \left\langle \pmb{H}, \sum_{i=1}^{j-1}  \pmb{\calL}_{i,h} \right\rangle + \frac{F(\pmb{H})}{\eta} \right\},  \,\text{where }
                \pmb{\calL}_{i,h} 
                = \frac{1}{2\tau} \sum_{k\in T_i} \left(\pmb{\hGamma}_{k,h} - \pmb{\widehat{B}}_{k,h} \right)
            \end{equation*}
            where $\calH_s = \left\{\widehat{\Cov}(s,p): p \in \Delta(\calA) \right\}$ and $F(\pmb{H}) = -\log\det\left(\pmb{H}\right)$.   \label{line: FTRL line}
            \STATE Let $\widetilde{\pi}_{j}(\cdot|s)$ be such that $\pmb{\widetilde{H}}_j(s)  = \widehat{\Cov}(s,\widetilde{\pi}_j(\cdot|s))$.    \label{line: convert cov to pi}
            
            \STATE Let $T_{j} = \{(j-1)\tau + K_0 + 1, \cdots, (j+1)\tau + K_0\}$. 
            Execute $\pi_k = \widetilde{\pi}_j$ for the $2\tau$ episodes $k \in T_j$, and collect $(s_{k,h}, a_{k,h}, \ell_{k,h})_{h \in [H], k \in T_j}$.   \label{line: execute tilde j}
            \STATE Let $T_{j,1}$ and $T_{j,2}$ be the first $\tau$ and the last $\tau$ episodes in $T_j$, respectively. For all $k\in T_j$ and $h\in[H]$, define
            \begin{align}
                 \textstyle
                 \calC_{k,h} &= \begin{cases}
                       \{(s_{k',h}, a_{k',h}, s_{k', h+1})\}_{k' \in T_{j,2}} &\text{if\ } k\in T_{j,1} \\
                       \{(s_{k',h}, a_{k',h}, s_{k', h+1})\}_{k' \in T_{j,1}} &\text{if\ } k\in T_{j,2} 
                 \end{cases}    \label{eq: construction 1}\\
                \widehat{\Sigma}_{k,h}  &=  \textstyle\gamma I +  \frac{1}{\tau} \sum_{(s, a, s') \in \calC_{k,h}}\phi(s, a) \phi(s, a)^\top  \label{eq: construction 2}  \\
                 \hq_{k,h} &= \widehat \Sigma^{-1}_{k,h} \phi(s_{k,h}, a_{k,h})\textstyle \sum_{t=h}^H \ell_{k,t}  \label{eq: construction 3} \\
                 \pmb{\hGamma}_{k,h}  &=  \begin{bmatrix} 
                0 & \frac{1}{2}\hq_{k,h}  \\
                \frac{1}{2}(\hq_{k,h})^\top & 0
                \end{bmatrix}  \label{eq: construction 4} \\
                \calD_{k,h} &= \calD_{k-1,h} \cup \{(s_{k,h}, a_{k,h}, s_{k, h+1})\} \\
                (\pmb{\widehat{B}}_{k,h})_{h=1}^H &= \text{\rm OBME}\left((\calD_{k,h})_{h=1}^H, (\widehat{\Sigma}_{k,h})_{h=1}^H, (\calZ_h)_{h=1}^H\right) \label{eq: construction 5} 
            \end{align}
            (OBME is presented in \pref{alg: dilated bonus})
            \label{line: update in epoch k}
           
           
        \ENDFOR
	\end{algorithmic}
\end{algorithm}

\begin{algorithm}[!ht]
    \caption{ OBME$\left((\calD_{k,h})_{h=1}^H, (\widehat{\Sigma}_{k,h})_{h=1}^H, (\calZ_h)_{h=1}^H\right)$ (\textbf{O}ptimistic \textbf{B}onus \textbf{M}atrix \textbf{E}stimation) }
    \label{alg: dilated bonus}
	\begin{algorithmic}[1]
            \STATE Parameters $\beta, \alpha, \gamma, \rho$ are the same as those in \pref{alg: logdet FTRL}. 
	    
	    \FOR {$h=H, \ldots, 1$}
             \STATE $B_h^{\max} = 4H\left(1+\frac{1}{H}\right)^{2(H-h+1)}\left(\frac{\beta}{\gamma} + \alpha\rho^2\right)$ 
             \STATE $\Lambda_{k,h} = I + \sum_{(s,a,s') \in \calD_{k,h}} \phi(s, a)\phi(s, a)^\top$
           

    
    \STATE Set $\widehat{w}_{k,h} = \left(1+\frac{1}{H}\right)\Lambda_{k,h}^{-1}\sum_{(s,a,s') \in \calD_{k,h}} \phi(s,a)\widehat{W}_{k}(s')\ind\{s' \in \calZ_{h+1}\}$ \ \ \ (if $h=H$, set $\widehat{w}_{k,h}=0$)\label{line: regression}
      \STATE Define $\pmb{\widehat{B}}_{k,h} =  \begin{bmatrix}
 \beta \hatSigma^{-1}_{k,h} +  \alpha \Lambda_{k,h}^{-1} &  \frac{1}{2} \widehat{w}_{k,h}
\\ \frac{1}{2} \widehat{w}_{k,h}^{\top} &  0
\end{bmatrix}$   

            \STATE For $s\in\calS_h$, define $\widehat{B}_{k}(s, a) 
			= \beta\|\phi(s,a)\|_{\hatSigma^{-1}_{k,h}}^2 +  \alpha \|\phi(s,a)\|_{\Lambda_{k, h}^{-1}}^2 + \phi(s, a)^\top \widehat{w}_{k,h} $     \label{line: hat B def}
   
            \STATE For $s\in\calS_h$, define $\widehat W_{k} (s) 
			= \langle \pi_{k}(\cdot|s), \widehat{B}_{k}^+(s, \cdot) \rangle $ where $\widehat{B}_{k}^+(s, a)  = \max\left\{\widehat{B}_{k}(s, a) , 0\right\}$ \label{line: B+}

        \ENDFOR 
        \STATE \textbf{return} $(\pmb{\widehat{B}}_{k,h})_{h \in [H]}$
	\end{algorithmic}
\end{algorithm}

Our algorithm (\pref{alg: logdet FTRL}) starts with an initial pure exploration phase that lasts for $K_0=\otil(K^{\frac{3}{4}})$ episodes (\pref{line: initial phase in alg}), which is crucial in controlling the magnitude of the bonus estimate (will be explained later). In the remaining $K-K_0$ episodes, episodes are divided into $\lceil(K-K_0)/(2\tau) \rceil$ epochs (indexed by~$j$), such that in each epoch $j$, a fixed policy $\widetilde{\pi}_j$ is executed for $2\tau$ episodes, and 
policies are updated only at the end of each epoch. The goal of dividing episodes into epochs is to let the learner collect sufficient samples and create accurate enough loss estimators for each update. Different from previous work \citep{luo2021policy, dai2023refined, sherman2023improved} that use exponential weights, we use the Follow-the-Regularized-Leader (FTRL) framework with logdet-barrier as the regularizer for policy updates. Logdet has been recently shown in adversarial linear (contextual) bandit to lead to a more stable update and can handle larger magnitude of the loss estimator bias \citep{zimmert2022return, liu2023bypassing}. It has similar benefits in our case as well. 

Specifically, with logdet-FTRL, the optimization of the policy on state~$s$ is over the space of \emph{lifted covariance matrix} $\calH_s = \left\{\widehat{\Cov}(s,p): p \in \Delta(\calA) \right\} \subset \mathbb{R}^{(d+1)\times(d+1)}$, where $\widehat{\Cov}(s,p) = \E_{a \sim p}\begin{bmatrix} \phi(s,a)\phi(s,a)^\top & \phi(s,a)\\ \phi(s,a)^\top & 1    
\end{bmatrix}$. In epoch $j$, for state $s$, the FTRL outputs a matrix $\pmb{\widetilde{H}}_j(s) \in \calH_{s}$ (\pref{line: FTRL line}), and the policy $\widetilde{\pi}_j(\cdot|s)$ is chosen such that $\pmb{\widetilde{H}}_j(s)  = \widehat{\Cov}(s,\widetilde{\pi}_j(\cdot|s))$ (\pref{line: convert cov to pi}). This policy is then executed for $2\tau$ episodes (\pref{line: execute tilde j}). Then the learner uses the collected samples to construct loss estimators for all episodes $k\in T_j$ (the $\hq_{k,h}$ in \pref{eq: construction 3}), where $T_j$ is the set of episodes in epoch $j$. This follows the standard loss estimator construction for linear bandits, except that in our case, the covariance matrix is unknown and also needs to be estimated using samples (the $\widehat{\Sigma}_{k,h}$ in \pref{eq: construction 2}). The validity of $\hq_{k,h}$ relies on the independence between $\widehat{\Sigma}_{k,h}$ and the loss obtained in episode $k$. To achieve this, we divide the set $T_j$ into two equal parts $T_{j,1}$ and $T_{j,2}$ (\pref{line: update in epoch k}). Then we use samples from $T_{j,2}$ to estimate the covariance matrix when constructing the loss estimator in episode $k\in T_{j,1}$, and vice versa (\pref{eq: construction 1}-\pref{eq: construction 3}).  
In \pref{eq: construction 4}, we further lift the loss estimator $\hq_{k,h}$ to $\pmb{\widehat{\Gamma}}_{k,h}\in\mathbb{R}^{(d+1)\times(d+1)}$ to be fed to FTRL. Finally, besides feeding the loss $\pmb{\widehat{\Gamma}}_{k,h}$, we also need to feed the \emph{bonus} $\pmb{\widehat{B}}_{k,h}$ required for sufficient exploration in policy optimization and to compensate the loss estimator bias coming from the estimation error of $\widehat{\Sigma}_{k,h}$. This is explained in the next subsection. 

\subsection{The Exploration Bonus}
\label{sec:bonus and regret}
Similar to previous work on policy optimization in adversarial linear MDPs \citep{luo2021policy, dai2023refined, sherman2023improved}, we use \emph{exploration bonus} to address the bias in the loss estimator $\hq_{k,h}$ and the stability term coming from the FTRL regret analysis. From a high level, the exploration bonus serves a similar purpose as ``optimism in the face of uncertainty'' as commonly used in the non-adversarial case, but now the sources of uncertainty additionally include the bias and the stability term. From a mathematical analysis perspective, the exploration bonus creates an effect of \emph{change of measure} that prevent the regret to depend on the distribution mismatch coefficient between the optimal policy and the learner's policy. This perspective is best explained in Section~3 of \cite{luo2021policy}. According to the analysis of \cite{luo2021policy}, when performing policy update on state $s\in\calS_h$, we should incorporate a bonus that is roughly of order $Q^{\pi_k}(s,a;b_t)$ where $b_t(s,a)=\beta \|\phi(s,a)\|_{\hatSigma_{k,h}^{-1}}^2$. 

Our bonus construction further incorporates the improvement from \cite{sherman2023improved} where an optimistic least-square policy evaluation (OLSPE) is used to fit the bonus (rather than sampling the bonus as in \cite{luo2021policy}). This creates another term of $\alpha\|\phi(s,a)\|^2_{\Lambda_{k,h}^{-1}}$ to be incorporated into the bonus to compensate the estimation error of future bonuses. Finally, we further adopt a technique developed in \cite{luo2021policy} called \emph{dilated bonus} to simplify our analysis. Overall, the bonus we use for the policy update on state $s\in\calS_h$ is defined recursively as 
\begin{align*}
    B_k(s,a) \approx \left(\beta \|\phi(s,a)\|_{\hatSigma_{k,h}^{-1}}^2 + \alpha\|\phi(s,a)\|^2_{\Lambda_{k,h}^{-1}}\right) + \left(1+\frac{1}{H}\right)\E_{s'\sim P(\cdot|s,a)} \E_{a'\sim \pi_k(\cdot|s')} \left[B_k(s',a')\right].  
\end{align*}
Notice that because of the dilation factor $(1+\frac{1}{H})$ \citep{luo2021policy}, this deviates from a standard Bellman equation. 
Recall that we run FTRL in the space of covariance matrix, so we would like to write $B_k(s,a)$ as a \emph{linear} function in that space. Fortunately, this is indeed possible because by the linear MDP structure, we can write the above as 
\begin{align}
     B_k(s,a) \approx \left\langle \begin{bmatrix}
         \phi(s,a)\phi(s,a)^\top   & \phi(s,a) \\
         \phi(s,a)^\top & 1
     \end{bmatrix}, \ \ 
     \begin{bmatrix}
         \beta \hatSigma_{k,h}^{-1} + \alpha \Lambda_{k,h}^{-1}  & \frac{1}{2}w_{k,h} \\
         \frac{1}{2}w_{k,h} & 0
     \end{bmatrix}
     \right\rangle  \label{eq: matrix form}
\end{align}
where $w_{k,h}= (1+\frac{1}{H})\sum_{s'\in\calS_{h+1}}\psi(s') \E_{a'\sim \pi_k(\cdot|s')} [B_k(s',a')]$. The purpose of \pref{alg: dilated bonus} is exactly to inductively find an estimator $ \widehat{w}_{k,h}$ of $w_{k,h}$ for all $h$. Then, we can form a \emph{bonus matrix} as the second matrix in \pref{eq: matrix form} (but replacing $w_{k,h}$ by $\widehat{w}_{k,h}$) and feed it to the FTRL algorithm. 

There are two technical complications regarding \pref{alg: dilated bonus}. First, in order to control the magnitude of $\widehat{w}_{k,h}$, we have to control the magnitude of $\alpha\|\phi(s,a)\|^2_{\Lambda_{k,h}^{-1}}$. This can be done by adding a pure exploration phase in the beginning of the algorithm (\pref{line: initial phase in alg} of \pref{alg: logdet FTRL}) and form a \emph{known state space} $\calZ\subset \calS$.  Known states are well-explored in the initial phase, and the values of $\|\phi(s,a)\|_{\Lambda_{k,h}^{-1}}^2$ on them are sufficiently small (in our case are of order $1/\sqrt{K}$). On the other hand, unknown states are hard to be reached by any policy (in our case, their probability of being reached is $\leq K^{-\frac{1}{4}}$) and thus can be ignored in the learning phase. The initial exploration phase is inspired by \cite{sherman2023rate}, who further built their algorithm on \cite{wagenmaker2022reward}'s reward-free exploration algorithm. We provide the guarantees for the initial exploration phase in \pref{app: pure exploration}. The other is that in order to ensure only positive bonuses are propagated over layers under estimation error of $\widehat{w}_{k,h}$, we force the bonus-to-go estimation to be non-negative in \pref{line: B+}. The additional penalty is related to $\|\widehat{w}_{k,h} - w_{k,h}\|$ and can be well-controlled.

\subsection{Regret Guarantee}
We defer the analysis of \pref{alg: logdet FTRL} to \pref{app:efficient}, and only state the final regret bound in the following theorem. 
\begin{theorem}
    \pref{alg: logdet FTRL} ensures a regret of order $\calR_K = \otil(d^{\frac{3}{2}}H^3K^{\frac{3}{4}})$. 
\end{theorem}

The improvement in our regret primarily stems from two sources. Firstly, we utilize an improved matrix concentration bound from \cite{liu2023bypassing}. This ensures that using $\tau = \frac{1}{\gamma}$ episodes (where $\gamma$ is the parameter in \pref{eq: construction 2}) is enough to gather data and build a reliable loss estimator. In contrast, previous works require $\tau = \frac{1}{\gamma^2}$ \citep{dai2023refined, sherman2023improved} or $\tau = \frac{1}{\gamma^3}$ \citep{luo2021policy}, thereby consuming excessive episodes to accumulate data for a single policy and consequently slowing down policy updates. Secondly, in previous works \citep{luo2021policy, dai2023refined, sherman2023improved}, the usage of exponential weights requires $\eta$ to be small compared to the magnitude of both loss estimators and exploration bonus. This prevents them from choosing the best $\eta$ in their algorithms. With the help of logdet barrier, in our algorithm, $\eta$ only needs to be small compared to the magnitude of the exploration bonus, which is already small given the initial exploration phase. This gives us more flexibility in choosing $\eta$.

\section{Conclusion}
In this work, we obtain the first optimal $\sqrt{K}$ regret bound for adversarial linear MDPs under bandit feedback and unknown transitions without the help of simulators or generative models. We also give a new $K^{\nicefrac{3}{4}}$ regret bound with an efficient policy optimization algorithm. We hope that the techniques and observations in the work could be helpful in developing an algorithm that is both statistically optimal and computationally efficient.

\section*{Acknowledgment}
We would like to thank Uri Sherman, Alon Cohen, Tomer Koren, and Yishay Mansour for sharing their withdrawn manuscript that inspires our solution. Their approach would give a computationally inefficient algorithm that ensures $K^{\nicefrac{2}{3}}$ regret. 

\bibliographystyle{plainnat}
\bibliography{ref}

\newpage
\appendix

\appendixpage

{
\startcontents[section]
\printcontents[section]{l}{1}{\setcounter{tocdepth}{2}}
}
\section{Omitted Details in \pref{sec: ineff alg}} \label{app: ineff}

\subsection{Policy Space Discretization}\label{app: policy disc}
\begin{proof}[Proof of \pref{lem: policy discretization}]
    Let $\bar\theta_h = \sum_{k=1}^K\theta_{k,h}$ and let $\bar\ell(s,a)=\langle  \phi(s,a), \bar\theta_h \rangle$ for $s\in\calS_h$ be the loss function under the loss vector $\bar\theta$. Under this loss function, the Q-function of a policy $\pi$ can be written as 
    \begin{align*}
        Q^\pi(s,a;\bar\ell) = \phi(s,a)^\top \xi^\pi_h \quad \text{for}\ s\in\calS_h, 
    \end{align*}
    where $\xi^\pi_h$ is recursively defined as 
    \begin{align*}
         \xi^\pi_h &= \bar\theta_h + \sum_{s'\in \calS_{h+1}}\psi(s')\sum_{a'\in\calA}\pi(a'|s')\langle \phi(s',a'),\xi^\pi_{h+1}\rangle. 
    \end{align*}
    Notice that by \pref{def: linear MDP}, we have $\|\xi_h^{\pi}\|_2 \leq H\sqrt{d}K$. 
    Let $\pi^\star$ be the optimal policy under loss function $\bar\ell$. Then by Bellman's optimality equation, $\pi^\star$ can be represented as
    \begin{align*}
         \pi^\star(s) = \argmin_a\left\{\phi(s,a)^\top \xi^{\pi^\star}_{h}\right\}
    \end{align*}
    and $\xi^{\pi^\star}_h$ can be found recursively from layer $H$ to layer $1$. 

    Now, let $\xi'_h$ be the closest element to $\xi^{\pi^\star}_h$ in the $H\sqrt{d}$-net of $\mathbb{B}^{d}(H\sqrt{d}K)$, and let $\pi'$ be the policy induced by $\xi'=(\xi'_1, \ldots, \xi_H')$, i.e., 
    \begin{align*}
        \pi'(s) = \argmin_a \left\{\phi(s,a)^\top \xi_h'\right\}. 
    \end{align*}
    
    
    Then for any $\pi$, we have 
    \begin{align*}
        &\sum_{k=1}^K \sum_{h=1}^H 
        \sum_{s\in\calS_h}\sum_{a\in\calA} (\mu^{\pi'}(s,a) - \mu^{\pi}(s,a)) \phi(s,a)^\top \theta_{k,h}    \\
        &=\sum_{k=1}^K \sum_{h=1}^H 
        \sum_{s\in\calS_h}\sum_{a\in\calA} (\mu^{\pi^\star}(s,a) - \mu^{\pi}(s,a)) \phi(s,a)^\top \theta_{k,h}  + \sum_{k=1}^K \sum_{h=1}^H 
        \sum_{s\in\calS_h}\sum_{a\in\calA} (\mu^{\pi'}(s,a) - \mu^{\pi^\star}(s,a)) \phi(s,a)^\top \theta_{k,h}    \\ 
        &= V^{\pi^\star}(s_1; \bar\ell) - V^{\pi}(s_1; \bar\ell) + \sum_{h=1}^H   \sum_{s\in \calS_h}\mu^{\pi'}(s)\sum_{a\in\calA}(\pi'(a|s)-\pi^\star(a|s))\phi(s,a)^\top\xi^\star_h \tag{by the performance difference lemma} \\
        &\leq 0 + \sum_{h=1}^H   \sum_{s\in \calS_h}\mu^{\pi'}(s)\sum_{a\in\calA}(\pi'(a|s)-\pi^\star(a|s))\phi(s,a)^\top\xi'_h +H^2\sqrt{d}  \tag{by the optimality of $\pi^\star$ under $\bar\ell$ and the discretization error}\\
        &\leq H^2\sqrt{d}
    \end{align*}
    where the last inequality is by the fact that $\pi'$ takes the argmin with respect to $\xi'_h$. Finally, notice that policy $\pi'$ belongs to $\Pi$ corresponding to the parameter $\theta_h=\frac{1}{H\sqrt{d}}\xi_h'$. 

\end{proof}

\subsection{Feature Estimation}\label{app: feature estimation}

\begin{proof}[Proof of \pref{lem: feasibility}]
    $\mu^\pi(s)$ satisfies \pref{eq: const 1} because $\mu^\pi$ is a valid occupancy measure. To show \pref{eq: const 2}, notice that 
    \begin{align}
        &\left|\sum_{s\in\calX_h} \sum_{a\in\calA} \mu^\pi(s)\pi(a|s) \clip\left[\phi(s,a)^\top \xi^\star_{h,f}\right] - \sum_{s\in\calX_{h+1}} \mu^{\pi}(s') f(s')\right|   \nonumber \\
        &=\left|\sum_{s\in\calS_h} \sum_{a\in\calA} \mu^\pi(s)\pi(a|s) \clip\left[\phi(s,a)^\top \xi^\star_{h,f}\right] - \sum_{s\in\calS_{h+1}} \mu^{\pi}(s') f(s')\right|  \tag{$\mu^{\pi}(s)=0$ for $s\in\calX \setminus \calS$}  \\ 
        &= \left|\sum_{s\in\calS_h} \sum_{a\in\calA} \mu^\pi(s)\pi(a|s) \clip\left[\phi(s,a)^\top \sum_{s'\in\calS_{h+1}} \psi(s')f(s') \right] - \sum_{s\in\calS_{h+1}} \mu^{\pi}(s') f(s')\right|    \nonumber \\
        &= \left|\sum_{s\in\calS_h} \sum_{a\in\calA} \mu^\pi(s)\pi(a|s) \clip\left[\sum_{s'\in\calS_{h+1}} P(s'|s,a)f(s') + z\right] - \sum_{s\in\calS_{h+1}} \mu^{\pi}(s') f(s') \right| \tag{for some $z$ such that $|z|\leq \zeta$ by \pref{def: mis-linear MDP}} \\
        &\leq  \left|\sum_{s\in\calS_h} \sum_{a\in\calA} \mu^\pi(s)\pi(a|s) \clip\left[\sum_{s'\in\calS_{h+1}} P(s'|s,a)f(s')\right] - \sum_{s\in\calS_{h+1}} \mu^{\pi}(s') f(s') \right| + \zeta   \nonumber \\
        &= \zeta
        \label{eq: check feasibility}
    \end{align}
    
    Finally, we show \pref{eq: const 4}. 
 For simplicity, let $\calD_h=\{(s_i, a_i, s_i')\}_{i=1}^{n}$ and let $\phi_i = \phi(s_i, a_i)$. 
We first consider a fixed policy $\pi$ and a layer $h$. Let $\epsilon=\frac{1}{K}$, and let $\calN_{\epsilon,1}$ be an $\epsilon$-net of $\calF^\pi$ on layer $h$ so that for any $f\in\calF^\pi$, there exists an $f'\in\calN_{\epsilon,1}$ such that $|f'(s) - f(s)|\leq \epsilon$ for all $s\in\calX_h$. Let $\calN_{\epsilon,2}$ be the $\epsilon$-net of $\mathbb{B}^d(\sqrt{d})$. Furthermore, define $|\Pi_h|=(3K)^d$ (whose meaning will be clear later). 

Then under this fixed $\pi$, for any $\xi\in \calN_{\epsilon,2}$ any $f\in\calN_{\epsilon,1}$, with probability at least $1-\frac{\delta}{|\calN_{\epsilon,1}||\calN_{\epsilon,2}||\Pi_h|K}$, 
 \begin{align*}
     & \sum_{i=1}^{n} \left(f(s_{i}') - \phi_i^\top \xi_{h,f}^\star\right)^2  - \sum_{i=1}^{n}\left(f(s_{i}') - \phi_i^\top \xi\right)^2 \\
     & = -2\sum_{i=1}^n (f(s_i') - \phi_i^\top \xi_{h,f}^\star)\left(\phi_i^\top \xi_{h,f}^\star - \phi_i^\top \xi\right) - \sum_{i=1}^n \left(\phi_i^\top \xi_{h,f}^\star - \phi_i^\top \xi\right)^2 \\
     & \leq -2\sum_{i=1}^n (f(s_i') - \E_{s'\sim P(\cdot|s_i,a_i)}[f(s')])\left(\phi_i^\top \xi_{h,f}^\star - \phi_i^\top \xi\right) - \sum_{i=1}^n \left(\phi_i^\top \xi_{h,f}^\star - \phi_i^\top \xi\right)^2 + 2\sqrt{d}n\zeta \\
     &\leq 6\sqrt{\sum_{i=1}^n  \left(\phi_i^\top \xi_{h,f}^\star - \phi_i^\top \xi\right)^2 \log\frac{|\calN_{\epsilon,1}||\calN_{\epsilon,2}||\Pi_h|K}{\delta}}  + 2\sqrt{d}\log\frac{|\calN_{\epsilon,1}||\calN_{\epsilon,2}||\Pi_h|K}{\delta}  \\
     &\qquad \qquad  - \sum_{i=1}^n \left(\phi_i^\top \xi_{h,f}^\star- \phi_i^\top \xi\right)^2 + 2\sqrt{d}n\zeta \tag{Freedman's inequality} \\
     &\leq 7\sqrt{d}\log\frac{|\calN_{\epsilon,1}||\calN_{\epsilon,2}||\Pi_h|K}{\delta} + 2\sqrt{d}n\zeta.   \tag{AM-GM}
 \end{align*}
Below, we take a union bound over $f\in\calN_{\epsilon,1}$, $\xi\in \calN_{\epsilon,2}$, and $\pi\in|\Pi|$. Notice that although the size of the policy set is $|\Pi|\leq (3K)^{dH}$ (a product of $H$ $\frac{1}{K}$-net for $\mathbb{B}^d(1)$), when considering the policies over layer $h$, the total number of different policies is only $|\Pi_h|\leq (3K)^d$. Therefore, a union bound over policies require only a size of $|\Pi_h|$. Bounding the distance between the full sets and $\epsilon$-nets, we conclude that with probability at least $\frac{\delta}{K}$, for all $\xi\in \mathbb{B}^d(\sqrt{d})$, all $\pi\in\Pi$, and all $f\in\calF^\pi$,  
\begin{align}
    \sum_{i=1}^{n} \left(f(s_{i}') - \phi_i^\top \xi_{h,f}^\star\right)^2  - \sum_{i=1}^{n}\left(f(s_{i}') - \phi_i^\top \xi\right)^2 \leq 7\sqrt{d}\log\frac{|\calN_{\epsilon,1}||\calN_{\epsilon,2}||\Pi_h|K}{\delta} + 2\sqrt{d}n\zeta + \sqrt{d}n\epsilon.     \label{eq: regression bound final}
\end{align}
By our choice of $\zeta$ and $\epsilon$, the second and third terms above are both negligible compared to the first term. Finally, we bound $|\calN_{\epsilon,1}|$ and $|\calN_{\epsilon,2}|$ via \cite{lattimore2020bandit} (Exercise 27.6). $|\calN_{\epsilon,2}|$ is the size of the $\epsilon$-net of $\mathbb{B}^d(\sqrt{d})$, equivalently the $(\epsilon/\sqrt{d})$-net of $\mathbb{B}^d(1)$, which is upper bounded by $(3\sqrt{d}/\epsilon)^{d}$. 
By the definition of $\calF^{\pi}$, the $\epsilon$-net of $\calF^\pi$ would be the union of the $\epsilon$-nets of $\{\theta:~ \theta\in\mathbb{B}^d(\sqrt{d})\}$ and $\{\Gamma\in\mathbb{R}^{d\times d}:~ \mathbf{0}\preceq \Gamma \preceq I\}$. 
Thus $|\calN_{\epsilon,1}|=(6d^\frac{3}{2}/\epsilon)^{d+d^2}$. Using these in \pref{eq: regression bound final} concludes the proof. 
\begin{align*}
    &7\sqrt{d}\log\frac{|\calN_{\epsilon,1}||\calN_{\epsilon,2}||\Pi_h|K}{\delta} + 2\sqrt{d}n\zeta + \sqrt{d}n\epsilon\\
    &\leq 8\sqrt{d}\log\frac{|\calN_{\epsilon,1}||\calN_{\epsilon,2}||\Pi_h|K}{\delta}\\
    &\leq 16d^{\frac{5}{2}}\log\frac{18d^\frac{3}{2}K}{\delta}\,.
\end{align*}
 
\end{proof}

\begin{lemma}\label{lem: xi_1 xi_2}
    Fix $\pi\in\Pi, h\in[H], f\in\calF^\pi$. Let $\xi_1$ and $\xi_2$ be two solutions for the $\hat{\xi}_{h,f}$ in \pref{eq: const 4}. Then $\|\xi_1- \xi_2\|_{\Lambda_h}\leq \frac{\Cbonus}{H}$. ($\Cbonus$ is defined in \pref{alg: exponetial weights})
\end{lemma}
\begin{proof}
    Let $\calD_h = \{(s_i, a_i, s_i')\}_{i=1}^n$ and denote $\phi_i=\phi(s_i,a_i)$. 
    Let $\xi_{\min} := \argmin_{\xi\in\mathbb{B}^d(\sqrt{d})} \sum_{i=1}^{n}\left(f(s_{i}') - \phi_i^\top \xi\right)^2$, where $\phi_i:=\phi(s_i,a_i)$. By the first-order optimality condition, 
    \begin{align}
        \sum_{i=1}^{n} \left(f(s_i') - \phi_i^\top \xi_{\min}\right)\left(\phi_i^\top \xi_1 - \phi_i^\top\xi_{\min}\right)\leq 0.   \label{eq: tmp 1}
    \end{align}
    By the fact that $\xi_1$ satisfies \pref{eq: const 4}, 
    \begin{align*}
        16d^{\frac{5}{2}}\log\frac{18d^\frac{3}{2}K}{\delta}  &\geq  \sum_{i=1}^{n}\left(f(s_{i}') - \phi_i^\top \xi_1 \right)^2 - \sum_{i=1}^{n}\left(f(s_{i}') - \phi_i^\top \xi_{\min}\right)^2 \\
        &= 2\sum_{i=1}^{n} \left(f(s_i')-\phi_i^\top \xi_{\min}\right)(\phi_i^\top \xi_{\min} - \phi_i^\top \xi_1) + \sum_{i=1}^{n} \left( 
\phi_i^\top (\xi_1-\xi_{\min}) \right)^2 \\
&\geq \sum_{i=1}^{n} \left( 
\phi_i^\top (\xi_1-\xi_{\min}) \right)^2 \tag{using \pref{eq: tmp 1}} \\
      &=  \|\xi_1 - \xi_{\min}\|^2_{\Lambda_h} - \|\xi_1 - \xi_{\min}\|_2^2  \tag{by the definition of $\Lambda_h$} \\
      &\geq \|\xi_1 - \xi_{\min}\|^2_{\Lambda_h} - 4d, 
    \end{align*}
    which gives $\|\xi_1 - \xi_{\min}\|^2_{\Lambda_h}\leq \frac{\Cbonus^2}{4H^2}$ (recall $\Cbonus = 10d^{\frac{5}{4}}H\sqrt{\log\frac{18d^\frac{3}{2}K}{\delta}}$.
    Similarly, $\|\xi_2 - \xi_{\min}\|^2_{\Lambda_h}\leq \frac{\Cbonus^2}{4H^2}$. Combining them proves the lemma. 
\end{proof}

\begin{proof}[Proof of \pref{lem: value diff linear case}]
    \begin{align*}
       &\sum_{s'\in\calX_{h+1}}  (\hat{\mu}^\pi(s') - \mu^{\pi}(s'))f(s') \\
       &\leq \sum_{s\in\calX_{h}}\sum_{a\in\calA} \hat{\mu}^\pi(s,a)\clip\left[\phi(s,a)^\top \hat{\xi}_{h,f}\right] - \sum_{s\in\calX_{h}}\sum_{a\in\calA} \mu^\pi(s,a)\clip\left[\phi(s,a)^\top \xi_{h,f}^\star \right] + 2\zeta \tag{by \pref{eq: const 2} and the same calculation as \pref{eq: check feasibility}} \\
       &= \sum_{s\in\calX_{h}}\sum_{a\in\calA} \mu^\pi(s,a)\left(\clip\left[\phi(s,a)^\top \hat{\xi}_{h,f}\right] - \clip\left[\phi(s,a)^\top\xi_{h,f}^\star\right]\right) \\
       &\qquad \qquad  + \sum_{s\in\calX_{h}}\sum_{a\in\calA} (\hat{\mu}^\pi(s,a) - \mu^\pi(s,a))\clip\left[\phi(s,a)^\top \hat{\xi}_{h,f}\right] + 2\zeta\\
       &\leq \sum_{s\in\calX_{h}}\sum_{a\in\calA} \mu^\pi(s,a)\|\phi(s,a)\|_{\Lambda_h^{-1}} \|\hat{\xi}_{h,f} - \xi_{h,f}^\star\|_{\Lambda_h} + \sum_{s\in\calX_{h}} (\hat{\mu}^\pi(s) - \mu^\pi(s)) \tilde{f}(s) + 2\zeta   \\
       &\leq \frac{\Cbonus}{H}\times\sum_{s\in\calX_{h}}\sum_{a\in\calA} \mu^\pi(s,a)\|\phi(s,a)\|_{\Lambda_h^{-1}}  + \sum_{s\in\calX_{h}} (\hat{\mu}^\pi(s) - \mu^\pi(s)) \tilde{f}(s) + 2\zeta \tag{by \pref{lem: xi_1 xi_2}} 
    \end{align*}
    where $\tilde{f}(s) := \sum_{a\in\calA} \pi(a|s)\clip\left[ \phi(s,a)^\top \hat{\xi}_{h,f}\right]$, which again belongs to $\calF^\pi$.  Recursively applying the inequality proves the first inequality in the lemma. To obtain the second inequality in the lemma, with slightly different decomposition in the second step above, we get
    \begin{align*}
        &\sum_{s\in\calX_{h}}\sum_{a\in\calA} \hat{\mu}^\pi(s,a)\left(\clip\left[\phi(s,a)^\top \hat{\xi}_{h,f}\right] - \clip\left[\phi(s,a)^\top\xi_{h,f}^\star\right]\right) \\
        &\qquad \qquad + \sum_{s\in\calX_{h}}\sum_{a\in\calA} (\hat{\mu}^\pi(s,a) - \mu^\pi(s,a))\clip\left[\phi(s,a)^\top \xi_{h,f}^\star\right] + 2\zeta\\
        &\leq \sum_{s\in\calX_{h}}\sum_{a\in\calA} \hat{\mu}^\pi(s,a)\|\phi(s,a)\|_{\Lambda_h^{-1}} \|\hat{\xi}_{h,f} - \xi_{h,f}^\star\|_{\Lambda_h} + \sum_{s\in\calX_{h}} (\hat{\mu}^\pi(s) - \mu^\pi(s)) \tilde{f}'(s)  + 2\zeta\\
        &\leq \frac{\Cbonus}{H}\times\sum_{s\in\calX_{h}}\sum_{a\in\calA}\hat{\mu}^\pi(s,a)\|\phi(s,a)\|_{\Lambda_h^{-1}}  + \sum_{s\in\calX_{h}} (\hat{\mu}^\pi(s) - \mu^\pi(s)) \tilde{f}'(s)  + 2\zeta
    \end{align*}
    where $\tilde{f}'(s) := \sum_{a\in\calA} \pi(a|s)\clip\left[ \phi(s,a)^\top \xi^\star_{h,f}\right]$. Following the same argument proves the second inequality. 

\end{proof}

\subsection{Regret Analysis}\label{app: regret analysis for ineff}

\begin{align*}
    &\E\left[\calR_K\right] \\
    &= \E\left[\sum_{k=1}^K\sum_{h=1}^H\sum_{\pi\in\Pi} q_k'(\pi)(\phi^\pi_h)^\top\theta_{k,h} - \sum_{k=1}^K\sum_{h=1}^H (\phi^{\pi^\star}_h)^\top \theta_{k,h}\right] \\
    &= \E\Bigg[\sum_{k=1}^K\sum_{h=1}^H\sum_{\pi\in\Pi} q_k(\pi)(\phi^\pi_h)^\top\theta_{k,h} - \sum_{k=1}^K\sum_{h=1}^H (\phi^{\pi^\star}_h)^\top \theta_{k,h} + \underbrace{\sum_{k=1}^K\sum_{h=1}^H (q_k'(\pi) - q_k(\pi)) (\phi^\pi_h)^\top \theta_{k,h}}_{\leq \eta HK}   \Bigg]\\
    &\leq \E\Bigg[\sum_{k=1}^K\sum_{\pi\in\Pi} q_k(\pi)(\hat{\phi}_k^\pi)^\top\hat{\theta}_{k} - \sum_{k=1}^K (\hat{\phi}_k^{\pi^\star})^\top \hat{\theta}_{k} + \underbrace{\sum_{k=1}^K\sum_{h=1}^H \sum_{\pi\in\Pi} q_k(\pi)\left( (\phi^\pi_h - \phi^{\pi^\star}_h)^\top\theta_{k,h} - (\hat{\phi}_{k,h}^\pi - \hat{\phi}_{k,h}^{\pi^\star})^\top\hat{\theta}_{k,h} \right)}_{\textbf{bias}}\Bigg] + \eta HK \\
    &= \E\Bigg[\underbrace{\sum_{k=1}^K\sum_{\pi\in\Pi} q_k(\pi)\left((\hat{\phi}_{k}^\pi)^\top\hat{\theta}_{k} - b_k^\pi\right) - \sum_{k=1}^K \left((\hat{\phi}_{k}^{\pi^\star})^\top \hat{\theta}_{k} -b_k^{\pi^\star}\right)}_{\FTRL} + \underbrace{\sum_{k=1}^K\sum_{\pi\in\Pi} q_k(\pi)b_k^\pi - \sum_{k=1}^K b_k^{\pi^\star}}_{\bonus} + \bias\Bigg] + \eta HK
\end{align*}
We bound the terms individually in \pref{lem: bias}, \pref{lem: ftrl} and \pref{lem: bonus}.
The potentially unbounded bias term is offset by a negative contribution in the bonus term. 
\subsubsection{Bounding the Bias}
\begin{lemma}  
\label{lem: bias}
    \begin{align*}
        \bias &\leq \E\Bigg[\Cbonus\sum_{k=1}^K\sum_{h=1}^H\sum_{s\in\calX_{h}}\sum_{a\in\calA} \hat\mu_k^{\pi^\star}(s,a)\|\phi(s,a)\|_{\Lambda_{k,h}^{-1}} + \eta\sum_{k=1}^K \|\hat{\phi}^{\pi^\star}_k\|^2_{M_k^{-1}} \Bigg]  \\
 &\qquad  + \otil\left(\frac{d^{\frac{9}{2}}H^3}{\eta} + \eta dH K + d^{3}H^3\sqrt{K} \right).   
    \end{align*}
\end{lemma}
\begin{proof}
The bias of any policy $\pi$ at episode $k$ and stage $h$ can be calculated as the following: 
\allowdisplaybreaks
\begin{align*}
    &(\phi^\pi_h)^\top \theta_{k,h} - (\hat{\phi}^\pi_{k,h})^\top \E[\hat{\theta}_{k,h}] 
    \leq \underbrace{\left|(\phi^\pi_h - \hat{\phi}_{k,h}^\pi)^\top \theta_{k,h}\right|}_{\bias_{k,h,1}^\pi} + \underbrace{\left|(\hat{\phi}_{k,h}^\pi)^\top (\theta_{k,h} - \E[\hat{\theta}_{k,h}])\right|}_{\bias_{k,h,2}^\pi}.
    \end{align*}
    Set 
    \[f(s) =\sum_{a\in\calA} \pi(a|s) \phi(s,a)^\top \theta_{k,h}= \sum_{a\in\calA} \pi(a|s) \clip\left[\phi(s,a)^\top \theta_{k,h}\right]\in \calF_1^\pi\,,\tag{$|\phi(s,a)^\top \theta_{k,h}|\leq 1$}\]
    then the first term is by \pref{lem: value diff linear case}
    \begin{align*}
    \bias_{k,h,1}^\pi = \left|\sum_{s\in\calX_h} (\mu^\pi(s)-\hat{\mu}_k^\pi(s))f(s)\right| \leq\frac{\Cbonus}{H}\times\sum_{h'<h}\sum_{s\in\calX_{h'}}\sum_{a\in\calA} \hat\mu_k^{\pi}(s,a)\|\phi(s,a)\|_{\Lambda_{k,h'}^{-1}} + 2\zeta H.
    \end{align*}
  Define $M_{k,h}= \sum_{\pi\in\Pi}q_k'(\pi) \hat{\phi}_{k,h}^\pi  (\hat{\phi}_{k,h}^\pi)^\top$. Then the second term is
  \begin{align*}
    \bias_{k,h,2}^\pi&\leq  \|\hat{\phi}_{k,h}^\pi\|_{M_{k,h}^{-1}} \left\|\theta_{k,h} - M_{k,h}^{-1}\sum_{\pi'} q_k'(\pi') \hat{\phi}^{\pi'}_{k,h}  (\phi^{\pi'}_h)^\top \theta_{k,h}\right\|_{M_{k,h}} \\
    &=  \|\hat{\phi}_{k,h}^\pi\|_{M_{k,h}^{-1}} \left\|M_{k,h}^{-1} \sum_{\pi'} q_k'(\pi') \hat{\phi}^{\pi'}_{k,h} (\hat{\phi}^{\pi'}_{k,h} - \phi^{\pi'}_h)^\top \theta_{k,h} \right\|_{M_{k,h}} \\
    &=  \|\hat{\phi}_{k,h}^\pi\|_{M_{k,h}^{-1}} \left\| \sum_{\pi'} q_k'(\pi') \hat{\phi}^{\pi'}_{k,h} (\hat{\phi}^{\pi'}_{k,h} - \phi^{\pi'}_h)^\top \theta_{k,h} \right\|_{M_{k,h}^{-1}} \\
    &\leq  \eta\|\hat{\phi}_{k,h}^\pi\|^2_{M_{k,h}^{-1}}+ \frac{1}{\eta}\left\| \sum_{\pi'} q_k'(\pi') \hat{\phi}^{\pi'}_{k,h} (\hat{\phi}^{\pi'}_{k,h} - \phi^{\pi'}_h)^\top \theta_{k,h} \right\|_{M_{k,h}^{-1}}^2 \\
    &\leq  \eta\|\hat{\phi}_{k,h}^\pi\|^2_{M_{k,h}^{-1}}+ \frac{1}{\eta}\left(\sum_{\pi'} q_k'(\pi')\left\| \hat{\phi}^{\pi'}_{k,h} \right\|^2_{M_{k,h}^{-1}}\right)\left(\sum_{\pi'} q_k'(\pi')((\hat{\phi}^{\pi'}_{k,h} - \phi^{\pi'}_h)^\top \theta_{k,h} )^2\right) \tag{by \pref{lem: technical 1}}\\
    &\leq  \eta\|\hat{\phi}_{k,h}^\pi\|^2_{M_{k,h}^{-1}}+ \frac{d}{\eta}\sum_{\pi'} q_k'(\pi')\left(\otil(d^{\frac{5}{4}})\times\sum_{h'<h}\sum_{s\in\calX_{h'}}\sum_{a\in\calA} \mu^{\pi'}(s,a)\|\phi(s,a)\|_{\Lambda_{k,h'}^{-1}} + 2\zeta H \right)^2\tag{by \pref{lem: value diff linear case}}\\
    &\leq  \eta\|\hat{\phi}_{k,h}^\pi\|^2_{M_{k,h}^{-1}}+ \frac{\otil(d^{\frac{7}{2}})}{\eta}\times\sum_{\pi'} q_k'(\pi')\left(\sum_{h'<h}\sum_{s\in\calX_{h'}}\sum_{a\in\calA}\mu^{\pi'}(s,a)\right)\left(\sum_{h'<h}\sum_{s\in\calX_{h'}}\sum_{a\in\calA} \mu^{\pi'}(s,a)\|\phi(s,a)\|^2_{\Lambda_{k,h'}^{-1}}\right) \tag{Cauchy-Schwarz} \\
    &\qquad + \order\left(\frac{d\zeta^2 H^2}{\eta}\right)\\
   &\leq  \eta\|\hat{\phi}_{k,h}^\pi\|^2_{M_{k,h}^{-1}}+ \frac{\otil(d^{\frac{7}{2}}H)}{\eta}\sum_{h'<h}\beta_{k,h'}  + \order\left(\frac{d\zeta^2 H^2}{\eta}\right)
\end{align*}
where $\beta_{k,h}=\sum_\pi\sum_{s\in\calS_h,a\in\calA}q_k'(\pi)\mu^\pi(s,a)\|\phi(s,a)\|^2_{\Lambda_{k,h}^{-1}}$. 
We have
\begin{align*}
    \E\left[\sum_{k=1}^K\sum_{h=1}^H\beta_{k,h}\right] &= \E\left[\sum_{k=1}^K\sum_{h=1}^H\E\left[\norm{\phi(s_{k,h},a_{k,h})}^2_{\Lambda_{k,h}^{-1}}\,|\,\calD_{k-1}\right]\right]\\
    &=\E\left[\sum_{k=1}^K\sum_{h=1}^H\norm{\phi(s_{k,h},a_{k,h})}^2_{\Lambda_{k,h}^{-1}}\right] \leq \order(dH\log(K)).  
\end{align*}
Thus, for any $\pi$, 
\begin{align*}
    \E\left[\sum_{k=1}^K \sum_{h=1}^H \bias_{k,h,2}^\pi\right] 
    &=\E\left[ \sum_{k=1}^K \sum_{h=1}^H   \eta\|\hat{\phi}_{k,h}^\pi\|^2_{M_{k,h}^{-1}} + \frac{\otil(d^{\frac{7}{2}}H)}{\eta} \times \sum_{k=1}^K \sum_{h=1}^H \sum_{h'<h}  \beta_{k,h'}\right] + \order\left(\frac{d\zeta^2 H^3 K}{\eta}\right) \\
    &\leq \E\left[ \sum_{k=1}^K    \eta\|\hat{\phi}_{k}^\pi\|^2_{M_{k}^{-1}}  \right] + \frac{\otil(d^{\frac{9}{2}}H^3)}{\eta}.  \tag{$\zeta=\frac{d}{K}$}
\end{align*}
Overall, 
\begin{align*}
    \bias 
    &\leq \E\left[\sum_{k=1}^K \sum_{h=1}^H \left(\bias_{k,h,1}^{\pi^\star} + \sum_{\pi} q_k(\pi) \bias_{k,h,1}^\pi \right) + \sum_{k=1}^K \sum_{h=1}^H \left(\bias_{k,h,2}^{\pi^\star} + \sum_{\pi} q_k(\pi) \bias_{k,h,2}^\pi \right) \right]\\ 
    &\leq \E\Bigg[\sum_{k=1}^K \sum_{h=1}^H \left(\frac{\Cbonus}{H}\times\sum_{h'<h}\sum_{s\in\calX_{h'}}\sum_{a\in\calA} \left(\hat\mu_k^{\pi^\star}(s,a) + \sum_\pi q_k(\pi) \hat{\mu}_k^\pi(s,a)\right)\|\phi(s,a)\|_{\Lambda_{k,h'}^{-1}}\right)  \\
    &\qquad + \sum_{k=1}^K \left(  \eta\|\hat{\phi}_{k}^{\pi^\star}\|^2_{M_{k}^{-1}} +  \eta\sum_\pi q_k(\pi) \|\hat{\phi}_{k}^\pi\|^2_{M_{k}^{-1}} \right) \Bigg] + \frac{\otil(d^{\frac{9}{2}}H^3)}{\eta}\\
    &\leq \E\Bigg[\Cbonus\sum_{k=1}^K\sum_{h=1}^H\sum_{s\in\calX_{h}}\sum_{a\in\calA} \left(\hat\mu_k^{\pi^\star}(s,a) + 2\sum_{\pi} q_k'(\pi)\hat{\mu}_k^\pi(s,a)\right)\|\phi(s,a)\|_{\Lambda_{k,h}^{-1}}\\
    &\qquad 
   +  \sum_{k=1}^K \left(  \eta\|\hat{\phi}^{\pi^\star}_k\|^2_{M_k^{-1}} + 2\eta \sum_{\pi}q_k'(\pi)\|\hat{\phi}^{\pi}_k\|^2_{M_k^{-1}}\right)\Bigg]+\frac{\otil(d^{\frac{9}{2}}H^3)}{\eta} \\
 &\leq \E\Bigg[\Cbonus\sum_{k=1}^K\sum_{h=1}^H\sum_{s\in\calX_{h}}\sum_{a\in\calA} \left(\hat\mu_k^{\pi^\star}(s,a)  + 2\sum_\pi 
q_k'(\pi) \mu^\pi(s,a)\right)\|\phi(s,a)\|_{\Lambda_{k,h}^{-1}} \\
  &\qquad  +  2\Cbonus \sum_{k=1}^K \sum_{h=1}^H \sum_{s\in\calX_h}\sum_{a\in\calA} \sum_\pi q_k'(\pi)(\hat{\mu}^\pi_k(s,a) - \mu^\pi(s,a))\|\phi(s,a)\|_{\Lambda_{k,h}^{-1}} \\   
 &\qquad + \eta\sum_{k=1}^K \|\hat{\phi}^{\pi^\star}_k\|^2_{M_k^{-1}}  +  2\eta dH K \Bigg]+\frac{\otil(d^{\frac{9}{2}}H^3)}{\eta} \\
 &\leq \E\Bigg[\Cbonus\sum_{k=1}^K\sum_{h=1}^H\sum_{s\in\calX_{h}}\sum_{a\in\calA} \hat\mu_k^{\pi^\star}(s,a)\|\phi(s,a)\|_{\Lambda_{k,h}^{-1}} + \otil(\Cbonus H\sqrt{dK})  \tag{*}  \\
 &\qquad  + 2\Cbonus \sum_{k=1}^K \sum_{h=1}^H \sum_{\pi}q_k'(\pi) \left(\frac{\Cbonus}{H} \sum_{h'<h} \sum_{s\in\calX_{h'} } \sum_{a\in\calA}\mu^\pi(s,a) \|\phi(s,a)\|_{\Lambda_{k,h'}^{-1}} \right)  \tag{by \pref{lem: value diff linear case}} \\
 &\qquad  + \eta\sum_{k=1}^K \|\hat{\phi}^{\pi^\star}_k\|^2_{M_k^{-1}}  +  2\eta dH K \Bigg]+\frac{\otil(d^{\frac{9}{2}}H^3)}{\eta} \\
 &\leq \E\Bigg[\Cbonus\sum_{k=1}^K\sum_{h=1}^H\sum_{s\in\calX_{h}}\sum_{a\in\calA} \hat\mu_k^{\pi^\star}(s,a)\|\phi(s,a)\|_{\Lambda_{k,h}^{-1}} + \eta\sum_{k=1}^K \|\hat{\phi}^{\pi^\star}_k\|^2_{M_k^{-1}}    \\
 &\qquad  + 2\Cbonus^2 \sum_{k=1}^K \sum_{h=1}^H \sum_{s\in\calX_{h} } \sum_{a\in\calA}\sum_{\pi}q_k'(\pi)  \mu^\pi(s,a) \|\phi(s,a)\|_{\Lambda_{k,h}^{-1}}  \Bigg]   \\
 &\qquad   + \otil\left(\frac{d^{\frac{9}{2}}H^3}{\eta} + \eta dH K + \Cbonus H\sqrt{dK}\right) \\
 &\leq \E\Bigg[\Cbonus\sum_{k=1}^K\sum_{h=1}^H\sum_{s\in\calX_{h}}\sum_{a\in\calA} \hat\mu_k^{\pi^\star}(s,a)\|\phi(s,a)\|_{\Lambda_{k,h}^{-1}} + \eta\sum_{k=1}^K \|\hat{\phi}^{\pi^\star}_k\|^2_{M_k^{-1}} \Bigg]  \\
 &\qquad  + \otil\left(\frac{d^{\frac{9}{2}}H^3}{\eta} + \eta dH K + \Cbonus H\sqrt{dK} + \Cbonus^2 H\sqrt{dK} \right)    \tag{*}
\end{align*}
where in the two (*) places we use 
\begin{align*}  
    &\E\left[\sum_{k=1}^K \sum_{h=1}^H \sum_{s\in\calX_{h} } \sum_{a\in\calA}\sum_{\pi}q_k'(\pi)  \mu^\pi(s,a) \|\phi(s,a)\|_{\Lambda_{k,h}^{-1}}\right] \\
    &\leq \E\left[\sqrt{\sum_{k=1}^K \sum_{h=1}^H \sum_{s\in\calX_{h} } \sum_{a\in\calA}\sum_{\pi}q_k'(\pi)  \mu^\pi(s,a) }\sqrt{\sum_{k=1}^K \sum_{h=1}^H \sum_{s\in\calX_{h} } \sum_{a\in\calA}\sum_{\pi}q_k'(\pi)  \mu^\pi(s,a) \|\phi(s,a)\|^2_{\Lambda_{k,h}^{-1}}}\right] \\
    &\leq \sqrt{HK\E\left[\sum_{k=1}^K \sum_{h=1}^H \beta_{k,h}\right]}  \\
    &\leq \otil(H\sqrt{dK}). 
\end{align*}
Finally, plugging in the definition of $\Cbonus=\otil(d^{\frac{5}{4}}H)$ gives the desired bound.

\end{proof}
\subsubsection{Bounding the FTRL regret}
\begin{lemma}  
\label{lem: ftrl}
\begin{align*}
    \FTRL \leq  \otil\left(\eta d^2H^4K+\frac{\eta^3H^2}{\gamma^2}K+\gamma H K\right). 
\end{align*}
\end{lemma}
\begin{proof}
The magnitude of the loss is bounded by 
    \begin{align*}
        |\hat{\phi}^{\pi^\top}_k \hat{\theta}_k - b_k^\pi| 
        & \leq \left|\hat{\phi}^{\pi^\top}_k M_k^{-1} \hat{\phi}_k^{\pi_k}L_k \right|+\Cbonus\sum_{h=1}^H\sum_{s\in\calX_h}\sum_{a\in\calA}\hat{\mu}_k^\pi(s,a)\|\phi(s,a)\|_{\Lambda_{k,h}^{-1}} + \eta\|\hat{\phi}_k^\pi\|_{M_k^{-1}}^2  \\
        &\leq \norm{ \hat{\phi}^{\pi}}_{M_k^{-1}}\norm{\hat{\phi}_k^{\pi_k}}_{M_k^{-1}}H+\Cbonus H+\frac{\eta dH}{\gamma}\\
        &\leq \frac{dH}{\gamma}+\Cbonus H+\frac{\eta dH}{\gamma}\leq \frac{2dH}{\gamma}+\Cbonus H. 
    \end{align*}
    If $\eta \leq \frac{1}{\frac{4dH}{\gamma}+2\Cbonus H}$, then we have $\eta|\hat{\phi}^{\pi^\top}_k\hat\theta_k-b_k^\pi|\leq \frac{1}{2}$ and we can use the standard FTRL regret bound of exponential weights \citep[Equation (27.2, 27.3)]{lattimore2020bandit}: 
    \begin{align*}
        \FTRL \leq \underbrace{\gamma KH}_{\text{John's exploration}} + \frac{\ln|\Pi|}{\eta} +\eta\sum_{k=1}^{K}\E\left[\E_{\pi_k\sim q_k'}\left[\sum_{\pi\in\Pi}q_{k}(\pi)(2(\hat{\phi}^{\pi^\top}_k \hat{\theta}_k)^2 + 2(b_k^\pi)^2)\right]\right]\,.
    \end{align*}
    Since $M_k=\E_{\pi\sim q_k'}[\hat{\phi}^{\pi}_k\hat{\phi}^{\pi^\top}_k]$, we have $M_k^{-1}\preceq \frac{1}{1-\gamma}\left(\E_{\pi\sim q_k}[\hat{\phi}^{\pi}_k\hat{\phi}^{\pi^\top}_k]\right)^{-1}$, and thus
    \begin{align*}
        \E_{\pi_k\sim q_k'}\left[\sum_{\pi\in\Pi}q_k(\pi)(\hat{\phi}^{\pi^\top}_k M_k^{-1} \hat{\phi}_k^{\pi_k}L_k)^2\right]
        &\leq H^2\frac{1}{(1-\gamma)^2}\operatorname{Tr}\left(M_kM_k^{-1}M_kM_k^{-1}\right)=\order(dH^3)\,.
    \end{align*}
    For the final term, we have
    \begin{align*}
        \sum_{k=1}^K\eta \sum_{\pi}q_k(\pi)(b_k^\pi)^2 &\leq \eta\Cbonus^2H^2K+\frac{\eta^3d^2H^2}{\gamma^2}K
        = \otil\left(\eta d^{\frac{5}{2}}H^4K+\frac{\eta^3d^2H^2}{\gamma^2}K\right)\,.
    \end{align*} 
\end{proof}

\subsubsection{Bounding the bonus}
\begin{lemma}
\label{lem: bonus}
    \begin{align*}
         \bonus &\leq  -\E\Bigg[\sum_{k=1}^K\eta  \|\hat{\phi}^{\pi^\star}_k\|_{M_k^{-1}}^2 + \Cbonus\sum_{h=1}^H\sum_{s\in\calX_{h}}\sum_{a\in\calA} \hat\mu_k^{\pi^\star}(s,a)\|\phi(s,a)\|_{\Lambda_{h}^{-1}}\Bigg]  \\
         &\qquad +\otil\left(\frac{d^{\frac{9}{2}}H^3}{\eta} + \eta dHK + d^3H^3\sqrt{K}\right). 
    \end{align*}
\end{lemma}
\begin{proof}
\begin{align*}
    \bonus &\leq \E\Bigg[\sum_{k=1}^K\eta \sum_\pi q_k'(\pi)\|\hat{\phi}^\pi_k\|_{M_k^{-1}}^2 + \Cbonus\sum_{h=1}^H\sum_{s\in\calX_{h}}\sum_{a\in\calA}\sum_\pi q_k'(\pi) \hat\mu_k^{\pi}(s,a)\|\phi(s,a)\|_{\Lambda_{h}^{-1}} \\
    &\qquad \qquad - \eta  \|\hat{\phi}^{\pi^\star}_k\|_{M_k^{-1}}^2 - \Cbonus\sum_{h=1}^H\sum_{s\in\calX_{h}}\sum_{a\in\calA} \hat\mu_k^{\pi^\star}(s,a)\|\phi(s,a)\|_{\Lambda_{h}^{-1}}\Bigg] 
\end{align*}
The first and the second term above have been handled in the proof of \pref{lem: bias}. Following the analysis there, we can bound their sum by $\otil\left(\frac{d^{\frac{9}{2}}H^3}{\eta} + \eta dHK + d^3H^3\sqrt{K}\right)$. 
\end{proof}

\subsubsection{Finishing up}
\begin{proof}[Proof of \pref{thm: inefficient bound}]
Combining the bounds in \pref{lem: bias}, \pref{lem: ftrl}, and \pref{lem: bonus}, we bound the regret as
\begin{align*}
    \E\left[\calR_K\right] &\leq \otil\left(\eta d^{\frac{5}{2}}H^4K+\frac{\eta^3d^2H^2}{\gamma^2}K+\gamma H K+\frac{d^{\frac{9}{2}}H^3}{\eta}+d^3H^3\sqrt{K}\right)\\
    &=\otil\left(\eta d^{\frac{5}{2}}H^4K+\frac{d^{\frac{9}{2}}H^3}{\eta}+d^3H^3\sqrt{K}\right)\tag{$\gamma =\Theta(\eta dH)$}\\
    &=\otil(d^\frac{7}{2}H^{\frac{7}{2}}\sqrt{K}). \tag{$\eta =\Theta(d/\sqrt{HK})$}
\end{align*}
\end{proof}
\newpage
\section{Initial Pure Exploration Phase}\label{app: pure exploration}

\begin{algorithm}[!ht]
    \caption{Initial Pure Exploration (Algorithm~2 of \cite{sherman2023rate})} 
    \label{alg:reward_free}
	\begin{algorithmic}
		\STATE \textbf{input:} $\delta,  \rho, \epsilon_{\rm cov}$
		\STATE Set $m=\lceil \log\frac{1}{\epsilon_{\rm cov}} \rceil$
		\STATE Set $\forall i\in[m], \; \rho_i = \rho$
		\FOR {$h=H, \ldots, 1$}			
			\STATE $\left\{\widetilde{\calX}_{h,i}, \widetilde \calD_{h,i}, \widetilde \Lambda_{h,i}\right\}_{i=1}^m \gets \textsc{CoverTraj}(h, \frac{\delta}{H}, \{\rho_i\}_{i=1}^m, m)$
			
			\STATE $\calD_{h} \gets \bigcup_i \widetilde \calD_{h,i}$
            \STATE
			$\Lambda_{h}  \gets 
			I + \sum_{(s,a,s')\in \calD_{h}}\phi(s, a)\phi(s, a)^\top $ 
   \STATE $\calZ_h \gets \left\{s\in \calS_h:~~ \forall a\in\calA,\ \  \norm{\phi(s, a)}_{\Lambda_{h}^{-1}}\leq \rho\right\} $
		\ENDFOR
		\RETURN $(\calD_{h},
                    	\calZ_{h})_{h=1}^H$
	\end{algorithmic}
\end{algorithm}

\begin{theorem}[Theorem 2 in \cite{sherman2023rate}]
\label{thm:wagenmaker_cover_traj}
	The $\textsc{CoverTraj}$ algorithm \citep[Algorithm~4]{wagenmaker2022reward} when instantiated with $\textsc{Force}$ \citep[Algorithm~1]{wagenmaker2022first} enjoys the following guarantee for linear MDPs.
	Given a sequence of tolerance parameters $\rho_1, \ldots, \rho_m > 0$ and $h\in [H]$, the algorithm interacts with the environment for $T$ steps, where
	\begin{align*}
		T \leq  T_{\max}\triangleq 	
			C \sum_{i=1}^m 2^i \max\left\{
				\frac{d}{\rho_i^2}\log \frac{2^i}{\rho_i^2}, 
				d^4 H^3 m^3  \log^{7/2}\frac1\delta
			\right\}, \quad C > 0 \text{ \ is a logarithmic term,}
	\end{align*}
	and outputs $
		\left\{\widetilde{\calX}_{h, i}, \widetilde \calD_{h,i}, \widetilde \Lambda_{h,i}\right\}_{i=1}^m
	$
	such that $\left\{\widetilde{\calX}_{h, i}\right\}_{i=1}^{m+1}$ forms a partition for the unit Euclidean ball, 
 $\widetilde \Lambda_{h,i} = I + \sum_{(s,a,s')\in \widetilde \calD_{h, i}} \phi(s, a) \phi(s, a)^\top $, and
 with probability $1-\delta$, it holds that:
	\begin{align*}
		&\forall i\in [m], \quad 
		\phi^\top \widetilde \Lambda_{h,i}^{-1}\phi \leq \rho_i^2, 
		\quad 
		\forall \phi \in \widetilde{\calX}_{h, i};
		\\
            \text{and \ \ } &\forall i\in [m+1], \quad 
		\sup_{\pi}\left\{\sum_{s\in\calS_h}\sum_{a\in\calA} 
        \ind\left\{\phi(s, a) \in \widetilde{\calX}_{h, i}\right\}\mu^\pi(s, a) \right\}\leq 2^{-i+1}
		.
	\end{align*}
\end{theorem}


\begin{lemma}[Lemma 15 in \cite{sherman2023rate}]
\label{lem:unknwon_states_weight}
	Assume $h\in [H], \epsilon_{\rm cov} > 0, \delta > 0, m=\lceil \log (1/\epsilon_{\rm cov})\rceil, \rho_m \geq \cdots \geq \rho_1 > 0$, and let
 $\left\{\widetilde{\Lambda}_{h, i}\right\}_{i\in[m]}$ be the covariance matrices returned from \textsc{CoverTraj}$(h, \frac{\delta}{H}, \{\rho_i\}_{i=1}^m, m)$. Then under the assumption that the event from \pref{thm:wagenmaker_cover_traj} holds, 
	we have for any policy $\pi$ and $i\in[m]$:
	\begin{align*}
		\sum_{s\in\calS_h} \mu^\pi(s)\ind\left\{ \exists a \text{ s.t. } \norm{\phi(s, a)}_{ \widetilde{\Lambda}_{h, i}^{-1}} > \rho_m\right\}
		\leq \epsilon_{\rm cov}.
	\end{align*} 
\end{lemma}

\begin{lemma}
\label{lem:goodevent_rfw}
	For linear MDPs, with inputs $\delta\in(0,1)$, $\rho>0$, $\epsilon_{\rm cov}>0$, \pref{alg:reward_free} will terminate in $T=\widetilde{\Theta}\left( \frac{dH/\rho^2 + d^4 H^4}{\epsilon_{\rm cov}} \polylog\left(\frac{1}{\delta}, \frac{1}{\rho}, \frac{1}{\epsilon_{\rm cov}}, d, H\right)\right)$ episodes, and output 
    $H$ datasets $\{\calD_{h}\}_{h=1}^H$ where $\calD_h\subset \calS_h\times \calA\times \calS_{h+1}$ such that with probability $\geq 1-\delta$, 
	\begin{align*}
		\forall h, \forall \pi, \quad 
  \sum_{s\in\calS_h} \mu^\pi(s) \ind\{s\notin \calZ_h\} \leq \epsilon_{\rm cov},
		 \text{\ \ where\ \ } 
		 \calZ_h \triangleq \left\{s\in \calS_h:~~ \forall a\in\calA,\ \  \norm{\phi(s, a)}_{\Lambda_{h}^{-1}}\leq \rho\right\} 
	\end{align*} 
with
    $
         \Lambda_h \triangleq I + \sum_{(s,a,s')\in \calD_h} \phi(s,a)\phi(s,a)^\top.  
   $
\end{lemma}
\begin{proof}[Proof of \pref{lem:goodevent_rfw}]
	Let  $T_h$ denote the number of episodes run by $\textsc{CoverTraj}$, by \pref{thm:wagenmaker_cover_traj}, 
	\begin{align*}
		T_h &\leq  	
			C \sum_{i=1}^m 2^i \max\left\{
				\frac{d}{\rho_i^2}\log \frac{2^i}{\rho_i^2}, 
				d^4 H^3 m^3  \log^{7/2}\frac1\delta
			\right\} \\
   &\le \otil\left(m2^m\left(\frac{d}{\rho^2}\log \left(\frac{2^m}{\rho^2}\right) + d^4H^3m^3 \log^{7/2}\frac{1}{\delta}\right)\right) \\
   &\le \otil\left( \frac{d/\rho^2 + d^4 H^3}{\epsilon_{\rm cov}} \polylog\left(\frac{1}{\delta}, \frac{1}{\epsilon_{\rm cov}}, \frac{1}{\rho}, d, H\right) \right). 
	\end{align*}
	Given that \pref{alg:reward_free} executes \textsc{CoverTraj}~$H$ times, the claim follows.
	For the claim on the un-reachability of $\calS_h \setminus \calZ_h$,
	fix $h\in [H]$, and observe that by \pref{lem:unknwon_states_weight}, w.p. $1-\delta/H$, for any $\pi$;
	\begin{align*}
		\sum_{s\in\calS_h} \mu^\pi(s)\ind\left\{ \exists a \text{ s.t. } \norm{\phi(s, a)}_{ \Lambda_{h}^{-1}} > \rho_m\right\}
		\leq \epsilon_{\rm cov},
	\end{align*} 
	where in the inequality we use that $\widetilde \Lambda_{h, i} \preceq \Lambda_{h}$.
	The proof is complete by a union bound over $h$.
\end{proof}

\section{Omitted Details in \pref{sec:efficient}}
\label{app:efficient}


We will be using several additional notations in the analysis. 

\begin{definition}[$\mu^\pi_h$, $\mu^k_h$, $\mu^\star_h$] Define $\mu^\pi_h(s)=\mu^\pi(s)\mathbb{I}\{s\in\calS_h\}$. By the definition of $\mu^\pi(s)$, we know that $\mu^\pi_h$ is a distribution over $\calS$ that is supported on $\calS_h$. Define $\mu^k_h=\mu^{\pi_k}_h$ and $\mu^\star_h=\mu^{\pi^\star}_h$. 
\end{definition}

\begin{definition}[$T^\pi_h$, $\E_h^{\pi}$, $\E_h^\star$] \label{def: E h pistar}
    We define $T_{h}^{\pi}$ be the distribution over trajectories $\{(s_i, a_i)\}_{i=1}^h$ for the first $h$ steps generated by policy $\pi$ and transition $P$. Then we define 
    $$\E_h^{\pi} \left[\cdot\right] = \E_{(s_i, a_i)_{i=1}^{h-1} \sim T_{h-1}^{\pi}}\E_{s \sim P(\cdot \mid s_{h-1}, a_{h-1})}\left[\cdot\right],$$
    where $[\cdot]$ can be a function of $(s_1, a_1, \ldots, s_{h-1}, a_{h-1}, s)$. 

    In the analysis, we will mainly consider the optimal policy $\pi^\star$. For notation simplicity, we write  $\E_{h}^\star\left[\cdot\right]=\E_{h}^{\pi^\star}\left[\cdot\right]$. 
\end{definition}

\begin{definition}[Good trajectory]
    For any trajectory $t = \{(s_h, a_h, s_{h+1})\}_{h=i}^j$ where $1 \le i \le j \le H$, if $s_h \in \calZ_h$ for any $h$, then we say $t$ is a \textit{good trajectory}. 
\end{definition}

\begin{definition}[$Q_k$]
    Define $Q_k(s,a)=Q^{\pi_k}(s,a;\ell_k)$. 
\end{definition}


\subsection{Regret Decomposition and Dilated Bonus Lemma}

\begin{lemma}
For any trajectory $t = \{(s_h, a_h, s_{h+1})\}_{h=i}^{j}$ with $1 \le i \le j \le H$ generated by any policy, we have
\begin{equation*}
    {\rm Pr}\left(t \text{ is not a  good trajectroy}\right) \le HK^{-\frac{1}{4}}
\end{equation*}
\label{lem:Pr good}
\end{lemma}
\begin{proof}
From \pref{lem:goodevent_rfw}, since we choose $\epsilon_{cov} = K^{-\frac{1}{4}}$, for any $h$ and $s_h$ generated by any policy, we have $P\left(t \notin \calZ_h \right) \le K^{-\frac{1}{4}}$. By union bound, we have 
\begin{equation*}
 \Pr\left(t \text{ is not a good trajectory}\right) = \Pr\left(\bigcup_{ i \le h \le j} s_h \notin \calZ_h \right) \le HK^{-\frac{1}{4}}
\end{equation*}
\end{proof}

In the regret decomposition below, we use the notation $\E^\star_h[\cdot]$ defined in \pref{def: E h pistar} to denote the expectation over trajectories $(s_1, a_1, \ldots, s_{h-1}, a_{h-1}, s_h=s)$ drawn from $\pi^\star$, and use $\calE_h$ to denote  the event that $\forall h' \le h, s_{h'} \in \calZ_{h'}$. By \pref{lem:Pr good}, we have $\E^\star_h[\ind\{\calE_h\}] \ge 1-HK^{-\frac{1}{4}}$ for any $h$.
By performance difference lemma \citep{kakade2002approximately}, we have
\begin{align}
	&\E\left[\calR_K\right] \nonumber
		\\&=\E\left[\sum_{k=1}^K \sum_{h=1}^H \E_{s \sim \mu_h^\star}
			\left[ \left\langle Q_{k}(s, \cdot), 
				\pi_k(\cdot|s) - \pi^\star(\cdot|s) \right\rangle \right]\right] \nonumber
        \\&= \E\left[\sum_{k=1}^K \sum_{h=1}^H \E_h^\star
			\left[ \left\langle Q_{k}(s, \cdot), 
				\pi_k(\cdot|s) - \pi^\star(\cdot|s) \right\rangle \right]\right] \nonumber
		\\&=\E\left[\sum_{k=1}^K \sum_{h=1}^H \E_h^\star
			\left[ \left\langle Q_{k}(s, \cdot), 
				\pi_k(\cdot|s) - \pi^\star(\cdot|s) \right\rangle \ind\{\calE_h\}\right]\right] + \E\left[\sum_{k=1}^K \sum_{h=1}^H \E_h^\star
			\left[ \left\langle Q_{k}(s, \cdot), 
				\pi_k(\cdot|s) - \pi^\star(\cdot|s) \right\rangle\ind\{\overline{\calE_h}\}\right] \right] \nonumber
        \\&\le \underbrace{\E\left[\sum_{k=1}^K \sum_{h=1}^H \E_h^\star
			\left[ \left\langle Q_{k}(s, \cdot), 
				\pi_k(\cdot|s) - \pi^\star(\cdot|s) \right\rangle \ind\{\calE_h\}\right]\right]}_{\regterm} + H^3K^{\frac{3}{4}}
\label{eqn:regret decomposition}
\end{align}
where the last step comes from \pref{lem:Pr good} and $Q_{k}(s,a) \le H$ for any $k,h,s,a$.

To handle \regterm, we utilize the dilated bonus technique proposed in \cite{luo2021policy}. We summarize the technique in \pref{lem:dilated bonus guarantee appendix}, with slight modification to make it align with our settings.

\begin{lemma}[Adaptation of Lemma 3.1 in \cite{luo2021policy}]
Suppose that for some bonus functions $b_{k}(s,a)$, $B_k(s,a)$ and some constants $f, g$, we have for all $s\in\calS_h$, 
\begin{equation}
B_{k}(s,a) \ge b_{k}(s,a) + \left(1 + \frac{1}{H}\right)\E_{s' \sim P(\cdot|s,a)}\E_{a' \sim \pi_k(\cdot|s')}\left[B_{k}(s', a')\ind\{s' \in \calZ_{h+1}\}\right] - f, 
\label{eqn:B condition}
\end{equation}
and suppose that our algorithm guarantees
\begin{align}
    &\E\left[\sum_{k=1}^K\sum_{h=1}^H\E_h^\star\left[\left\langle Q_{k}(s,\cdot) - B_{k}(s,a), \pi_k(\cdot|s) - \pi^\star(\cdot|s)\right\rangle\ind\{\calE_h\}\right] \right] \nonumber
    \\&\le g + \E\left[\sum_{k=1}^K\sum_{h=1}^H\E_h^\star\E_{a \sim \pi^\star(\cdot|s)}\left[b_{k}(s,a)\ind\{\calE_h\}\right]\right]  + \frac{1}{H}\E\left[\sum_{k=1}^K\sum_{h=1}^H\E_h^\star\E_{a \sim \pi_k(\cdot|s)}\left[B_{k}(s,a)\ind\{\calE_h\}\right]\right].
\label{eqn:regret condition}
\end{align}
Then, we have (recall the \regterm defined in the proof of \pref{eqn:regret decomposition})
\begin{align*}
    \regterm \le g + fHK + \left(1+\frac{1}{H}\right)\E\left[\sum_{k=1}^K\E_{a \sim \pi_k(\cdot|s_1)}\left[B_{k}(s_1,a)\right] \right]. 
\end{align*}
\label{lem:dilated bonus guarantee appendix}
\end{lemma}
\allowdisplaybreaks
\begin{proof}
Notice that for any function $X$ of $(s_1, a_1,\ldots, $ $s_{H}, a_H)$, it holds that
\begin{equation}
\E_h^\star\E_{a \sim \pi^\star(\cdot|s)}\E_{s' \sim P(\cdot|s,a)}\left[X \ind\{\calE_h\}\ind\{s' \in \calZ_{h+1}\}\right] = \E_{h+1}^\star\left[X\ind\{\calE_{h+1}\}\right].
\label{eqn:expectation eq appendix}
\end{equation}
By the definition of \regterm, we have
\begin{align*}
&\regterm
\\&= \E\left[\sum_{k=1}^K\sum_{h=1}^H\E_h^\star\left[\left\langle Q_{k}(s,\cdot), \pi_k(\cdot|s) - \pi^\star(\cdot|s)\right\rangle\ind\{\calE_h\}\right] \right]
    \\&\le g + \E\left[\sum_{k=1}^K\sum_{h=1}^H\E_h^\star\E_{a \sim \pi^\star(\cdot|s)}\left[b_{k}(s,a)\ind\{\calE_h\}\right]\right]  + \frac{1}{H}\E\left[\sum_{k=1}^K\sum_{h=1}^H\E_h^\star\E_{a \sim \pi_k(\cdot|s)}\left[B_{k}(s,a)\ind\{\calE_h\}\right] \right]
    \\&\quad + \E\left[\sum_{k=1}^K\sum_{h=1}^H\E_h^\star\left[\left\langle  B_{k}(s,\cdot), \pi_k(\cdot|s) \right\rangle\ind\{\calE_h\}\right] \right] - \E\left[\sum_{k=1}^K\sum_{h=1}^H\E_h^\star\left[\left\langle  B_{k}(s,\cdot), \pi^\star(\cdot|s)\right\rangle\ind\{\calE_h\}\right] \right]
    \tag{by \pref{eqn:regret condition}}
    \\&\le  g + fHK + \E\left[\sum_{k=1}^K\sum_{h=1}^H\E_h^\star\E_{a \sim \pi^\star(\cdot|s)}\left[b_{k}(s,a)\ind\{\calE_h\}\right]\right]  
    \\&\quad + \left(1+\frac{1}{H}\right)\E\left[\sum_{k=1}^K\sum_{h=1}^H\E_h^\star\E_{a \sim \pi_k(\cdot|s)}\left[B_{k}(s,a)\ind\{\calE_h\}\right] \right] - \E\left[\sum_{k=1}^K\sum_{h=1}^H\E_h^\star\E_{a \sim \pi^\star(\cdot|s)}\left[b_{k}(s,a)\ind\{\calE_h\}\right] \right]
    \\&\qquad -  \left(1 + \frac{1}{H}\right)\E\left[\sum_{k=1}^K\sum_{h=1}^H\E_h^\star\E_{a \sim \pi^\star(\cdot|s)}\E_{s' \sim P(\cdot|s,a)}\E_{a' \sim \pi_k(\cdot|s')}\left[B_{k}(s', a')\ind\{\calE_h\}\ind\{s' \in \calZ_{h+1}\}\right]\right]\tag{by \pref{eqn:B condition}}
    \\&=  g + fHK + \left(1+\frac{1}{H}\right)\E\left[\sum_{k=1}^K\sum_{h=1}^H\E_{h}^\star\E_{a \sim \pi_k(\cdot|s)}\left[B_{k}(s,a)\ind\{\calE_h\}\right] \right] 
    \\&\quad- \left(1+\frac{1}{H}\right)\E\left[\sum_{k=1}^K\sum_{h=1}^H\E_{h+1}^\star\E_{a \sim \pi^\star(\cdot|s)}\left[B_{k}(s,a)\ind\{\calE_{h+1}\}\right] \right] \tag{by \pref{eqn:expectation eq appendix}}
    \\&=  g + fHK + \left(1+\frac{1}{H}\right)\E\left[\sum_{k=1}^K \E^\star_1 \E_{a \sim \pi_k(\cdot|s)}\left[B_{k}(s,a)\ind\{\calE_1\}\right] \right] \tag{telescoping}\\
    &= g + fHK + \left(1+\frac{1}{H}\right)\E\left[\sum_{k=1}^K \E_{a \sim \pi_k(\cdot|s_1)}\left[B_{k}(s_1,a)\right] \right].   \tag{$\calS_1=\{s_1\}$ and $s_1\in \calZ_1$} 
\end{align*}
\end{proof}
In the following \pref{app: show bonus part} and \pref{app: regret analysis Q-B}, we aim to show that our \pref{alg: logdet FTRL} and \pref{alg: dilated bonus} could induce bonus functions $b_k(s,a), B_{k}(s,a)$ that satisfy the condition of \pref{lem:dilated bonus guarantee appendix}. This allows us to directly apply it and get the desired regret bound in \pref{app: combining }. Our choices of $B_k(s,a)$ and $b_k(s,a)$ are the following: 
\begin{framed}
For $s\in\calS_h$, $a\in\calA$, 
   \begin{align}
      b_{k}(s,a) &= \beta\|\phi(s,a)\|_{\hatSigma^{-1}_{k,h}}^2 + \left(1-\frac{1}{4H}\right)\alpha \|\phi(s,a)\|_{\Lambda_{k, h}^{-1}}^2 
\label{eqn: b define app}  \\
B_{k}(s,a) &= b_k(s,a) + \phi(s, a)^\top w_{k,h} \label{eqn: B define app}
\end{align}
where 
\begin{align}
    w_{k,h} = \left(1+\frac{1}{H}\right)\sum_{s' \in \calS_{h+1}}\psi(s')\widehat{W}_{k}(s')\ind\{s' \in \calZ_{h+1}\} \qquad (w_{k,H}\triangleq 0)
\label{eqn:true w}
\end{align}
with the $\widehat{W}_{k}(s')$ defined in \pref{alg: dilated bonus}. 
\end{framed}

\subsection{Construction of Dilated Bonus (achieving \pref{eqn:B condition} using \pref{alg: dilated bonus})}\label{app: show bonus part}
In the linear regression (\pref{line: regression}) of \pref{alg: dilated bonus}, the $\widehat{w}_{k,h}$ is an estimation of $w_{k,h}$ defined in \pref{eqn:true w}, 
where for $s' \in \calS_{h+1}$, 
\begin{align}
   &\widehat{W}_{k}(s')  =\E_{a'\sim\pi_k(\cdot|s')}\left[\left[\beta\|\phi(s',a')\|_{\hatSigma^{-1}_{k,h+1}}^2 +   \alpha \|\phi(s',a')\|_{\Lambda_{k, h+1}^{-1}}^2 + \phi(s', a')^\top \widehat{w}_{k,h+1}\right]^+\right],   \label{eq: What def}
\end{align}
with $[x]^+$ denoting $\max\{x, 0\}$. 
 

The next \pref{lem: uniform w concen} is a key lemma that 1) bounds the error between $\widehat{w}_{k,h}$ and $w_{k,h}$, and 2) bounds the magnitude of $\widehat{w}_{k,h}$ and $w_{k,h}$ for all $h\in[H]$. 
\begin{lemma}\label{lem: uniform w concen}
Let $\con=15\sqrt{\log\left(\frac{12dK}{\delta}\right)}$ and suppose that \highlight{$B_h^{\max}\leq \frac{\alpha}{\con^2 Hd^2}$}. Then with probability at least $1-\delta$, the following inequalities hold for all $k\in[K]$, $h\in[H]$, and all $s\in\calS_h$:  
\begin{align}
    &\|w_{k,h}\|_2  \leq  \sqrt{d}B_h^{\max}, \label{eq: induct1}
    \\
    &\left|\phi(s,a)^\top \widehat{w}_{k,h} - \phi(s,a)^\top w_{k,h} \right| \le \con dB_h^{\max}\|\phi(s,a)\|_{\Lambda_{k,h}^{-1}}, \label{eq: induct2} 
    \\
    &|\phi(s,a)^\top \widehat{w}_{k,h}|\ind\{s \in \calZ_{h}\} \le \left(1+\frac{1}{2H}\right)B_{h}^{\max}.  \label{eq: induct3}
\end{align}
\end{lemma}
\begin{proof}
We use induction to prove these three inequalities. For the base case $h=H$, we have $w_{k,H}=\mathbf{0}$ and $\widehat{w}_{k,H}=\mathbf{0}$, so all three inequalities holds.  

Suppose that all three inequalities holds for the case of $h+1$. Below, we show that that also holds for $h$. 

\paragraph{Showing \pref{eq: induct1}. }  
Observe that for any $s'\in\calS_{h+1}$, 
\begin{align}
    &\left(1+\frac{1}{H}\right) \widehat{W}_k(s')\ind\{s'\in\calZ_{h+1}\}   \nonumber 
    \\&\leq \max_{a'\in\calA}\left(1+\frac{1}{H}\right) \left(\beta\|\phi(s',a')\|_{\hatSigma^{-1}_{k,h+1}}^2 + \alpha \|\phi(s',a')\|_{\Lambda_{k, h+1}^{-1}}^2 + |\phi(s', a')^\top \widehat{w}_{k,h+1}| \right) \ind\{s' \in \calZ_{h+1}\}    \nonumber 
    \\&\leq \left(1+\frac{1}{H}\right)\left(\frac{\beta}{\gamma} + \alpha\rho^2 \right) + \left(1+\frac{1}{H}\right)\left(1+\frac{1}{2H}\right)B_{h+1}^{\max}    \tag{$\|\phi(s',a')\|_{\Lambda_{k,h+1}^{-1}}\leq \rho$ for $s'\in\calZ_{h+1}$ by \pref{alg:reward_free}; using induction hypothesis \pref{eq: induct3} for $h+1$}
    \\&\leq  \left(1+\frac{1}{H}\right) \frac{1}{2H}B^{\max}_{h+1} + \left(1+\frac{1}{H}\right)\left(1+\frac{1}{2H}\right)B_{h+1}^{\max}  \tag{by the definition of $B^{\max}_{h+1}$}
    \\&\leq \left(1+\frac{1}{H}\right)^2 B_{h+1}^{\max} \nonumber 
    \\&= B_{h}^{\max}.    \label{eq: bound on hat W}
\end{align}
Thus, 
\begin{align*}
    \|w_{k,h}\|_2 = \left\|\left(1+\frac{1}{H}\right)\sum_{s'\in\calS_{h+1}}\psi(s') \widehat{W}_k(s')\ind\{s'\in\calZ_{h+1}\}\right\|_2 \leq B_{h}^{\max} \left\|\sum_{s' \in \calS_{h+1}} \psi(s')\right\|_2\leq \sqrt{d}B_{h}^{\max}  
\end{align*}
where in the last inequality we use the linear MDP assumption (\pref{def: linear MDP}). 
\paragraph{Showing \pref{eq: induct2}. }
\begin{align}
     \left|\phi(s,a)^\top \widehat{w}_{k,h} -  \phi(s,a)^\top w_{k,h}\right| 
     \le \|\phi(s,a)\|_{\Lambda_{k,h}^{-1}}\|\widehat{w}_{k,h} -  w_{k,h}\|_{\Lambda_{k,h}}. 
     \label{eq: first expansion}
\end{align}

By \pref{lem:regression difference} and $\|w_{k,h}\| \le \sqrt{d}B_h^{\max}$ (which we just proved), it holds that 
\begin{align}
&\|\widehat{w}_{k,h} - w_{k,h}\|_{\Lambda_{k,h}}  \nonumber  
\\&\le\left\|\sum_{(s,a,s') \in \calD_{k,h}}\phi(s,a)\left(\left(1+\frac{1}{H}\right)\widehat{W}_{k}(s')\ind\{s' \in \calZ_{h+1}\} - \phi(s,a)^\top w_{k,h}\right)\right\|_{\Lambda_{k,h}^{-1}} + \sqrt{d}B_h^{\max}.
\label{eqn:w local first m}
\end{align}
By \pref{lem:uniform concen}, the first term above can be upper bounded by   
\begin{align}
\sqrt{4(B_{h}^{\max})^2\left(\frac{d}{2}\log{K} + \log{\frac{\calN_\epsilon \left(\calV_h\right)}{\delta}}\right) + 8K^2\epsilon^2}.   \label{eq: tmp bound for Nepsilon}
\end{align}
where $\calV_h$ is the function class where $\left(1+\frac{1}{H}\right)\widehat{W}_k(s')\ind\{s'\in\calZ_{h+1}\}$ lies, and $\calN_\epsilon(\calV_h)$ is its $\epsilon$-covering number. 
By the form of $\widehat{W}_k(s')$ given in \pref{eq: What def}, $\calV_h$ can be chosen as the  that defined in \pref{def: function class V}. Then by \pref{lem:cover number} with $\epsilon = \frac{1}{K}$ and $\frac{\beta}{\gamma} + 2\alpha \le K^2$, we have
\begin{align*}
\log\left(\calN_\epsilon \left(\calV_h \right)\right) \le 4(d+1)^2\log\left(400(d+1)^2K^3\right) \le 48d^2\log\left(12dK\right)
\end{align*}
Combining this with \pref{eqn:w local first m} and \pref{eq: tmp bound for Nepsilon}, we get 
\begin{align*}
\|\widehat{w}_{k,h} - w_{k,h}\|_{\Lambda_{k,h}} \le 15dB_h^{\max}\sqrt{\log\left(\frac{12dK}{\delta}\right)}. 
\end{align*}
Further combining this with \pref{eq: first expansion} proves \pref{eq: induct2}. 

\paragraph{Showing \pref{eq: induct3}. }
\begin{align*}
&\left|\phi(s,a)^\top \widehat{w}_{k,h}\right|\ind\{s \in \calZ_{h}\} 
\\&\leq \left|\phi(s,a)^\top w_{k,h}\right|\ind\{s \in \calZ_{h}\} + \left|\phi(s,a)^\top \left(\widehat{w}_{k,h} - w_{k,h}\right)\right|\ind\{s \in \calZ_{h}\} 
\\&\leq \left(1+\frac{1}{H}\right)\sup_{s'\in\calS_{h+1}} \widehat{W}_k(s')\ind\{s'\in\calZ_{h+1}\} + \con dB_h^{\max}\|\phi(s,a)\|_{\Lambda_{k,h}^{-1}} \tag{by the definition of $w_{k,h}$ and \pref{eq: induct2}}
\\&\le B_h^{\max} + \left(\frac{(\con d B_h^{\max})^2}{4\alpha} + \alpha\|\phi(s,a)\|_{\Lambda_{k,h}^{-1}}^2\right) \ind\{s\in\calZ_h\} \tag{by \pref{eq: bound on hat W} and AM-GM inequality} 
\\& \leq B_h^{\max} + \left( \frac{1}{4H}B^{\max}_h + \alpha\rho^2\right) \tag{by the condition specified in the lemma and that $\|\phi(s,a)\|_{\Lambda_{k,h}^{-1}}\leq \rho$ for $s\in\calZ_h$}
\\& \leq B_h^{\max} + \frac{1}{2H}B^{\max}_h\tag{by the definition of $B_h^{\max}$} 
\end{align*}
This proves \pref{eq: induct3}.  
\end{proof}

\begin{lemma}With the definition of \pref{eqn: b define app} and \pref{eqn: B define app}, any $s\in \calS_h$, we have
\begin{align*}
B_{k}(s, a) &\ge b_{k}(s,a) + \left(1+\frac{1}{H}\right) \E_{s' \sim P(\cdot|s,a)}\E_{a' \sim \pi_k(\cdot|s')}\left[B_{k}(s',a') \ind\{s' \in \calZ_{h+1}\}\right] - \frac{(\con dB^{\max})^2}{\alpha}.
\end{align*}
where $B^{\max}\triangleq \max_{h\in[H]} B_h^{\max}$ and $\con$ is a logarithmic term defined in \pref{lem: uniform w concen}.  
\label{lem: B condition satisfy}
\end{lemma}

\begin{proof}
Recall the definition of $w_{k,h}$ in \pref{eqn:true w}, from the definition of linear MDP, for all $k,h,s,a$, we have 
\begin{align*}
&\phi(s,a)^\top w_{k,h} 
\\&=\left(1+\frac{1}{H}\right) \E_{s' \sim P(\cdot|s,a)}\left[\widehat{W}(s')\ind\{s' \in \calZ_{h+1}\}\right]
\\&=  \left(1+\frac{1}{H}\right)\E_{s' \sim P(\cdot|s,a)}\E_{a' \sim \pi_k(\cdot| s)}\left[\widehat{B}_{k}^+(s',a')\ind\{s' \in \calZ_{h+1}\}\right] 
\\&\ge \left(1+\frac{1}{H}\right)\E_{s' \sim P(\cdot|s,a)}\E_{a' \sim \pi_k(\cdot| s)}\left[\widehat{B}_{k}(s',a')\ind\{s' \in \calZ_{h+1}\}\right] 
\\&= \left(1+\frac{1}{H}\right)\E_{s' \sim P(\cdot|s,a)}\E_{a' \sim \pi_k(\cdot| s)}\left[\left(B_{k}(s',a') + \frac{\alpha}{4H}\|\phi(s',a')\|_{\Lambda_{k,h+1}^{-1}}^2 + \phi(s',a')^\top\left(\widehat{w}_{k,h+1} - w_{k,h+1}\right)\right)\ind\{s' \in \calZ_{h+1}\}\right] \tag{by the definition of $\widehat{B}_k(s',a')$ in \pref{line: hat B def} and $B_k(s',a')$ in \pref{eqn: B define app}}
\\&\ge \left(1+\frac{1}{H}\right) \E_{s' \sim P(\cdot|s,a)}\E_{a' \sim \pi_k(\cdot|s')}\left[B_{k}(s',a') \ind\{s' \in \calZ_{h+1}\}\right] -\frac{(\con dB^{\max})^2}{\alpha} \tag{\pref{eq: induct2} and AM-GM inequlity}
\end{align*}
Thus, we have
\begin{align*}
&B_{k}(s, a) 
\\&= b_k(s,a) + \phi(s,a)^\top w_{k,h} 
\\&\ge b_k(s,a) + \left(1+\frac{1}{H}\right) \E_{s' \sim P(\cdot|s,a)}\E_{a' \sim \pi_k(\cdot|s')}\left[B_{k}(s',a') \ind\{s' \in \calZ_{h+1}\}\right]  - \frac{(\con dB^{\max})^2}{\alpha}. 
\end{align*}
\end{proof}


\subsection{Regret Analysis (achieving \pref{eqn:regret condition} using \pref{alg: logdet FTRL})} \label{app: regret analysis Q-B}
The goal of this subsection is to prove \pref{eqn:regret condition} for the definitions of $b_k(s,a)$ and $B_k(s,a)$ in \pref{eqn: b define app} and \pref{eqn: B define app}. We first decompose the left-hand side of \pref{eqn:regret condition}. 
\begin{align}
&\E\left[\sum_{k=1}^K \sum_{h=1}^H \E_h^\star\left[\left\langle Q_{k}(s,\cdot) - B_{k}(s,a), \pi_k(\cdot|s) - \pi^\star(\cdot|s)\right\rangle\ind\{\calE_h\}\right] \right] \nonumber
\\
&\le \underbrace{\E\left[\sum_{k=1}^K \sum_{h=1}^H \E_h^\star \left[ \left\langle Q_{k}(s, \cdot) - \widehat{Q}_{k}(s, \cdot), 
				\pi_k(\cdot|s) \right\rangle\ind\{\calE_h\}\right]
		\right]}_{\biasone}
		\nonumber \\
  &\qquad + \underbrace{\E\left[\sum_{k=1}^K \sum_{h=1}^H \E_h^\star\left[ \left\langle \widehat{Q}_{k}(s, \cdot) - Q_{k}(s, \cdot), 
				\pi^\star(\cdot|s) \right\rangle\ind\{\calE_h\}\right]
		\right]}_{\biastwo} \nonumber
        \\
   &\qquad + \underbrace{\E\left[\sum_{k=1}^K \sum_{h=1}^H \E_h^\star
			\left[ \left\langle \pmb{\widehat{\Gamma}}_{k,h} - \pmb{\widehat{B}}_{k,h}, 
				\pmb{H}_k(s) - \pmb{H}_{\star}(s) \right\rangle \ind\{\calE_h\}\right]
		\right]}_{\ftrl}  \nonumber  \\
  &\qquad + \underbrace{\E\left[\sum_{k=1}^K \sum_{h=1}^H \E_h^\star \left[ \left\langle \widehat{B}_{k}(s,\cdot) - B_{k}(s,\cdot), \pi_k(\cdot|s) - \pi^\star(\cdot|s) \right\rangle\ind\{\calE_h\} \right]\right]}_{\biasthree}
\label{eqn:reg-term decomposition}
\end{align}
where we use that for $s\in\calS_h$,  $\E_{a\sim \pi(\cdot|s)}\widehat{Q}_k(s,a) = \langle  \widehat{\Cov}(s, \pi(\cdot|s)), \pmb{\widehat{\Gamma}}_{k,h}\rangle$ and $\E_{a\sim \pi(\cdot|s)}\widehat{B}_k(s,a) = \langle  \widehat{\Cov}(s, \pi(\cdot|s)), \pmb{\widehat{B}}_{k,h}\rangle$, and we define $\pmb{H}_k(s) = \widehat{\Cov}(s, \pi_k(\cdot|s))$, $\pmb{H}_\star(s) = \widehat{\Cov}(s, \pi^\star(\cdot|s))$.  

We further deal with the \ftrl term. This term is analyzed through the standard FTRL analysis. 
In order to deal with the issue that $F$ can be unbounded on the boundary of $\mathcal{H}_s$, we define the following auxiliary comparator: 
\begin{align*}
\overline{\pmb{H}}_{\star}(s) = \left(1-\frac{1}{K^3}\right)\pmb{H}_{\star}(s) + \frac{1}{K^3}\pmb{H}_{\min}(s)
\end{align*}
where $\pmb{H}_{\min}(s) = \argmin\limits_{\pmb{H}\in \calH_s}F(\pmb{H})$

Applying \pref{lem:block FTRL} for logdet FTRL, we have 
\begin{align}
&\ftrl = \E\left[\sum_{k=1}^K \sum_{h=1}^H \E_h^\star
			\left[ \left\langle \pmb{\widehat{\Gamma}}_{k,h} - \pmb{\widehat{B}}_{k,h}, 
				\pmb{H}_k(s) - \pmb{H}_{\star}(s) \right\rangle \ind\{\calE_h\}\right]
		\right] \nonumber
\\&= \E\left[\sum_{k=1}^K \sum_{h=1}^H \E_h^\star
			\left[ \left\langle \pmb{\widehat{\Gamma}}_{k,h} - \pmb{\widehat{B}}_{k,h}, 
				\pmb{H}_k(s) - \overline{\pmb{H}}_{\star}(s) \right\rangle \ind\{\calE_h\}\right]
		\right]  \nonumber \\
  &\qquad + \E\left[\sum_{k=1}^K \sum_{h=1}^H \E_h^\star
			\left[ \left\langle \pmb{\widehat{\Gamma}}_{k,h} - \pmb{\widehat{B}}_{k,h}, 
				\overline{\pmb{H}}_{\star}(s) - \pmb{H}_{\star}(s) \right\rangle \ind\{\calE_h\}\right]
		\right] \nonumber
\\&\le \underbrace{\E_h^\star\left[\frac{\tau\left(F\left( \overline{\pmb{H}}_{\star}(s) \right) - \min_{\pmb{H}\in \calH_s}F(\pmb{H}) \right)}{\eta} \ind\{\calE_h\}\right]}_{\penal} \nonumber
\\&\quad + \underbrace{\E\left[\sum_{k=1}^K\sum_{h=1}^H\E_h^\star\left[ \left(\max_{\pmb{H} \in \calH_s} \langle \pmb{H}_k(s) - \pmb{H}, \pmb{\hGamma}_{k,h}  \rangle - \frac{D_F(\pmb{H}, \pmb{H}_k(s))}{2\eta}\right) \ind\{\calE_h\}\right]\right]}_{\stabone}  \nonumber
\\ &\quad \quad +  \underbrace{\E\left[\sum_{k=1}^K\sum_{h=1}^H\E_h^\star\left[ \left(\max_{\pmb{H} \in \calH_s} \langle \pmb{H}_k(s) - \pmb{H}, -\pmb{\widehat{B}}_{k,h}  \rangle - \frac{D_F(\pmb{H}, \pmb{H}_k(s))}{2\eta}\right) \ind\{\calE_h\}\right]\right]}_{\stabtwo}  \nonumber
\\&\quad \quad \quad+ \underbrace{\E\left[\sum_{k=1}^K \sum_{h=1}^H \E_h^\star
			\left[ \left\langle \pmb{\widehat{\Gamma}}_{k,h} - \pmb{\widehat{B}}_{k,h}, 
				\overline{\pmb{H}}_{\star}(s) - \pmb{H}_{\star}(s) \right\rangle \ind\{\calE_h\}\right]
		\right]}_{\error}
\label{eqn:FTRL decomposition}
\end{align}

Below, we further bound the individual terms in \pref{eqn:reg-term decomposition} and \pref{eqn:FTRL decomposition}. 

\subsubsection{Bound \biasone, \biastwo, \biasthree in \pref{eqn:reg-term decomposition}}
\begin{lemma}
For any policy $\pi_k$, there exists a ${\rm q}_{k,h}$ such that for any $s \in \calS_h$, $Q_k(s,a) = \phi(s,a)^\top{\rm q}_{k,h}$. Moreover, $\|{\rm q}_{k,h}\|_2 \le H\sqrt{d}$. 
\label{lem:q linear}
\end{lemma}
\begin{proof}
Define ${\rm q}_{k,h} = \theta_{k,h} + \sum_{s' \in \calS_{h+1}}\psi(s')\E_{a' \sim \pi_k(\cdot|s')}\left[Q_{k}(s',a')\right]$, we have
\begin{align*} 
Q_k(s,a) = Q^{\pi_k}(s,a;\ell_k) &= \ell_k(s,a) + \E_{s' \sim P(\cdot|s,a)}\E_{a' \sim \pi_k(\cdot|s')}\left[Q_{k}(s',a')\right]
\\&= \phi(s,a)^\top\left(\theta_{k,h} + \sum_{s' \in \calS_{h+1}}\psi(s')\E_{a' \sim \pi_k(\cdot|s')}\left[Q_{k}(s',a')\right]\right) 
\\&= \phi(s,a)^\top{\rm q}_{k,h}.
\end{align*}
Moreover, 
\begin{align*}
\|{\rm q}_{k,h}\|_2 = \left\|\theta_{k,h} + \sum_{s' \in \calS_{h+1}}\psi(s')\E_{a' \sim \pi_k(\cdot|s')}\left[Q_{k}(s',a')\right]\right\|_2 \le \sqrt{d} + \sqrt{d}(H-1) = \sqrt{d}H.
\end{align*}
\end{proof}

\begin{lemma}\label{lem:matrix concentration}
Let $\Sigma_{k,h} = \E_{s\sim \mu^{k}_h}\E_{a\sim \pi_k(\cdot|s)}\left[\phi(s,a)\phi(s,a)^\top\right]$.  If $\gamma \ge   \frac{5d\log\left(6dHK/\delta\right)}{\tau}$, then with probability of $1-\delta$, for all $k,h$, 
\begin{equation*}
    \left\| \left(\hatSigma_{k,h} - \Sigma_{k,h} \right) {\rm q}_{k,h} \right\|_{\hatSigma^{-1}_{k,h}}^2 \le \order\left( \frac{d^2H^2\log\left(dHK/\delta\right)}{\tau} \right)
\end{equation*}
\end{lemma}
\begin{proof}
    This follows the fact the $\|{\rm q}_{k,h}\|_2 \le H\sqrt{d}$ given in \pref{lem:q linear} and the matrix concentration bound in Lemma 14 of \cite{liu2023bypassing} with a union bound over $k,h$. Taking a union bound for all $k,h$ finishes the proof.
\end{proof}

\begin{lemma}
If \highlight{$\gamma \ge   \frac{5d\log\left(6dHK/\delta\right)}{\tau}$}, then
\begin{align*}
&\biasone \le \otil\left(\frac{d^2H^3}{\tau\beta}K\right) + \frac{\beta}{4H} \E\left[\sum_{k=1}^K \sum_{h=1}^H \E_h^\star\E_{a \sim \pi_k(\cdot|s)} \left[\|\phi(s,a)\|_{\hatSigma^{-1}_{k,h}}^2  \ind\{\calE_h\}\right]\right]
\\ &\biastwo \le \otil\left(\frac{d^2H^3}{\tau\beta}K\right) + \frac{\beta}{4H} \E\left[\sum_{k=1}^K \sum_{h=1}^H \E_h^\star\E_{a \sim \pi^\star(\cdot|s)} \left[\|\phi(s,a)\|_{\hatSigma^{-1}_{k,h}}^2  \ind\{\calE_h\}\right]\right].
\end{align*}
\label{lem: bias 12 bound}
\end{lemma}

\begin{proof}
Let $\E_k\left[\cdot\right]$ be the expectation conditioned on history up to episode $k-1$. We have 
\begin{align*}
    \E_k \left[\sum_{t=h}^H \ell_{k,t} \right] =  \E_k \left[ Q_{k}(s_{k,h}, a_{k,h}) \right] = \E_k \left[ \phi(s_{k,h}, a_{k,h})^\top {\rm q}_{k,h} \right].
\end{align*}

Therefore,
\begin{align*}
\E_k\left[\hq_{k,h}\right] &= \E_k\left[\hatSigma_{k,h}^{-1} \phi(s_{k,h}, a_{k,h})\phi(s_{k,h}, a_{k,h})^\top {\rm q}_{k,h}\right]
= \hatSigma_{k,h}^{-1} \Sigma_{k,h} {\rm q}_{k,h}, 
\end{align*}
and for $s\in\calS_h$, 
\begin{align*}
\E_k\left[ Q_{k}(s, a) - \hatQ_{k}(s, a) \right] &=  \E_k\left[ \phi(s,a)^\top {\rm q}_{k,h} - \phi(s,a)^\top \hq_{k,h} \right]
\\&= \phi(s,a)^\top \left(I -  \hatSigma_{k,h}^{-1} \Sigma_{k,h} \right) {\rm q}_{k,h}
\\&= \phi(s,a)^\top \hatSigma_{k,h}^{-1} \left(\hatSigma_{k,h} - \Sigma_{k,h} \right) {\rm q}_{k,h}
\\&\le \|\phi(s,a)\|_{\hatSigma^{-1}_{k,h}} \left\|\left(\hatSigma_{k,h} - \Sigma_{k,h} \right) {\rm q}_{k,h}\right\|_{\hatSigma^{-1}_{k,h}}\tag{Cauchy-Schwarz}
\\&\le \order\left( \sqrt{\frac{d^{2}H^2\log\left(dK/\delta\right)}{\tau}}\|\phi(s,a)\|_{\hatSigma^{-1}_{k,h}}\right)\tag{\pref{lem:matrix concentration}}
\\&\le  \order\left(\frac{d^2H^3\log\left(dK/\delta\right)}{\tau\beta}\right) + \frac{\beta }{4H}\|\phi(s,a)\|_{\hatSigma^{-1}_{k,h}}^2. \tag{AM-GM inequality}
\end{align*}
Thus, 
\begin{align*}
\biasone &= \E\left[\sum_{k=1}^K \sum_{h=1}^H \E_h^\star\left[ \left\langle Q_{k}(s, \cdot) - \widehat{Q}_k(s, \cdot), 
				\pi_k(\cdot|s) \right\rangle \ind\{\calE_h\}\right]\right]
\\&\le \otil\left(\frac{d^{2}H^3}{\tau\beta}K\right) + \frac{\beta}{4H}\E\left[\sum_{k=1}^K \sum_{h=1}^H \E_h^\star\E_{a \sim \pi_{k}(\cdot|s)} \left[\|\phi(s,a)\|_{\hatSigma^{-1}_{k,h}}^2 \ind\{\calE_h\}\right]\right]
\end{align*}
\noindent Similarly, we can prove
\begin{align*}
\biastwo \le \otil\left(\frac{d^{2}H^3}{\tau\beta}K\right) + \frac{\beta}{4H}  \E\left[\sum_{k=1}^K \sum_{h=1}^H \E_h^\star\E_{a \sim \pi^\star(\cdot|s)} \left[\|\phi(s,a)\|_{\hatSigma^{-1}_{k,h}}^2 \ind\{\calE_h\}\right]\right]
\end{align*} 
\end{proof}

\begin{lemma} 
Suppose that $B_h^{\max}\leq \frac{\alpha}{\con^2 Hd^2}$ where $\con=15\sqrt{\log\left(\frac{12dK}{\delta}\right)}$. Then
\begin{align*}
\biasthree \le  \otil\left(\frac{H^2d^2 (B^{\max})^2}{\alpha}K\right) + \frac{\alpha}{2H}\E\left[\sum_{k=1}^K \sum_{h=1}^H \E_h^\star\E_{a \sim \pi_k(\cdot|s)} \left[\|\phi(s,a)\|_{\Lambda^{-1}_{k,h}}^2 \ind\{\calE_h\}\right] \right]. 
\end{align*}
\label{lem:bias 3 bound}
\end{lemma}

\begin{proof}
By \pref{eq: induct2} and AM-GM inequality, we have that with probability at least $1-\delta$, for all $k,h,s,a$, $\left|\phi(s,a)^\top\left(\widehat{w}_{k,h} - w_{k,h}\right)\right| \le \frac{H(\con dB^{\max})^2}{\alpha} +  \frac{\alpha}{4H}\|\phi(s,a)\|_{\Lambda^{-1}_{k,h}}^2$. Combining this with the definitions of $\widehat{B}_k(s,a)$ in \pref{line: hat B def} and $B_k(s,a)$ in \pref{eqn: B define app}, we get 
\begin{align*}
&\widehat{B}_k(s,a) - B_k(s,a) = \frac{\alpha}{4H}\|\phi(s,a)\|_{\Lambda^{-1}_{k,h}}^2 + \phi(s,a)^\top\left(\widehat{w}_{k,h} - w_{k,h}\right) \ge -\frac{H(\con dB^{\max})^2}{\alpha}
\\&\widehat{B}_k(s,a) - B_k(s,a) \le \frac{\alpha}{2H}\|\phi(s,a)\|_{\Lambda^{-1}_{k,h}}^2 + \frac{H(\con dB^{\max})^2}{\alpha}. 
\end{align*}
With the two inequalities above, we have 
\begin{align*} 
&\biasthree \\
&= 
\E\left[\sum_{k=1}^K \sum_{h=1}^H \E_h^\star \left[ \left\langle \widehat{B}_{k}(s,\cdot) - B_{k}(s,\cdot), \pi_k(\cdot|s) - \pi^\star(\cdot|s) \right\rangle\ind\{\calE_h\} \right]\right] \\
&\leq \E\left[\sum_{k=1}^K \sum_{h=1}^H \E_h^\star \E_{a\sim \pi_k(\cdot|s)} \left[\widehat{B}_k(s,a) - B_{k}(s,a)\right] \right] - \E\left[\sum_{k=1}^K \sum_{h=1}^H \E_h^\star \E_{a\sim \pi^\star(\cdot|s)} \left[\widehat{B}_k(s,a) - B_{k}(s,a)\right] \right] \\
&\leq \otil\left(\frac{H^2(dB^{\max})^2}{\alpha}K\right) + \frac{\alpha}{2H}\left[\sum_{k=1}^K \sum_{h=1}^H \E_h^\star\E_{a \sim \pi_k(\cdot|s)} \left[\|\phi(s,a)\|_{\Lambda^{-1}_{k,h}}^2 \ind\{\calE_h\}\right]\right].  
\end{align*}
\end{proof}

\subsubsection{Bound \penal in \pref{eqn:FTRL decomposition}}
\begin{lemma}
 $\penal \le \frac{3d\tau\log(K)}{\eta}$
 \label{lem: penalty bound}
\end{lemma}

\begin{proof}
Since $\overline{\pmb{H}}_{\star}(s) = \left(1-\frac{1}{K^3}\right)\pmb{H}_{\star}(s) + \frac{1}{K^3}\pmb{H}_{\min}(s)$, we have $\overline{\pmb{H}}_{\star}(s) \succeq \frac{1}{K^3}\pmb{H}_{\min}(s)$. Then 
\begin{equation*}
   \frac{\tau\left(F(\overline{\pmb{H}}_{\star}(s)) - \min_{\pmb{H} \in \mathcal{H}_s}F(\pmb{H})\right)}{\eta} = \frac{\tau}{\eta}\log{\frac{\det(\pmb{H}_{\min}(s))}{\det(\overline{\pmb{H}}_{\star}(s))}} \le \frac{3d\tau\log(K)}{\eta}
\end{equation*}
\end{proof}

\subsubsection{Bound \error in \pref{eqn:FTRL decomposition}}
\begin{lemma}$\error \le  \order\left(H\right)$.  
\label{lem:error bound}
\end{lemma}
\begin{proof}
By the choices of $\beta,\gamma,\alpha$, it holds that $\frac{\beta}{\gamma} + \alpha\rho^2 \le \order(K)$ and $\frac{H}{\gamma} \le \order(K)$.  
Let  $\pi_{\min}$ be such that $\pmb{H}_{\min}(s) = \E_{a \sim \pi_{\min}(\cdot|s)}\begin{bmatrix}
    \phi(s,a)\phi(s,a)^\top & \phi(s,a)\\
    \phi(s,a)^\top & 1
\end{bmatrix}$. For $s\in\calS_h$, we have $\big|\widehat{Q}_{k}(s,a)\big| = |\phi(s,a)^\top \hq_{k,h}| \le  \frac{H}{\gamma}$ by the definition of $\hq_{k,h}$, and  $\|\widehat{w}_{k,h}\|_2 \le K^2$, which implies $\big|\widehat{B}_{k}(s,a)\big|\ind\{s \in \calZ_h\} \le  2K^2$. 

Therefore, 
\begin{align*}
&\E\left[\sum_{k=1}^K \sum_{h=1}^H \E_h^\star
			\left[ \left\langle \pmb{\widehat{\Gamma}}_{k,h} - \pmb{\widehat{B}}_{k,h}, 
				\overline{\pmb{H}}_{\star}(s) - \pmb{H}_{\star}(s) \right\rangle \ind\{\calE_h\}\right]
		\right]  
\\&= \frac{1}{K^3}\E\left[\sum_{k=1}^K \sum_{h=1}^H \E_h^\star
			\left[ \left\langle \pmb{\widehat{\Gamma}}_{k,h} - \pmb{\widehat{B}}_{k,h}, 
				\pmb{H}_{\min}(s) - \pmb{H}_{\star}(s) \right\rangle \ind\{\calE_h\}\right]
		\right]  \tag{by the definition of $\overline{\pmb{H}}_{\star}(s)$}
\\&= \frac{1}{K^3}\E\left[\sum_{k=1}^K \sum_{h=1}^H \E_h^\star
			\left[ \left\langle \widehat{Q}_{k}(s, \cdot) - \widehat{B}_{k}(s,\cdot), 
				\pi_{\min}(\cdot|s) - \pi^\star(\cdot|s) \right\rangle \ind\{\calE_h\}\right]
		\right]
\\&\le \order\left(H\right)
\end{align*}
\end{proof}

\subsubsection{Bound \stabone in \pref{eqn:FTRL decomposition}}
To bound \stabone, we first introduce a useful identity in \pref{lem:bregman bound}. This is first proposed in \cite{zimmert2022return} and restated in \cite{liu2023bypassing}.
\begin{lemma}[Lemma 25 in \cite{liu2023bypassing}] \label{lem:bregman bound}
Let $\pmb{G} = \begin{bmatrix} G+gg^{\top}&g \\g^{\top}&1\end{bmatrix}$ and $\pmb{H} = \begin{bmatrix} H+hh^{\top}&h \\h^{\top}&1\end{bmatrix}$, we have 
\begin{equation*}
D_F(\pmb{G}, \pmb{H}) = D_F(G, H) +  \|g-h\|_{H^{-1}}^2 \ge  \|g-h\|_{H^{-1}}^2
\end{equation*}
\end{lemma}

\begin{lemma}[Lemma 12 in \cite{liu2023bypassing}]
Define 
$
\Sigma_{k,h}  = \E_{s\sim \mu_{h}^k}\E_{a\sim\pi_k(\cdot|s)}\left[\phi(s,a)\phi(s,a)^\top\right]
$. If $\gamma \ge  \frac{5d\log\left(6dHK/\delta\right)}{\tau}$, for any $k,h$, with probability $1-\delta$, we have
\begin{align*}
    \hatSigma_{k,h} =  \frac{1}{\tau} \sum_{(s, a, s') \in \calD_{k,h}}\phi(s, a) \phi(s, a)^\top + \gamma I  \succeq \frac{1}{2}\E_{s\sim \mu_{h}^k}\E_{a\sim\pi_k(\cdot|s)}\left[\phi(s,a)\phi(s,a)^\top\right] = \frac{1}{2}\Sigma_{k,h}. 
\end{align*}
\label{lem:Sigma estimation}
\end{lemma}

\begin{lemma}
If $\gamma \ge  \frac{5d\log\left(6dHK/\delta\right)}{\tau}$, then
\begin{equation*}
    \stabone \le  \eta H^2 \E\left[\sum_{k=1}^K\sum_{h=1}^H\E_h^\star\E_{a \sim \pi_{k}(\cdot|s)}\left[\|\phi(s,a)\|_{\hatSigma^{-1}_{k,h}}^2 \ind\{\calE_h\}\right]\right]. 
\end{equation*}
\label{lem: stability-1 bound}
\end{lemma}

\begin{proof}
In this proof, we define 
\begin{itemize}
    \item $\phi(s,\pi) = \E_{a \sim \pi(\cdot|s)}\left[\phi(s,a)\right]$
    \item $\bcov(s,\pi) = \E_{a \sim \pi(\cdot|s)} \left[\left(\phi(s,a) - \phi(s,\pi) \right) \left(\phi(s,a) - \phi(s,\pi) \right)^\top\right]$  
    \item $\Cov(s,\pi) = \E_{a \sim \pi(\cdot|s)} \left[\phi(s,a)\phi(s,a)^\top\right]$
\end{itemize}

Let $\E_k\left[\cdot\right]$ be the expectation conditioned on history up to episode $k-1$. Consider a fixed $s \in \calS_h$ and any policy $\pi$. Let 
$$\pmb{H}(s) = \E_{a\sim\pi(\cdot|s)}\begin{bmatrix}
        \phi(s,a)\phi(s,a)^\top & \phi(s,a)\\
        \phi(s,a)^\top & 1
    \end{bmatrix}. 
    $$ 
We have 
\begin{align*}
&\E_k\left[\left\langle \pmb{H}_k(s) - \pmb{H}(s), \pmb{\hGamma}_{k,h} \right\rangle - \frac{D(\pmb{H}(s), \pmb{H}_k(s))}{2\eta} \right]
\\&\le \E_k\left[\left\langle \phi(s,\pi_k) - \phi(s,\pi), \hq_{k,h}\right\rangle - \frac{\left\|\phi(s,\pi_k) - \phi(s,\pi)\right\|_{\bcov(s,\pi_k)^{-1}}^2}{2\eta}\right] \tag{\pref{lem:bregman bound}}
\\&\le \E_k\left[\left\| \phi(s,\pi_k) - \phi(s,\pi)\right\|_{\bcov(s,\pi_k)^{-1}}\|\hq_{k,h}\|_{\bcov(s,\pi_k)}- \frac{\left\|\phi(s,\pi_k) - \phi(s,\pi)\right\|_{\bcov(s,\pi_k)^{-1}}^2}{2\eta}\right]
\\&\leq\frac{\eta}{2}\E_k\left[ \left\| \hq_{k,h} \right\|_{\bcov(s,\pi_k)}^2\right] \tag{AM-GM inequality}
\\&\leq \frac{\eta}{2}\E_k\left[ \left\|\hatSigma_{k,h}^{-1} \phi(s_{k,h}, a_{k,h}) \sum_{t=h}^H \ell_t^k \right\|_{\Cov(s,\pi_k)}^2\right] \tag{$\Cov(s,\pi)\succeq \bcov(s,\pi)$}
\\&\le \frac{\eta H^2}{2}\E_k\left[ \phi(s_{k,h}, a_{k,h})^\top \hatSigma_{k,h}^{-1} \Cov(s,\pi_k)\hatSigma_{k,h}^{-1} \phi(s_{k,h}, a_{k,h}) \right]
\\& = \frac{\eta H^2}{2}\E_k\left[ \Tr\left(\phi(s_{k,h}, a_{k,h}) \phi(s_{k,h}, a_{k,h})^\top \hatSigma_{k,h}^{-1} \Cov(s,\pi_k) \hatSigma_{k,h}^{-1} \right)\right]
\\&= \frac{\eta H^2}{2} \Tr\left(\Sigma_{k,h} \hatSigma_{k,h}^{-1} \Cov(s,\pi_k)\hatSigma_{k,h}^{-1} \right)
\\&\le \eta H^2 \Tr\left(\Cov(s,\pi_k)\hatSigma_{k,h}^{-1} \right) \tag{\pref{lem:Sigma estimation}}  
\\&= \eta H^2 \E_{a \sim \pi_{k}(\cdot|s)}\left[\|\phi(s,a)\|_{\hatSigma^{-1}_{k,h}}^2\right]
\end{align*}
Taking expectation and adding indicator for $s$, and then summing over all $k,h$ finish the proof.
\end{proof}

\subsubsection{Bound \stabtwo in \pref{eqn:FTRL decomposition}}
Given $F(X) = -\log{\det(X)}$, $D^2F(X) = X^{-1} \otimes X^{-1}$ where $\otimes$ is the Kronecker product. For any matrix $A = \begin{bmatrix}
    a_1 & a_2 & \cdots & a_n
\end{bmatrix}$, let $\text{vec}(A) = \begin{bmatrix}
    a_1 \\ \vdots \\ a_n
\end{bmatrix}$ 
which vectorizes matrix $A$ to a column vector by stacking the columns $A$. The second order directional derivative for $F$ is $D^2F(X)[A, A] = \text{vec}(A)^\top\left(X^{-1} \otimes X^{-1} \right)\text{vec}(A) = \Tr(A^\top X^{-1}AX^{-1})$.  We define $\|A\|_{\nabla^2F(X)} = \sqrt{\Tr(A^\top X^{-1}AX^{-1})}$ and $\|A\|_{\nabla^{-2}F(X)} = \sqrt{\Tr(A^\top XAX)}$. It is a pseudo-norm, and more discussion can be found in Appendix D of \cite{zimmert2022pushing}. In the following analysis, we will only use one property of this pseudo-norm which is similar to the Holder inequality. It is standard and also appears as Lemma 8 in \cite{liu2023bypassing}.
\begin{lemma} \label{lem:matrix norm holder}
For any two symmetric matrices $A,B$ and positive definite matrix $X$,
\begin{equation*}
\langle A, B \rangle \le \|A\|_{\nabla^{2}F(X)}\|B\|_{\nabla^{-2}F(X)}
\end{equation*}
\end{lemma}

\begin{proof}
Since $(X \otimes X)^{-1} = X^{-1} \otimes X^{-1}$, from Holder inequality, we have
\begin{align*}
    \langle A, B\rangle = \langle \text{vec}(A), \text{vec}(B) \rangle \le \|\text{vec}(A)\|_{X^{-1} \otimes X^{-1} } \|\text{vec}(B)\|_{(X^{-1} \otimes X^{-1})^{-1} } = \|A\|_{\nabla^{2}F(X)}\|B\|_{\nabla^{-2}F(X)}
\end{align*}
\end{proof}

\pref{lem:general stability-2} gives a general argument to bound \stabtwo with arbitrary $\pmb{B} \in \mathbb{R}^{(d+1)\times(d+1)}$. Similar theorems are also stated in Lemma 34 of \cite{dann2023blackbox} and Lemma 27 of \cite{liu2023bypassing}.

\begin{lemma} For any matrix $\pmb{B} \in \mathbb{R}^{(d+1)\times(d+1)}$, for any state $s$, given $\sqrt{\Tr( \pmb{H}_k(s)\pmb{B} \pmb{H}_k(s)\pmb{B})} \le m$, if $\eta \le \frac{1}{16m}$, 
\begin{align*}
\max\limits_{\pmb{H} \in \mathcal{H}_s}\left\langle \pmb{H}_k(s) - \pmb{H}, - \pmb{B} \right\rangle - \frac{D_F(\pmb{H}, \pmb{H}_k(s))}{\eta}  &\le  8\eta \|\pmb{B} \|_{\nabla^{-2}F( \pmb{H}_k(s))}^2 = 8\eta \Tr\left(\pmb{H}_k(s)\pmb{B}\pmb{H}_k(s)\pmb{B}\right).
\end{align*}
\label{lem:general stability-2}
\end{lemma}

\begin{proof}
For any $\pmb{H} \in \calH_s$, define 
 \begin{equation*}
 G(\pmb{H}) = \left\langle \pmb{H}_k(s) - \pmb{H}, - \pmb{B} \right\rangle - \frac{D_F(\pmb{H}, \pmb{H}_k(s))}{\eta}    
 \end{equation*}
 and $\lambda = \|\pmb{B} \|_{\nabla^{-2}F( \pmb{H}_k(s))}$. Since $\sqrt{\Tr( \pmb{H}_k(s)\pmb{B} \pmb{H}_k(s)\pmb{B})} \le m$ and $\eta \le \frac{1}{16m}$, we have
\begin{equation*}
\eta\lambda = \eta\|\pmb{B} \|_{\nabla^{-2}F( \pmb{H}_k(s))}  
 = \eta\sqrt{\Tr( \pmb{H}_k(s)\pmb{B} \pmb{H}_k(s)\pmb{B})} \le \eta m\le  \frac{1}{16}.    
\end{equation*}
 Let $\pmb{H}'$ be the maximizer of $G$. Since $G( \pmb{H}_k(s)) = 0$, we have $G(\pmb{H}') \ge 0$. It suffices to show $\|\pmb{H}' - \pmb{H}_k(s)\|_{\nabla^{2}F( \pmb{H}_k(s))} \le 8\eta\lambda$ because from \pref{lem:matrix norm holder} it leads to
 \begin{equation*}
G(\pmb{H}') \le \| \pmb{H}_k(s) - \pmb{H}'\|_{\nabla^{2}F( \pmb{H}_k(s))} \|\pmb{B}\|_{\nabla^{-2}F( \pmb{H}_k(s))} \le 8\eta \lambda  \|\pmb{B}\|_{\nabla^{-2}F( \pmb{H}_k(s))} = 8\eta \|\pmb{B} \|_{\nabla^{-2}F( \pmb{H}_k(s))}^2
 \end{equation*}

To show $\|\pmb{H}' - \pmb{H}_k(s)\|_{\nabla^{2}F( \pmb{H}_k(s))} \le 8\eta\lambda$, it suffices to show that for all $\pmb{U}$ such that $\|\pmb{U} - \pmb{H}_k(s)\|_{\nabla^{2}F( \pmb{H}_k(s))} = 8\eta \lambda$, $G(\pmb{U}) \le 0$. This is because given this condition, if $\|\pmb{H}' - \pmb{H}_k(s)\|_{\nabla^{2}F( \pmb{H}_k(s))} > 8\eta\lambda$, then there is a $\pmb{U}$ in the line segment between $ \pmb{H}_k(s)$ and $\pmb{H}'$ such that  $\|\pmb{U} - \pmb{H}_k(s)\|_{\nabla^{2}F( \pmb{H}_k(s))} = 8\eta\lambda$. From the condition, $G(\pmb{U}) \le 0 \le \min\{G( \pmb{H}_k(s)), G(\pmb{H}')\}$ which contradicts to the concavity of $G$.

Now consider any $\pmb{U}$ such that $\|\pmb{U} - \pmb{H}_k(s)\|_{\nabla^{2}F( \pmb{H}_k(s))} = 8\eta \lambda$. By Taylor expansion, there exists $\pmb{U}'$ in the line segment between $\pmb{U}$ and $ \pmb{H}_k(s)$ such that
\begin{equation*}
G(\pmb{U}) \le \|\pmb{U} -  \pmb{H}_k(s)\|_{\nabla^{2}F( \pmb{H}_k(s))} \|\pmb{B}\|_{\nabla^{-2}F( \pmb{H}_k(s))} - \frac{1}{2\eta}\|\pmb{U} -  \pmb{H}_k(s)\|_{\nabla^{2}F(\pmb{U}')}^2
\end{equation*}
We have $\|\pmb{U}' - \pmb{H}_k(s)\|_{\nabla^{2}F( \pmb{H}_k(s))} \le \|\pmb{U} - \pmb{H}_k(s)\|_{\nabla^{2}F( \pmb{H}_k(s))} = 8\eta \lambda \le \frac{1}{2}$. From the Equation 2.2 in page 23 of \cite{nemirovski2004interior} (also appear in Eq.(5) of \cite{abernethy2009competing}) and $\log\det$ is a self-concordant function, we have  $\|\pmb{U} -  \pmb{H}_k(s)\|_{\nabla^{2}F(\pmb{U}')}^2 \geq \frac{1}{4}\|\pmb{U} -  \pmb{H}_k(s)\|_{\nabla^{2}F( \pmb{H}_k(s))}^2$. Thus, we have 
\begin{equation*}
G(\pmb{U}) \le  \|\pmb{U} -  \pmb{H}_k(s)\|_{\nabla^{2}F( \pmb{H}_k(s))} \|\pmb{B}\|_{\nabla^{-2}F( \pmb{H}_k(s))} - \frac{1}{8\eta}\|\pmb{U} -  \pmb{H}_k(s)\|_{{( \pmb{H}_k(s))^{-1}}}^2 = 8\eta \lambda^2 - \frac{(8\eta\lambda)^2}{8\eta} = 0.
\end{equation*}

\end{proof}

\begin{lemma}
Given $B_{k}(s,a) = \beta\|\phi(s,a)\|_{\hatSigma^{-1}_{k,h}}^2 + \alpha\left(1-\frac{1}{4H}\right)\|\phi(s,a)\|_{\Lambda_{k,h}^{-1}}^2 + \phi(s,a)^\top w_{k,h} $ defined in \pref{eqn: B define app} for $s\in\calS_h$, if \highlight{$\eta \le \frac{1}{3328 H^2\left(\frac{\beta}{\gamma} + \alpha \rho^2\right)}$}, we have 
\begin{equation*}
    \stabtwo \le \frac{1}{2H}\E\left[\sum_{k=1}^K\sum_{h=1}^H\E_h^\star\E_{a\sim \pi_{k}(\cdot|s)}\left[ B_{k}(s,a)\ind\{\calE_h\}\right]\right] + \otil\left(\frac{(d B^{\max})^2}{\alpha}K\right). 
\end{equation*}
\label{lem:stability-2 bound}
\end{lemma}

\begin{proof}
We can decompose the bonus matrix in the following form and consider stability separately 
\begin{equation*}
\pmb{\widehat{B}}_{k,h} =  \begin{bmatrix}
 \beta \hatSigma_{k,h}^{-1} + \alpha \Lambda_{k,h}^{-1} &  \frac{1}{2} \widehat{w}_{k,h}
\\ \frac{1}{2} \widehat{w}_h^{k\top} &  0 \end{bmatrix} =  \underbrace{\begin{bmatrix}
 \beta \hatSigma_{k,h}^{-1} + \alpha \Lambda_{k,h}^{-1} & 0
\\ 0 &  0 \end{bmatrix}}_{\pmb{\widehat{B}}_{k,h}^1} +  \underbrace{\begin{bmatrix}
0 &  \frac{1}{2} \widehat{w}_{k,h}
\\ \frac{1}{2} \widehat{w}_h^{k\top} &  0 \end{bmatrix}}_{\pmb{\widehat{B}}_{k,h}^2}.
\end{equation*}
Then we have
\begin{align*}
\stabtwo &= \E\left[\sum_{k=1}^K\sum_{h=1}^H\E_h^\star\left[ \left(\max_{\pmb{H} \in \calH_s} \langle \pmb{H}_k(s) - \pmb{H}, -\pmb{\widehat{B}}_{k,h}  \rangle - \frac{D_F(\pmb{H}, \pmb{H}_k(s))}{2\eta}\right) \ind\{\calE_h\}\right]\right]
\\&\le \E\left[\sum_{k=1}^K\sum_{h=1}^H\E_h^\star\left[ \left(\max_{\pmb{H} \in \calH_s} \langle \pmb{H}_k(s) - \pmb{H}, -\pmb{\widehat{B}}_{k,h}^1  \rangle - \frac{D_F(\pmb{H}, \pmb{H}_k(s))}{4\eta}\right) \ind\{\calE_h\}\right]\right] 
\\&\quad + \E\left[\sum_{k=1}^K\sum_{h=1}^H\E_h^\star\left[ \left(\max_{\pmb{H} \in \calH_s} \langle \pmb{H}_k(s) - \pmb{H}, -\pmb{\widehat{B}}_{k,h}^2  \rangle - \frac{D_F(\pmb{H}, \pmb{H}_k(s))}{4\eta}\right) \ind\{\calE_h\}\right]\right]
\end{align*}

For any matrix $A \in \mathbb{R}^{d \times d}$ with all non-negative eigenvalues, we have
\begin{align*}
    \Tr\left(A^2\right) = \sum_{i=1}^d\lambda_i(A^2) \le \left(\sum_{i=1}^d\lambda_i(A)\right)^2 = \Tr\left(A\right)^2
\end{align*}
Since both $\phi(s,a)\phi(s,a)^\top$ and  $\beta\hatSigma_{k,h}^{-1} + \alpha\Lambda_{k,h}^{-1}$ are positive semi-definite, the eigenvalues of $\phi(s,a)\phi(s,a)^\top\left(\beta\hatSigma_{k,h}^{-1} + \alpha\Lambda_{k,h}^{-1}\right)$ are all non-negative. Thus, for any $s \in \calZ_h$, we have
\begin{align*}
\sqrt{\Tr\left(\pmb{H}_k(s)\pmb{\widehat{B}}_{k,h}^1\pmb{H}_k(s)\pmb{\widehat{B}}_{k,h}^1\right)} &\le \sqrt{\Tr\left(\left(\E_{a\sim \pi_{k}(\cdot|s)}\left[\phi(s,a)\phi(s,a)^\top\left(\beta\hatSigma_{k,h}^{-1} + \alpha\Lambda_{k,h}^{-1}\right)\right]\right)^2\right)}
\\&\le \Tr\left(\E_{a\sim \pi_{k}(\cdot|s)}\left[\phi(s,a)\phi(s,a)^\top\left(\beta\hatSigma_{k,h}^{-1} + \alpha\Lambda_{k,h}^{-1}\right)\right]\right)
\\&= \E_{a\sim \pi_{k}(\cdot|s)}\left[\beta\|\phi(s,a)\|_{\hatSigma^{-1}_{k,h}}^2 + \alpha\|\phi(s,a)\|^2_{\Lambda_{k,h}^{-1}}\right]
\\&\le \frac{\beta}{\gamma} + \alpha\rho^2.  \tag{$\|\phi(s,a)\|_{\Lambda_{k,h}^{-1}}\leq \rho$ for $s \in \calZ_h$}
\end{align*}
Thus, from  \pref{lem:general stability-2}, if $\eta \le \frac{1}{64H\left(\frac{\beta}{\gamma} + \alpha \rho^2\right)}$,  we have
\begin{align}
&\E\left[\sum_{k=1}^K\sum_{h=1}^H\E_h^\star\left[ \left(\max_{\pmb{H} \in \calH_s} \langle \pmb{H}_k(s) - \pmb{H}, -\pmb{\widehat{B}}_{k,h}^1  \rangle - \frac{D_F(\pmb{H}, \pmb{H}_k(s))}{4\eta}\right) \ind\{\calE_h\}\right]\right] \nonumber
\\&\le 8\eta\sum_{k=1}^K\sum_{h=1}^H\E_h^\star\left[\Tr\left(\pmb{H}_k(s)\pmb{\widehat{B}}_{k,h}^1\pmb{H}_k(s)\pmb{\widehat{B}}_{k,h}^1\right)\ind\{\calE_h\}\right] \nonumber
\\&\le \frac{1}{8H}\sum_{k=1}^K\sum_{h=1}^H\E_h^\star\left[\sqrt{\Tr\left(\pmb{H}_k(s)\pmb{\widehat{B}}_{k,h}^1\pmb{H}_k(s)\pmb{\widehat{B}}_{k,h}^1\right)}\ind\{\calE_h\}\right] \nonumber
\\&\le \frac{1}{8H}\sum_{k=1}^K\sum_{h=1}^H\E_h^\star\E_{a\sim \pi_{k}(\cdot|s)}\left[\left(\beta\|\phi(s,a)\|_{\hatSigma^{-1}_{k,h}}^2 + \alpha\|\phi(s,a)\|^2_{\Lambda_{k,h}^{-1}}\right)\ind\{\calE_h\}\right].
\label{eqn: stability-2-1}
\end{align}

Now consider $\pmb{\widehat{B}}_{k,h}^2$, for any $s\in\calZ_h$, we have
\begin{align*}
&\sqrt{\Tr\left(\pmb{H}_k(s)\pmb{\widehat{B}}_{k,h}^2\pmb{H}_k(s)\pmb{\widehat{B}}_{k,h}^2\right)}
\\&= \sqrt{2\Tr\left((\widehat{w}_{k,h})^\top \E_{a \sim \pi_{k}(\cdot|s)}\left[\phi(s,a)\right] \E_{a \sim \pi_{k}(\cdot|s)}\left[\phi(s,a)^\top\right]\widehat{w}_{k,h} + (\widehat{w}_{k,h})^\top \E_{a \sim \pi_{k}(\cdot|s)}\left[\phi(s,a)\phi(s,a)^\top\right] \widehat{w}_{k,h}  \right)}
\\&\le 2\sqrt{\E_{a \sim \pi_{k}(\cdot|s)}\left[\left(\phi(s,a)^\top\widehat{w}_{k,h}\right)^2\right]} \le 2\left(1+\frac{1}{2H}\right)B^{\max}_h \le 26H\left(\frac{\beta}{\gamma} + \alpha\rho^2\right). \tag{by \pref{eq: induct3}}
\end{align*}

\noindent Similarly, from \pref{lem:general stability-2}, if $\eta \le \frac{1}{3328 H^2\left(\frac{\beta}{\gamma} + \alpha \rho^2\right)} \le \frac{1}{256  HB^{\max}}$, then for all $h \in [H]$ and any state $s \in \calZ_h$, we have
\begin{align*}
&\max_{\pmb{H} \in \calH} \left\langle \pmb{H}_k(s) - \pmb{H}, -\pmb{\widehat{B}}_{k,h}^2 \right\rangle - \frac{D_F(\pmb{H}, \pmb{H}_k(s))}{4\eta}
\\&\le 8\eta \Tr\left(\pmb{H}_k(s)\pmb{\widehat{B}}_{k,h}^2\pmb{H}_k(s)\pmb{\widehat{B}}_{k,h}^2\right)
\\&\le 32\eta \E_{a \sim \pi_{k}(\cdot|s)}\left[\left(\phi(s,a)^\top\widehat{w}_{k,h}\right)^2\right]
\\&= 32\eta \E_{a \sim \pi_{k}(\cdot|s)}\left[\left(\phi(s,a)^\top w_{k,h} + \phi(s,a)^\top\left(\widehat{w}_{k,h} - w_{k,h} \right) \right)^2\right]
\\&\le 64\eta \E_{a \sim \pi_{k}(\cdot|s)}\left[\left(\phi(s,a)^\top w_{k,h}\right)^2\right] + 64\eta\E_{a \sim \pi_{k}(\cdot|s)}\left[\left(\phi(s,a)^\top\left(\widehat{w}_{k,h} - w_{k,h} \right) \right)^2\right]\tag{$(a+b)^2 \le 2a^2 + 2b^2$}
\\&\le \frac{1}{4H}\E_{a \sim \pi_{k}(\cdot|s)}\left[\phi(s,a)^\top w_{k,h}\right]+ \frac{1}{H}\E_{a \sim \pi_{k}(\cdot|s)}\left[\left|\phi(s,a)^\top\left(\widehat{w}_{k,h} - w_{k,h} \right)\right|\right]\tag{see the explanation below} 
\\&\le \frac{(\con d B^{\max})^2}{H\alpha} + \frac{1}{4H}\E_{a \sim \pi_{k}(\cdot|s)}\left[\phi(s,a)^\top w_{k,h}\right] + \frac{1}{4H}\E_{a \sim \pi_{k}(\cdot|s)}\left[\alpha\left\|\phi(s,a)\right\|_{\Lambda_{k,h}^{-1}}^2\right]. \tag{\pref{lem: uniform w concen} and AM-GM}
\end{align*}
where in the second-last inequality, we use the condition of $\eta$ and that 
\begin{align*}
     |\phi(s,a)^\top w_{k,h}| &\leq  \left(1+\frac{1}{H}\right)\sup_{s'\in\calS_h} \widehat{W}(s')\ind\{s'\in\calZ_{h+1}\} \leq  B^{\max},  \tag{by \pref{eq: bound on hat W}}\\
    |\phi(s,a)^\top (\widehat{w}_{k,h} - w_{k,h})| &\leq  |\phi(s,a)^\top \widehat{w}_{k,h}| + |\phi(s,a)^\top w_{k,h}| \leq \left(2+\frac{1}{2H}\right)B^{\max}.   \tag{by \pref{eq: induct3}}
\end{align*}
Thus, 
\begin{align}
&\E\left[\sum_{k=1}^K\sum_{h=1}^H\E_h^\star\left[ \left(\max_{\pmb{H} \in \calH_s} \langle \pmb{H}_k(s) - \pmb{H}, -\pmb{\widehat{B}}_{k,h}^2  \rangle - \frac{D_F(\pmb{H}, \pmb{H}_k(s))}{4\eta}\right) \ind\{\calE_h\}\right]\right] \nonumber
\\&\le  \frac{1}{4H}\left[\sum_{k=1}^K\sum_{h=1}^H\E_h^\star\E_{a \sim \pi_{k}(\cdot|s)}\left[\phi(s,a)^\top w_{k,h}  \ind\{\calE_h\}\right]\right]   \nonumber 
\\&\qquad \qquad 
+ \frac{1}{4H}\left[\sum_{k=1}^K\sum_{h=1}^H \E_h^\star\E_{a \sim \pi_{k}(\cdot|s)}\left[\alpha\left\|\phi(s,a)\right\|_{\Lambda_{k,h}^{-1}}^2 \ind\{\calE_h\}\right]\right] + \otil\left(\frac{ (dB^{\max})^2}{\alpha}K\right). 
\label{eqn:stability-2-2}
\end{align}
Combining \pref{eqn: stability-2-1} and \pref{eqn:stability-2-2}, we see that if $\eta \le \frac{1}{3328 H^2\left(\frac{\beta}{\gamma} + \alpha \rho^2\right)}$, then 
\begin{align*}
&\stabtwo 
\\&\le \frac{1}{8H} \E\left[\sum_{k=1}^K\sum_{h=1}^H\E_h^\star\E_{a\sim \pi_{k}(\cdot|s)}\left[\left(\beta\|\phi(s,a)\|_{\hatSigma^{-1}_{k,h}}^2 + \alpha\|\phi(s,a)\|^2_{\Lambda_{k,h}^{-1}}\right)\ind\{\calE_h\}\right]\right] 
\\&\quad \quad + \frac{1}{4H}\E\left[\sum_{k=1}^K\sum_{h=1}^H\E_h^\star\E_{a \sim \pi_{k}(\cdot|s)}\left[\phi(s,a)^\top w_{k,h}  \ind\{\calE_h\}\right]\right] 
\\&\quad \quad + \frac{1}{4H}\E\left[\sum_{k=1}^K\sum_{h=1}^H \E_h^\star\E_{a \sim \pi_{k}(\cdot|s)}\left[\alpha\left\|\phi(s,a)\right\|_{\Lambda_{k,h}^{-1}}^2 \ind\{\calE_h\}\right]\right] + \otil\left(\frac{ (dB^{\max})^2}{\alpha}K\right)
\\&\le \frac{1}{2H}\E\left[\sum_{k=1}^K\sum_{h=1}^H\E_h^\star\E_{a\sim \pi_{k}(\cdot|s)}\left[ B_{k}(s,a)\ind\{\calE_h\}\right]\right] + \otil\left(\frac{ (dB^{\max})^2}{\alpha}K\right). 
\end{align*}
\end{proof}

\begin{lemma} If $\eta \le \frac{1}{3228H^2\left(\frac{\beta}{\gamma} + \alpha \rho^2\right)}$ and $\gamma \ge   \frac{5d\log\left(6dHK/\delta\right)}{\tau}$ and $B_h^{\max}\leq \frac{\alpha}{225\log(\frac{dK}{\delta}) Hd^2}$ and $\eta H^2 \leq \frac{3}{4}\beta$,  then we have
\begin{align*}
&\E\left[\sum_{k=1}^K \sum_{h=1}^H \E_h^\star\left[\left\langle Q_{k}(s,\cdot) - B_{k}(s,a), \pi_k(\cdot|s) - \pi^\star(\cdot|s)\right\rangle\ind\{\calE_h\}\right] \right] 
\\&\le \otil\left(\frac{d^2H^3}{\tau\beta}K + \frac{d^2 H^2 (B^{\max})^2}{\alpha}K  + \frac{d\tau}{\eta}  \right) \\
&\qquad + \sum_{k=1}^K\sum_{h=1}^H\E_h^\star\E_{a \sim \pi^\star(\cdot|s)}\left[b_{k}(s,a)\ind\{\calE_h\}\right]  + \frac{1}{H}\sum_{k=1}^K\sum_{h=1}^H\E_h^\star\E_{a\sim \pi_{k}(\cdot|s)}\left[ B_{k}(s,a)\ind\{\calE_h\}\right]
\end{align*}
\label{lem:regret-term condition proof}
where $b_{k}(s,a)$ is defined in \pref{eqn: b define app}.
\end{lemma}

\begin{proof}
Since $\eta \le \frac{1}{3328 H^2\left(\frac{\beta}{\gamma} +\alpha\rho^2\right)}$, adding up the bound in \pref{lem: penalty bound}, \pref{lem:error bound}, \pref{lem: stability-1 bound}, and \pref{lem:stability-2 bound} following the decomposition in \pref{eqn:FTRL decomposition}, we get 
\begin{align}
    \ftrl &\le \otil\left(\frac{d\tau}{\eta} + H + \frac{d^2(B^{\max})^2}{\alpha}K\right) +  \eta H^2 \sum_{k=1}^K\sum_{h=1}^H\E_h^\star\E_{a \sim \pi_{k}(\cdot|s)}\left[\|\phi(s,a)\|_{\hatSigma^{-1}_{k,h}}^2 \ind\{\calE_h\}\right] \nonumber
\\&\quad+ \frac{1}{2H}\sum_{k=1}^K\sum_{h=1}^H\E_h^\star\E_{a\sim \pi_{k}(\cdot|s)}\left[ B_{k}(s,a)\ind\{\calE_h\}\right].   \label{eq: ftrl tmp} 
\end{align}

From the decomposition in \pref{eqn:reg-term decomposition} and \pref{lem: bias 12 bound}, \pref{lem:bias 3 bound}, and \pref{eq: ftrl tmp}, under the specified conditions, we have 
\begin{align*}
&\E\left[\sum_{k=1}^K \sum_{h=1}^H \E_h^\star\left[\left\langle Q_{k}(s,\cdot) - B_{k}(s,a), \pi_k(\cdot|s) - \pi^\star(\cdot|s)\right\rangle\ind\{\calE_h\}\right] \right] 
\\&\le \biasone + \biastwo + \biasthree + \ftrl
\\&\le \otil\left(\frac{d^2H^3}{\tau\beta}K + \frac{d^2 H^2 (B^{\max})^2}{\alpha}K  + \frac{d\tau}{\eta}  \right) 
\\&\quad \quad \quad + \left(\frac{\beta}{4} + \eta H^2\right)\sum_{k=1}^K \sum_{h=1}^H \E_h^\star\E_{a \sim \pi^\star(\cdot|s)} \left[\|\phi(s,a)\|_{\hatSigma^{-1}_{k,h}}^2 \ind\{s \in \calZ_h\}\right]  
\\&\quad \quad \quad + \frac{1}{H}\sum_{k=1}^K\sum_{h=1}^H\E_h^\star\E_{a\sim \pi_{k}(\cdot|s)}\left[ B_{k,h}(s,a)\ind\{\calE_h\}\right] 
\\&\le \otil\left(\frac{d^2H^3}{\tau\beta}K + \frac{d^2 H^2 (B^{\max})^2}{\alpha}K  + \frac{d\tau}{\eta}  \right) 
\\&\quad \quad \quad + \sum_{k=1}^K \sum_{h=1}^H \E_h^\star\E_{a \sim \pi^\star(\cdot|s)} \left[b_k(s,a)\right]  + \frac{1}{H}\sum_{k=1}^K\sum_{h=1}^H\E_h^\star\E_{a\sim \pi_{k}(\cdot|s)}\left[ B_{k,h}(s,a)\ind\{\calE_h\}\right] 
\end{align*}
where in the last inequality we use $\frac{\beta}{4} + \eta H^2 \leq \beta$. 
\end{proof}

\subsection{Final Steps}\label{app: combining }

\begin{lemma}\label{lem: upper bound final}
Let $s\in\calS_h$. We have 
\begin{align*}
B_{k}(s,a) &\le r_{k}(s,a) + \left(1+\frac{1}{H}\right) \E_{s' \sim P(\cdot|s,a)}\E_{a' \sim \pi_k(\cdot|s')}\left[B_{k}(s',a') \ind\{s' \in \calZ_{h+1}\}\right]
\end{align*}
where we define 
\begin{align*}
   r_{k}(s,a) = b_k(s,a) + \E_{s' \sim P(\cdot|s,a)}\E_{a' \sim \pi_k(\cdot|s')}\left[
\alpha\|\phi(s',a')\|^2_{\Lambda_{k, h+1}^{-1}}\ind\{s' \in \calZ_{h+1}\}\right] +\frac{2(\con dB^{\max})^2}{\alpha}. 
\end{align*}
\label{lem: B upper bound}
\end{lemma}
\begin{proof}
Since $B_k(s,a) \ge 0$, we have 
\begin{align*}
\left|\widehat{B}_k^+(s,a) - B_k(s,a)\right| &\le \left|\widehat{B}_k(s,a) - B_k(s,a)\right| \\
&= \left|\frac{\alpha}{4H}\|\phi(s,a)\|_{\Lambda_{k,h}^{-1}}^2  + \phi(s,a)^\top\left(\widehat{w}_{k,h} - w_{k,h}\right)\right|
\\&\le \frac{(\con dB^{\max})^2}{\alpha} + \alpha \|\phi(s,a)\|_{\Lambda_{k,h}^{-1}}^2.    \tag{by \pref{lem: uniform w concen}}
\end{align*}
Thus, 
\begin{align*}
&\phi(s,a)^\top w_{k,h} 
\\&=  \left(1+\frac{1}{H}\right)\E_{s' \sim P(\cdot|s,a)}\E_{a' \sim \pi_k(\cdot| s)}\left[\widehat{B}_{k}^+(s',a')\ind\{s' \in \calZ_{h+1}\}\right] 
\\&\le \left(1+\frac{1}{H}\right) \E_{s' \sim P(\cdot|s,a)}\E_{a' \sim \pi_k(\cdot|s')}\left[B_{k}(s',a') \ind\{s' \in \calZ_{h+1}\}\right]
\\&\qquad + \E_{s' \sim P(\cdot|s,a)}\E_{a' \sim \pi_k(\cdot|s')}\left[
\alpha\|\phi(s',a')\|^2_{\Lambda_{k, h+1}^{-1}}\ind\{s' \in \calZ_{h+1}\}\right] +\frac{2(\con dB^{\max})^2}{\alpha}, 
\end{align*}
and
\begin{align*}
&B_{k}(s,a) 
\\&= b_k(s,a) + \phi(s,a)^\top w_{k,h}
\\&\le b_k(s,a) + \left(1+\frac{1}{H}\right) \E_{s' \sim P(\cdot|s,a)}\E_{a' \sim \pi_k(\cdot|s')}\left[B_{k}(s',a') \ind\{s' \in \calZ_{h+1}\}\right]
\\&\qquad + \E_{s' \sim P(\cdot|s,a)}\E_{a' \sim \pi_k(\cdot|s')}\left[
\alpha\|\phi(s',a')\|^2_{\Lambda_{k, h+1}^{-1}}\ind\{s' \in \calZ_{h+1}\}\right] +\frac{2(\con dB^{\max})^2}{\alpha}
\\&= r_{k}(s,a) + \left(1+\frac{1}{H}\right) \E_{s' \sim P(\cdot|s,a)}\E_{a' \sim \pi_k(\cdot|s')}\left[B_{k}(s',a') \ind\{s' \in \calZ_{h+1}\}\right]. 
\end{align*}
\end{proof}

\begin{theorem} Suppose the parameters are properly chosen so that all conditions in \pref{lem:regret-term condition proof} holds (see the proof for the final parameters). Then the regret of \pref{alg: logdet FTRL} has the following guarantee
\begin{align*}
    \E\left[\calR_K\right] \le \otil\left(d^{\frac{3}{2}}H^3K^{\frac{3}{4}}\right). 
\end{align*}
\end{theorem}

\begin{proof}
By \pref{lem: B condition satisfy}, we have for $s\in\calS_h$, 
\begin{align*}
B_{k}(s, a) &\ge b_{k}(s,a) + \left(1+\frac{1}{H}\right) \E_{s' \sim P(\cdot|s,a)}\E_{a' \sim \pi_k(\cdot|s')}\left[B_{k}(s',a') \ind\{s' \in \calZ_{h+1}\}\right]- \otil\left(\frac{d^2(B^{\max})^2}{\alpha}\right). 
\end{align*}
Combining this with \pref{lem:regret-term condition proof}, we see that the two conditions of \pref{lem:dilated bonus guarantee appendix} are satisfied with $f =  \otil\left(\frac{d^2(B^{\max})^2}{\alpha}\right)$ and $g =  \otil\left(\frac{d^2H^3}{\tau\beta}K + \frac{d^2 H^2 (B^{\max})^2}{\alpha}K  + \frac{d\tau}{\eta}  \right) $. Thus, by directly applying \pref{lem:dilated bonus guarantee appendix}, we have
\begin{align*}
\regterm \le \otil\left(\frac{d^2H^3}{\tau\beta}K + \frac{d^2 H^2 (B^{\max})^2}{\alpha}K  + \frac{d\tau}{\eta}  \right)  + \left(1+\frac{1}{H}\right)\E\left[\sum_{k=1}^K\E_{a \sim \pi_k(\cdot|s_1)}\left[B_{k}(s_1,a)\right] \right]
\end{align*}
To bound the last term, below we use induction to show that for $s\in\calS_h$, the following holds: 
\begin{equation*}
\E_{a \sim \pi_k(\cdot|s)}\left[B_{k}(s, a)\right]  \le \left(1+\frac{1}{H}\right)^{H-h}V^{\pi_k}(s; r_k)
\end{equation*}
for the $r_k$ defined in \pref{lem: upper bound final}. 

\paragraph{Base case (step $H$). } for any $s \in \calS_H$, we have
\begin{align*}
    \E_{a \sim \pi_{k}(\cdot|s)}\left[B_{k}(s, a)\right]  =  \E_{a \sim \pi_{k}(\cdot|s)}\left[b_k(s,a)\right] \leq V^{\pi_k}(s; r_k) 
\end{align*}

\paragraph{Induction.} Assume that for any $s \in \calS_{h+1}$, 
\begin{equation*}
\E_{a \sim \pi_k(\cdot|s)}\left[B_{k}(s, a)\right]  \le \left(1+\frac{1}{H}\right)^{H-h-1}V^{\pi_k}(s; r_k). 
\end{equation*}

Then for any $s \in \calS_h$, we have
\begin{align*}
&\E_{a \sim \pi_{k}(\cdot|s)}\left[B_{k}(s,a) \right]
\\&\le \E_{a\sim \pi_{k}(\cdot|s)}\left[r_{k}(s,a) +  \left(1+ \frac{1}{H}\right)\E_{s' \sim P(\cdot|s,a)}\E_{a' \sim \pi_k(\cdot|s')}\left[B_{k}(s',a')\right]
\right]\tag{\pref{lem: B upper bound}}
\\&\le \E_{a\sim \pi_{k}(\cdot|s)}\left[r_{k}(s,a) + \left(1+ \frac{1}{H}\right)^{H-h}\E_{s' \sim P(\cdot|s,a)}\left[V^{\pi_k}(s'; r_k)
\right]\right] \tag{induction hypothesis}
\\&\le \left(1+ \frac{1}{H}\right)^{H-h} \E_{a\sim \pi_{k}(\cdot|s)}\left[r_{k}(s,a) + \E_{s' \sim P(\cdot|s,a)}\left[V^{\pi_k}(s'; r_k)
\right]\right]\tag{$r_{k}(s,a) \ge 0$}
\\&= \left(1+ \frac{1}{H}\right)^{H-h}V^{\pi_k}(s; r_k). 
\end{align*}
Since $ \left(1+ \frac{1}{H}\right)^{H} < e < 3$, we have
\begin{align*}
&\left(1+\frac{1}{H}\right)\sum_{k=1}^K\E_{a \sim \pi_k(\cdot|s_1)}\left[B_{k}(s_1,a) \right] 
\\&\le 3\sum_{k=1}^KV^{\pi_k}(s_1; r_k)
\\&= \otil\left(\sum_{k=1}^K \sum_{h=1}^H \E_{s \sim \mu^{k}_h}\E_{a \sim \pi_{k}(\cdot|s)}\left[\beta  \|\phi(s,a)\|_{\hatSigma^{-1}_{k,h}}^2 + \alpha\|\phi(s,a)\|^2_{\Lambda_{k,h}^{-1}}\right] + \frac{( dB^{\max})^2}{\alpha}K  \right)
\\&\le \otil\left(\beta dHK + \alpha dH + \frac{(dB^{\max})^2}{\alpha}K\right).  
\end{align*}
Given that $B_h^{\max} = 4H\left(1+\frac{1}{H}\right)^{2(H-h+1)}\left(\frac{\beta}{\gamma} + \alpha \rho^2\right)$, we have $B^{\max} \le 36H\left(\frac{\beta}{\gamma} + \alpha \rho^2\right)$. Thus, 
\begin{align*}
\regterm \le \otil\left(\frac{d^2H^3}{\tau\beta}K + \frac{d^2 H^4 \beta^2}{\alpha \gamma^2}K + d^2 H^4 \alpha \rho^4 K  + \frac{d\tau}{\eta} + \beta dHK + \alpha dH \right)
\end{align*}

We pick $\rho = H^{-\frac{1}{2}}d^{-\frac{1}{4}}K^{-\frac{1}{4}}$, $\beta = \sqrt{d}K^{-\frac{1}{4}}$, $\alpha = HK^{\frac{3}{4}}$, $\tau = K^{\frac{1}{2}}$, $\delta = \frac{1}{K^3}$, $\gamma = \frac{5d\log\left(6dHK^4\right)}{\tau}$, $\eta = \frac{K^{-\frac{1}{4}}}{3328\sqrt{d}H^2}$. In that case, if $\sqrt{K} \ge 16200 d^{\frac{3}{2}}H \log\left(dK^4\right) = \widetilde{\Omega}\left(d^{\frac{3}{2}}H\right)$, all conditions in \pref{lem:regret-term condition proof} are satisfied and $\regterm \le \otil(d^{\frac{3}{2}}H^3K^{\frac{3}{4}})$. 

By \pref{lem:goodevent_rfw}, the initial pure exploration phase takes 
\begin{align*}
    K_0 = \otil\left(\frac{\frac{dH}{\rho^2} + d^4 H^4}{\epsilon_{\rm cov}}\right) = \otil\left( d^{\frac{3}{2}}H^2K^{\frac{3}{4}} + d^4H^4K^{\frac{1}{4}} \right)
\end{align*} 
episodes, which contributes to an additional regret of $HK_0 = \otil( d^{\frac{3}{2}}H^3K^{\frac{3}{4}})$ (omitting lower-order terms). Finally, the cost of ignoring states outside of $\calZ$ is 
$H^3K^{-\frac{3}{4}}$ as calculated in \pref{eqn:regret decomposition}. 

Combining all parts of regret finishes the proof.

\end{proof}
\section{Auxilary Lemmas}
\subsection{Uniform Concentration via Covering}


Consider policy class
\begin{align}   
    \mathbf{P}(s) = \left\{p:~ \widehat{\Cov}(s,p) = \argmin_{\pmb{H}\in\calH_s} \left\{\left\langle \pmb{H}, \pmb{Z}\right\rangle + F(\pmb{H}) \right\}, \text{for\ } \pmb{Z}\in \calZ \right\}   \label{eq: bold P}
\end{align}
where $\calZ=[-K^3, K^3]^{(d+1)\times (d+1)} \cap \mathbb{S}$ with $\mathbb{S}$ denoting the set of symmetric matrices.
We define the following function class.
\begin{definition}\label{def: function class V} For any $h$ and any $s \in \calS_h$,
\begin{align*}
&V_h\left(s; \Sigma, \Lambda, w, p\right) =\left(1+\frac{1}{H}\right)\E_{a \sim p}\left[ \left[\beta \|\phi(s,a)\|_{\Sigma^{-1}}^2 + \phi^\top(s,a)w + 2\alpha \|\phi(s,a)\|^2_{\Lambda^{-1}}\right]^+\ind\{s\in\calZ_h\}\right],
\\&\calV_h = \{V\left(s~; \Sigma, \Lambda, w, p\right) \mid \lambda_{\min}\left(\Sigma\right) \ge \gamma, \lambda_{\min}\left(\Lambda\right) \ge 1, \|w\| \le K^2, p \in \mathbf{P}(s)\}.
\end{align*}
where $\mathbf{P}(s)$ is defined in \pref{eq: bold P}. 
\end{definition}

We propose the following two covering lemma. \pref{lem: cover of ball} is standard which argues the upper bound of the cover number of a Euclidean ball. \pref{lem:logdet policy cover} inherits from Lemma 15 in \cite{liu2023bypassing}.
\begin{lemma}[Cover number of Euclidian Ball] For any $\epsilon > 0$, the $\epsilon$-covering of the Euclidean ball in $\mathbb{R}^d$ with radius $R > 0$ is upper bounded by $\left(1 + \frac{2R}{\epsilon}\right)^d$.
\label{lem: cover of ball}
\end{lemma}

\begin{lemma}[Covering for logdet policy class, Lemma 15 in \cite{liu2023bypassing}]
For any $s$, there exists an $\epsilon$-cover $\mathbf{P}'(s)$ of $\mathbf{P}(s)$ with size $\log |\mathbf{P}'(s)| = (d+1)^2\log\frac{24(d+1)^2}{\epsilon}$ such that for any $p \in\mathbf{P}(s)$, there exists an $p' \in\mathbf{P'}(s)$ satisfying 
      \begin{align*}
           \left\| \widehat{\Cov}(s,p) - \widehat{\Cov}(s,p') \right\|_{\rm F}\leq \epsilon.
      \end{align*}
\label{lem:logdet policy cover}
\end{lemma}

\pref{lem:cover number} gives the covering number of function class $\calV_h$.
\begin{lemma}
Let $\calN_{\epsilon}(\calV_h)$ be the $\|\cdot\|_{\infty}$ $\epsilon$-covering number of function class $\calV_h$, for any $h$, we have 
\begin{align*}
\log\left(\calN_\epsilon(\calV_h)\right) \le d\log\left(1 + \frac{16K^2}{\epsilon}\right) &+ d^2\log\left(1 + \frac{16\sqrt{d}\beta}{\epsilon\gamma}\right)  + d^2\log\left(1 + \frac{16\sqrt{d}\alpha}{\epsilon}\right) 
\\&\quad + (d+1)^2\log\left(\frac{96(d+1)^2\left(2\beta \gamma^{-1} + 2\alpha + K^2\right)}{\epsilon}\right).
\end{align*}
If $\frac{\beta}{\gamma} + 2\alpha \le K^2$, then 
\begin{align*}
\log\left(\calN_\epsilon(\calV_h)\right) \le 4(d+1)^2\log\left(\frac{400(d+1)^2K^2}{\epsilon}\right).
\end{align*}
\label{lem:cover number}
\end{lemma}

\begin{proof}
Define 
\begin{align*}
     B\left(s,a; D, E, w \right) =  \|\phi(s,a)\|_{D}^2 + \|\phi(s,a)\|^2_{E} + \phi^\top(s,a)w 
\end{align*}
and consider the following function classes 
\begin{align*}
&\calB = \left\{ B\left(s,a; D, E, w \right)  \mid \|D\|_2 \le 2\beta\gamma^{-1}, \|E\|_2 \le 2\alpha, \|w\|_2 \le 2K^2\right\},
\\&\widetilde{\calV} = \left\{ \E_{a \sim p}\left[B(s,a;D,E,w)\right] \mid B(s,a;D,E,w) \in \calB, p \in \mathbf{P}(s)\right\}.
\end{align*}
For any $V_1 =  \E_{a \sim p_1}\left[B(s,a;D_1,E_1,w_1)\right]$ and $V_2 =  \E_{a \sim p_2}\left[B(s,a; D_2,E_2,w_2)\right]$, it holds that
\begin{align*}
\left|V_1 - V_2\right| &= \left| \E_{a \sim p_1}\left[B(s,a;D_1,E_1,w_1)\right] - \E_{a \sim p_2}\left[B(s,a; D_2,E_2,w_2)\right]\right|
\\&= \left| \E_{a \sim p_1}\left[B(s,a;D_1,E_1,w_1)\right] - \E_{a \sim p_1}\left[B(s,a; D_2,E_2,w_2)\right]\right| 
\\&\quad + \left|\E_{a \sim p_1}\left[B(s,a; D_2,E_2,w_2)\right] - \E_{a \sim p_2}\left[B(s,a; D_2,E_2,w_2)\right]\right|.
\end{align*}
On the one hand, we have
\begin{align*}
 &\left|B\left(s,a; D_1, E_1, w_1 \right) -  B\left(s,a; D_2, E_2, w_2 \right)\right|
 \\&= \left|\|\phi(s,a)\|_{D_1}^2  -  \|\phi(s,a)\|_{D_2}^2 \right| + \left|\phi^\top(s,a)\left(w_1 - w_2\right)\right| + \left|\|\phi(s,a)\|^2_{E_1} - \|\phi(s,a)\|^2_{E_2}\right|
 \\&=  \left|\phi(s,a)^\top\left(D_1 - D_2\right)\phi(s,a)\right| + \left|\phi^\top(s,a)\left(w_1 - w_2\right)\right| + \left|\phi(s,a)^\top\left(E_1 - E_2\right)\phi(s,a)\right| 
 \\&\le \left\|D_1 - D_2\right\|_2 + \|w_1 - w_2\|_2 + \left\|E_1 - E_2\right\|_2 \tag{$\|\phi(s,a)\|_2 \le 1$}
 \\&\le \left\|D_1 - D_2\right\|_{\rm F} + \|w_1 - w_2\|_2 + \left\|E_1 - E_2\right\|_{\rm F}.
\end{align*}


Since for any matrix $A \in \mathbb{R}^{d \times d}$, $\|A\|_{\rm F} \le \sqrt{d}\|A\|_2$, we consider a $\frac{\epsilon}{4}$ net on $\{D \in \mathbb{R}^{d \times d} \mid \|D\|_{\rm F} \le 2\sqrt{d}\beta \gamma^{-1}\}$, a $\frac{\epsilon}{4}$ net on $\{w \in \mathbb{R}^d \mid \|w\|_2 \le 2K^2\}$, a $\frac{\epsilon}{4}$ net on $\{E \in \mathbb{R}^{d \times d} \mid \|E\|_{\rm F} \le 2\sqrt{d}\alpha\}$.
From \pref{lem: cover of ball}, the $\log$ size of these nets is
\begin{align*}
 d\log\left(1 + \frac{16K^2}{\epsilon}\right) + d^2\log\left(1 + \frac{16\sqrt{d}\beta}{\epsilon \gamma}\right)  + d^2\log\left(1 + \frac{16\sqrt{d}\alpha}{\epsilon}\right).
\end{align*}
On the other hand, define $\pmb{B}_2 = \begin{bmatrix}
    D_2+E_2 & \frac{1}{2}w_2\\
    \frac{1}{2}w_2^\top &0
\end{bmatrix}$, we have $\|\pmb{B}_2\|_2 \le 2\beta \gamma^{-1} + 2\alpha + K^2$ and
\begin{align*}
&\left|\E_{a \sim p_1}\left[B(s,a; D_2,E_2,w_2)\right] - \E_{a \sim p_2}\left[B(s,a; D_2,E_2,w_2)\right]\right|
\\&=\left|\left\langle \widehat{\Cov}(s,p_1) - \widehat{\Cov}(s,p_2), \pmb{B}_2\right\rangle\right|
\\&\le \left\| \widehat{\Cov}(s,p_1) - \widehat{\Cov}(s,p_2)\right\|_2 \left\|\pmb{B}_2\right\|_2
\\&\le \left(2\beta \gamma^{-1} + 2\alpha + K^2\right)\left\|\widehat{\Cov}(s,p_1) - \widehat{\Cov}(s,p_2)\right\|_{\rm F}.
\end{align*}
 Moreover, we construct a $\frac{\epsilon}{4\left(2\beta \gamma^{-1} + 2\alpha + K^2\right)}$ net on policy class $\mathbf{P}(s)$ based on Frobenius norm. From \pref{lem:logdet policy cover}, the $\log$ size of this net is 
 \begin{align*}
    (d+1)^2\log\left(\frac{96(d+1)^2\left(2\beta\gamma^{-1} + 2\alpha + L\right)}{\epsilon}\right).
 \end{align*}
Since clipping and adding more constraints will not increase the cover number, for any $h$, we have 
\begin{align*}
\log{\calN_{\epsilon}(\calV_h)} \le \log{\calN_{\epsilon}(\widetilde{\calV})} \le d\log\left(1 + \frac{16K^2}{\epsilon}\right) &+ d^2\log\left(1 + \frac{16\sqrt{d}\beta}{\epsilon\gamma}\right)  + d^2\log\left(1 + \frac{16\sqrt{d}\alpha}{\epsilon}\right) 
\\&\quad + (d+1)^2\log\left(\frac{96(d+1)^2\left(2\beta \gamma^{-1} + 2\alpha + K^2\right)}{\epsilon}\right).
\end{align*}
\end{proof}

\pref{lem:uniform concen} shows the uniform concentration of all functions in $\calV$. It also appears as Lemma D.4 of \cite{jin2020provably}, Lemma D.7 of \cite{sherman2023improved} and Lemma 24 of \cite{sherman2023rate}.
\begin{lemma}
Let $\{x_{\tau}\}$ be a stochastic process on state space $\calS$ with corresponding filtration $\{\calF_{\tau}\}_{\tau = 1}^{\infty}$. Let $\{\phi_{\tau}\}$ be an $\mathbb{R}^d$-valued stochastic process where $\phi_{\tau} \in \calF_{\tau}$, and $\|\phi_{\tau}\| \le 1$. Further, let $\Lambda_n = \lambda I + \sum_{\tau = 1}^n\phi_{\tau}\phi_{\tau}^\top$. Then for any $\delta > 0$, with probability at least $1-\delta$, for all $n \ge 1$ and any $V \in \cal{V}$ such that $\|V\|_{\infty} \le D$, we have
\begin{align*}
\left\|\sum_{\tau = 1}^n \phi_{\tau}\left(V\left(x_{\tau}\right) - \E\left[V(x_{\tau}| \calF_{\tau-1})\right]\right)\right\|_{\Lambda_n^{-1}}^2 \le 4D^2\left(\frac{d}{2}\log\left(\frac{n+\lambda}{\lambda}\right) + \log{\frac{\cal{N}_{\epsilon}(\calV)}{\delta}}\right) + \frac{8n^2\epsilon^2}{\lambda}
\end{align*}
where $\cal{N}_{\epsilon}(\calV)$ is $\|\cdot\|_{\infty}$ $\epsilon$- covering number of $\calV$ with difference $\epsilon$.
\label{lem:uniform concen}
\end{lemma}

\begin{lemma}[Lemma D.4 in \cite{sherman2023improved}]
Let $\{\phi_i\}_{i=1}^n \in \mathbb{R}^d, \{y_i\}_{i=1}^n \in \mathbb{R}, \lambda \in \mathbb{R}$ and set $\Lambda = \sum_{i=1}^N \phi_i\phi_i^\top + \lambda I$, and $\widehat{w} = \Lambda^{-1}\sum_{i=1}^N\phi_iy_i$. Then for any $w^\star \in \mathbb{R}^d$
\begin{align*}
    \|\widehat{w} - w^\star\|_{\Lambda} \le \left\|\sum_{i=1}^N \phi_i \left(y_i-\phi_iw^\star \right)\right\|_{\Lambda^{-1}} + \sqrt{\lambda} \|w^\star\|
\end{align*}
\label{lem:regression difference}
\end{lemma}

\subsection{FTRL Regret Bounds}
\begin{lemma}[Standard FTRL bound]\label{lem: FTRL guarantee}
Let $\Omega\subset \mathbb{R}^d$ be a convex set, $g_1, \ldots, g_T  \in \mathbb{R}^d$, and $\eta > 0$.  Then the FTRL update 
   \begin{align*}
       w_t = \argmin_{w\in\Omega}\left\{ \left\langle w, \sum_{\tau=1}^{t-1} g_\tau\right\rangle + \frac{1}{\eta}\psi(w)\right\}
   \end{align*}
   ensures for any $u\in\Omega$ and $\eta_0>0$, 
   \begin{align*}
       \sum_{t=1}^T  \langle w_t - u, g_t\rangle  \leq \underbrace{\frac{\psi(u) - \min_{w\in\Omega}\psi(w)}{\eta}}_{\textbf{\textup{Penalty}}} + \underbrace{\sum_{t=1}^T \left(\max_{w\in\Omega} \langle w_t - w, g_t\rangle - \frac{D_{\psi}(w,w_t)}{\eta}\right)}_{\textbf{\textup{Stability}}}. 
   \end{align*}
\end{lemma}    

Since we do not use standard FRTL but run the same policy $\pi$ in $2\tau$ episodes. We will introduce a blocked FTRL regret bound in \pref{lem:block FTRL}.
\begin{lemma}\label{lem:block FTRL}
Let $K \in \mathbb{Z}_{+}, \tau \le K, J = \lceil \frac{K}{\tau} \rceil$, and set $T_j = \{\tau(j-1)+1, \cdots, \tau j\}$ for all $j \in [J]$. Assume $\eta > 0$, let $g_k$ be a sequence of input, define
\begin{align*}
    g_{(j)} &= \frac{1}{\tau}\sum_{k \in T_j} g_k, \forall j \in [J] 
    \\w_{(j+1)} &= \argmin_{w\in\Omega}\left\{ \left\langle w, \sum_{\tau=1}^{j} g_{(\tau)} \right\rangle + \frac{1}{\eta}\psi(w)\right\}
\end{align*}
Then if $w_k \in \Omega$ are such that $w_k  = w_{(j)}$ for all $k \in T_j, j \in [J]$, for any $u \in \Omega$ we have
\begin{equation*}
\sum_{k=1}^K \langle g_k, w_k - u \rangle \le \frac{\tau(\psi(u) - \min_{w\in\Omega}\psi(w))}{\eta} + \sum_{k=1}^K \left(\max_{w\in\Omega}\langle w_k - w, g_k\rangle - \frac{D_{\psi}(w,w_{k})}{\eta}\right) 
\end{equation*}
\end{lemma}

\begin{proof}
By applying \pref{lem: FTRL guarantee} on $g_{(j)}, x_{(j)}$, we get
\begin{equation*}
\sum_{j=1}^J \langle g_{(j)}, w_{(j)} - u \rangle \le \frac{\psi(u) - \min_{w\in\Omega}\psi(w)}{\eta} + \sum_{j=1}^J\left(\max_{w\in\Omega} \langle w_{(j)} - w, g_{(j)}\rangle - \frac{D_{\psi}(w,w_{(j)})}{\eta}\right)
\end{equation*}
In addition,
\begin{equation*}
\sum_{j=1}^J \langle g_{(j)}, w_{(j)} - u \rangle  = \sum_{j=1}^J \left\langle  \frac{1}{\tau}\sum_{k \in T_j} g_k, w_k - u \right\rangle = \frac{1}{\tau} \sum_{j=1}^J\sum_{k \in T_j} \langle g_k, w_k - u\rangle = \frac{1}{\tau}\sum_{k=1}^K  \langle g_k, w_k - u \rangle 
\end{equation*}
On the other hand,
\begin{align*}
\sum_{j=1}^J\left(\max_{w\in\Omega} \langle w_{(j)} - w, g_{(j)}\rangle - \frac{D_{\psi}(w,w_{(j)})}{\eta}\right) &\le \sum_{j=1}^J\left(\max_{w\in\Omega} \left\langle w_{(j)} - w, \frac{1}{\tau}\sum_{k \in T_j} g_k \right\rangle - \frac{D_{\psi}(w,w_{(j)})}{\eta}\right) 
\\&\le \sum_{j=1}^J\left(\max_{w\in\Omega} \frac{1}{\tau}\sum_{k \in T_j} \langle w_k - w, g_k\rangle - \frac{1}{\tau}\sum_{k \in T_j} \frac{D_{\psi}(w,w_{k})}{\eta}\right)
\\&\le \frac{1}{\tau}\sum_{j=1}^J\sum_{k \in T_j}\left(\max_{w\in\Omega}\langle w_k - w, g_k\rangle - \frac{D_{\psi}(w,w_{k})}{\eta}\right)
\\&= \frac{1}{\tau}\sum_{k=1}^K \left(\max_{w\in\Omega}\langle w_k - w, g_k\rangle - \frac{D_{\psi}(w,w_{k})}{\eta}\right)
\end{align*}
Thus, we have
\begin{equation*}
\sum_{k=1}^K \langle g_k, w_k - u \rangle \le \frac{\tau(\psi(u) - \min_{w\in\Omega}\psi(w))}{\eta} + \sum_{k=1}^K \left(\max_{w\in\Omega}\langle w_k - w, g_k\rangle - \frac{D_{\psi}(w,w_{k})}{\eta}\right) 
\end{equation*}
\end{proof}

\subsection{Other Technical Lemmas}
\begin{lemma}
\label{lem: technical 1}
Let $x_i$ be a sequence of vectors, $p_i$ a probability distribution and $a_i$ arbitrary scalars, then
\begin{align*}
\norm{\sum_{i}p_ia_ix_i}^2 \leq \left(\sum_{i}p_i\norm{x_i}^2\right)\left(\sum_{j}p_ja_j^2\right)\,.
\end{align*}
\end{lemma}
\begin{proof}
\begin{align*}
\norm{\sum_{i}p_ia_ix_i}^2 &= \norm{\sum_{i}p_ia_i^2\frac{x_i}{a_i}}^2 =
\norm{\sum_{i}\frac{p_ia_i^2}{\sum_jp_ja_j^2}\frac{x_i}{a_i}}^2\left(\sum_jp_ja_j^2\right)^2\\
&\leq \sum_{i}\frac{p_ia_i^2}{\sum_jp_ja_j^2}\norm{\frac{x_i}{a_i}}^2\left(\sum_jp_ja_j^2\right)^2\tag{Jensen's}\\
&=\left(\sum_{i}p_i\norm{x_i}^2\right)\left(\sum_{j}p_ja_j^2\right). 
\end{align*}
\end{proof}

\end{document}